\documentclass[11pt]{article}
\usepackage{latexsym,amssymb,amsmath} % for \Box, \mathbb, split, etc.
\usepackage{path}
\usepackage{url}
\usepackage{verbatim}

\usepackage{natbib}
\usepackage{amsfonts}
\usepackage{amsfonts}
\usepackage{amsfonts}
\usepackage{array}
\usepackage{mathrsfs}
\usepackage{amsfonts}
\usepackage{amsmath}
\usepackage{amssymb}
\usepackage{amsthm}
\usepackage{bm}
\usepackage{appendix}
\usepackage{lscape}
\usepackage{graphicx} % more modern
\usepackage{subfigure}

\usepackage[lined,boxed,ruled, commentsnumbered]{algorithm2e}

\newtheorem{lemma}{Lemma}
\newtheorem{theorem}{Theorem}
\newtheorem{corollary}{Corollary}
\newtheorem{proposition}{Proposition}
\newtheorem{definition}{Definition}

\newtheorem{assumption}{Assumption}

\newcommand{\Lpa}{L_{\text{pa}}}

\newcommand{\supp}{\text{supp}}

\newcommand{\argmin}{\mathop{\arg\min}}

\newcommand{\Tr}{\mathrm{tr}}


\numberwithin{equation}{section}
\numberwithin{table}{section}
\numberwithin{figure}{section}

\newcommand{\cN}{{\mathcal {N}}}
\newcommand{\bR}{{\mathbb{R}}}

\newcommand{\junk}[1]{{}}

\newlength{\fwtwo} 
\DeclareGraphicsExtensions{.jpg}

\title{Partial Gaussian Graphical Model Estimation}
\author{
  Xiao-Tong Yuan \\
  Department of Statistics,
  Rutgers University \\
  New Jersey, 08816 \\
  \texttt{xyuan@stat.rutgers.edu} \\
  \and
  Tong Zhang \\
  Department of Statistics,
  Rutgers University \\
  New Jersey, 08816 \\
  \texttt{tzhang@stat.rutgers.edu}
}

\date{}

\begin{document}

\maketitle

\begin{abstract}
This paper studies the partial estimation of Gaussian graphical
models from high-dimensional empirical observations. We derive a
convex formulation for this problem using $\ell_1$-regularized
maximum-likelihood estimation, which can be solved via a block
coordinate descent algorithm. Statistical estimation performance can
be established for our method. The proposed approach has competitive
empirical performance compared to existing methods, as demonstrated
by various experiments on synthetic and real datasets.
\end{abstract}

%\subparagraph{Key words.} Sparse Eigenvalue, Power Method, Sparse
%Principal Component Analysis, Densest $k$-Subgraph.

%\subparagraph{AMS subject classifications (2010).} Provide up to
%five subject classification codes here; search for the string
%``MSC'' at \url"www.ams.org".

\section{Introduction}

Given $n$ independent copies $\{Z^{(i)}\}_{i=1}^n$ of a random
vector $Z\in \bR^d$ with an unknown covariance matrix $\Sigma$, the
problem of precision matrix (inverse covariance matrix) estimation
is to estimate $\Omega =\Sigma^{-1}$. In particular, for
multivariate normal data, the precision matrix induces the
underlying Gaussian graphical structure among the variables. For
such Gaussian graphical models (GGMs), it is usually assumed that a
given variable can be predicted by a small number of other
variables. This assumption implies that the precision matrix is
sparse. Therefore estimating Gaussian graphical model can be reduced
to the problem of estimating a sparse precision matrix.

One approach to sparse precision matrix estimation is covariance
selection or neighborhood
selection~\citep{Dempster-1972,Meinshausen-NSLasso-2006}, which
tries to estimate each row (or column) of the precision matrix by
predicting the corresponding variable using a sparse linear
combination of other variables. An alternative formulation is
maximum-likelihood estimation method that directly estimate the full
precision matrix. The sparseness of the precision matrix can be
achieved by adding sparse penalty functions such as the
$\ell_1$-penalty or the SCAD
penalty~\citep{Aspremont-2008,Friedman-Glasso-2008,Fan-SCAD-2009}.

In this paper, we are interested in the problem of estimating
blockwise partial precision matrix. Given $n$ independent copies
$\{Y^{(i)}; X^{(i)}\}_{i=1}^n$ of a random vector $Z =(Y;X) \in
\bR^{p} \times \bR^{q}$ with an unknown precision matrix
$$ \Omega = \left[ { {\begin{array}{*{20}{c}}
   {\Omega_{yy}}   & {\Omega_{yx}}   \\
   {\Omega_{yx}^\top} & {\Omega_{xx}}   \\
\end{array}}
} \right],
$$
our goal is to simultaneously estimate the blocks $\Omega_{yy}$ and
$\Omega_{yx}$, without attempting to estimate the block
$\Omega_{xx}$. If the joint distribution of $Z=(Y;X)$ is normal,
then $\Omega_{yy}$ has an interpretation of conditional precision
matrix of $Y$ conditioned on $X$, and $\Omega_{yx}$ induces the
mutual conditional dependency between these two groups of variants.
In machine learning applications where $Y$ is the response and $X$
is the input feature, estimating partial precision matrix can be a
useful tool for constructing graphical models for the response
conditioned on the input. For instance, in multi-label image
annotation, the response $Y$ is the indicator vector of annotation
and the input $X$ is the associated image feature vector. In this
case, $\Omega_{yy}$ induces a Gaussian graphical model for the tags
while $\Omega_{yx}$ identifies the conditional dependency between
tags and features. If we are mainly interested in the conditional
precision matrix $\Omega_{yy}$ and the interaction matrix
$\Omega_{yx}$, then it is natural to ignore $\Omega_{xx}$.
Consequently, we should not have to impose any assumption on the
structure of $\Omega_{xx}$.

Although the existing algorithms for GGMs can be used to estimate
the full precision matrix $\Omega$ and consequently its blocks
$\Omega_{yy}$ and $\Omega_{yx}$, it requires an accurate estimation
of $\Omega_{xx}$; in order to estimate $\Omega_{xx}$ in high
dimension, we have to impose the assumption that $\Omega_{xx}$ is
sparse; and the degree of its sparsity affects the estimation
accuracy of $\Omega_{yy}$ and $\Omega_{yx}$. Moreover, when $q$ is
much larger than $p$, computational procedures for the full GGMs
formulation do not scale well with respect to $\Omega_{xx}$. For
example, the computational complexity of graphical
Lasso~\citep{Friedman-Glasso-2008}, a representative GGMs solver,
for estimating $\Omega$ is $O((p+q)^3)$. This complexity is
dominated by $q$ when $q \gg p$ and thus can be quite inefficient
when $q$ is large. Unfortunately, it is not uncommon for the feature
dimensionality of modern datasets to be of order $10^4\sim10^7$.
Taking document analysis as an example, the typical size of
bag-of-word features is of the order $10^4$. In web data mining, the
feature dimensionality of a webpage is typically of the order
$10^6\sim10^7$. In contrast, the dimensionality of the response $Y$,
e.g., the number of document categories, is usually of a much
smaller order $10^2 \sim 10^3$. The purpose of this paper is to
develop a formulation that directly estimates the precision matrix
blocks $\Omega_{yy}$ and $\Omega_{yx}$ without explicit estimation
of the block $\Omega_{xx}$.

To estimate the underlying graphical model of $Y$, one might
consider applying existing GGMs to the marginal precision matrix
$\tilde\Omega_{yy} = \Sigma_{yy}^{-1}$. However, this approach
ignores the contribution of $X$ in predicting $Y$, and from a
graphical model point of view, the marginal precision matrix
$\tilde\Omega_{yy}$ may be dense. Taking the expression quantitative
trait loci (eQTL) data~\citep{Jansen-eQTL} as an example, if two
genes in $Y$ are both regularized by the same genetic variants in
$X$ at the gene expression level, then there should not be any
dependency of these two genes. However, without taking the genetic
effects of $X$ into consideration, a link between these two genes is
expected.

We introduce in this paper a new sparse partial precision matrix
estimation model that simultaneously estimates the conditional
precision matrix $\Omega_{yy}$ and the block matrix $\Omega_{yx}$
under the assumption that there are many zeros in both matrices. The
key idea is to drop the $\ell_1$ regularization for the
$\Omega_{xx}$ part in the full GGMs formulation; as we will show,
this leads to a convex formulation that does not depend on
$\Omega_{xx}$, and consequently, we do not have to estimate
$\Omega_{xx}$. Numerically this idea allows us to solve the
reformulated problem more efficiently. We propose an efficient
coordinate descent procedure to find the global minimum. The
computational complexity is $O(p^3 + p^2q + pq\min\{n,q\})$, where
$n$ is the sample size. Statistically, we can obtain convergence
results for $\Omega_{yx}$ and $\Omega_{yy}$ in the high dimensional
setting even though we do not impose sparsity assumption on
$\Omega_{xx}$.

Although derived in the context of GGMs, our method is immediately
applicable to the problem of multivariate regression with unknown
noise covariance. This observation establishes the connection
between our method and the conditional GGM proposed
by~\citet{Li-cGGM} which estimates conditional precision matrix
$\Omega_{yy}$ via multivariate regression. However, the conditional
graphical model formulation derived there is quite different from
the partial graphical model formulation of this paper. In fact, the
resulting formulations are different: we impose the sparsity
assumption on $\Omega_{yx}$, which leads to a convex formulation,
while they impose the sparsity assumption on $\Omega_{yy}^{-1}
\Omega_{yx}$, which leads to a non-convex formulation.

In summary, our method has the following merits compared to the
standard GGMs and the method by~\citet{Li-cGGM}:
\begin{itemize}
  \item \textbf{Convexity:} We estimate partial precision matrix via solving a convex optimization problem. In contrast,
  the formulation proposed by~\citet{Li-cGGM} for a similar purpose is non-convex and thus the global minimum cannot be guaranteed.
  \item \textbf{Scalability:} The proposed approach directly estimates the blocks $\Omega_{yy}$ and
  $\Omega_{yx}$ by optimizing out the block of $\Omega_{xx}$. This leads to improved scalability with respect to
   the dimensionality of $X$ in comparison to the standard GGMs formulation that estimates the full precision matrix.

  \item \textbf{Interpretability:} For normal data, the sparsity
    constraint on $\Omega_{yx}$ in our formulation has a natural interpretation
    in terms of the conditional dependency between the variables in $X$ and $Y$.
    This differs from the assumption in~\citep{Li-cGGM} that essentially assumes the sparsity of $\Omega_{yy}^{-1} \Omega_{yx}$
    which does not have natural graphical model interpretation.
  \item \textbf{Theoretical Guarantees:} Theoretical performance of
    our estimator can be established without the sparsity assumption on $\Omega_{xx}$.
\end{itemize}

\subsection{Related Work}

Numerous methods have been proposed for sparse precision matrix
estimation in recent years. For GGMs estimation, a popular
formulation is maximum likelihood estimation with $\ell_1$-penalty
on the entries of the precision matrix
\citep{Yuan-Lin-2007,Banerjee-2008,Rothman-2008}. The
$\ell_1$-penalty leads to sparsity, and the resultant problem is
convex. Theoretical guarantees of this type of methods have been
investigated by~\citet{Ravikumar-EJS-2011,Rothman-2008}, and its
computation has been extensively studied in the
literature~\citep{Aspremont-2008,Friedman-Glasso-2008,Lu-VSM-2009}.
Non-convex formulations have also been considered because it is
known that $\ell_1$-penalty suffers from a so-called {\em bias}
problem that can be remedied using non-convex
penalties~\citep{Fan-SCAD-2009,Johnson-AISTAT-2012}. As an
alternative approach to the maximum likelihood formulation, one may
directly estimate the support (that is, nonzero entries) of the
sparse precision matrix using separate neighborhood estimations for
each variable followed by a proper aggregation
rule~\citep{Meinshausen-NSLasso-2006,Yuan-JMLR-2010,Cai-CLIME-2011}.

The conditional precision matrix $\Omega_{yy}$ is related to the
latent Gaussian Graphical model of~\citep{Venkat-LatentGGM}, where
$Y$ is observed and $X$ are unobserved hidden variables. If we
further assume that $X$ is low-dimensional (which is different from
the situation of observed high dimensional $X$ in this paper), then
the we may write the marginal precision matrix $\tilde\Omega_{yy}$
using the Schur complement as $ \tilde\Omega_{yy} = \Omega_{yy} -
\Omega_{yx}\Omega_{xx}^{-1}\Omega_{yx}^\top $. This exhibits a
sparse low-rank structure because $\Omega_{yy}$ is sparse and the
dimensionality of $X$ is low. \citet{Venkat-LatentGGM} explored such
a sparse low-rank structure and proposed a convex minimization
method to recover $\Omega_{yy}$ as well as the low-rank component.
Although the model is more accurate than standard GGMs, the
formulation does not take advantage of the additional information
provided by $X$ when it is observed. Another issue is that this
latent Gaussian graphical model assumes that the hidden variable $X$
is of low dimension, which may not be realistic for many
applications.

Our approach is also closely related to the conditional Gaussian
graphical model (cGGM)~\citep{Li-cGGM} studied in the context of
eQTL data analysis. The cGGM assumes a sparse multivariate
regression model between $Y$ and $X$ with (unknown) sparse error
precision matrix. However, the log-likelihood objective function
associated with the model is non-convex. Their theoretical analysis
applies for a local minimum solution which may not be the solution
found by the algorithm. The cGGM model has also been considered
in~\citep{Cai-CAPME}. The authors proposed to first estimate the
linear regression parameters by multivariate Dantzig-selector and
then estimate the conditional precision matrix by the CLIME
estimator~\citep{Cai-CLIME-2011}. The rate of convergence for such a
two-stage estimator was analyzed. Different from cGGM, our partial
precision matrix estimation approach directly estimates the blocks
of the full precision matrix via a convex formulation. This
significantly simplifies the computational procedure and statistical
analysis. Particularly, when $Y$ is univariate, our model reduces to
the  $\ell_1$-penalized maximum likelihood estimation studied
by~\citet{Stadler-2010-mixture} for sparse linear regression. For
multivariate random vector $Y$, our method can be regarded as a
multivariate generalization of \citet{Stadler-2010-mixture} for sparse
linear regression with unknown noise covariance.

\subsection{Notation}
In the following, $\Omega$ is a positive semi-definite matrix: $\Omega
\succeq 0$; $x \in \bR^p$ is a vector;
$A \in \bR^{p\times q}$ is a matrix. The following notations
will be used in the text.
\begin{itemize}
\item $\lambda_{\min}(\Omega)$: the smallest eigenvalue of $\Omega$.
\item $\lambda_{\max}(\Omega)$: the largest eigenvalue of $\Omega$.
\item $\Omega^{-}$: the off-diagonals of $\Omega$.
\item $x_i$: the $i$th entry of a vector.
\item $\|x\|_2=\sqrt{x^\top x}$: the Euclidean norm of vector $x$
\item $\|x\|_1 =\sum_{i=1}^d |x_i|$: the $\ell_1$-norm of vector $x$
\item $\|x\|_0$: the number of nonzero of $x$.
\item $A_{ij}$: the element on the $i$th row and $j$th column of  matrix $A$.
\item $A_{i\cdot}$: the $i$th row of $A$.
\item $A_{\cdot j}$: the $j$th column of $A$.
\item $|A|_\infty=\max_{1 \le i \le p, 1 \le j \le q}|A_{ij}|$:
  $\ell_\infty$-norm of $A$.
\item $|A|_1 = \sum_{i=1}^p\sum_{j=1}^q |A_{ij}|$: the element-wise
  $\ell_1$-norm of matrix $A$.
\item  $\|A\|_1 =
\max_{1 \le j \le q} \sum_{i=1}^p |A_{ij}|$: the matrix $\ell_1$-norm
of $A$.
\item $\|A\|_{F} = \sqrt{\sum_{i=1}^p\sum_{j=1}^q A_{ij}^2}$: the
  Frobenius norm of matrix $A$.
\item $\|A\|_2 =\sup_{\|x\|_2 \le 1} \|Ax\|_2$: the spectral norm of
  matrix $A$.
\item $\supp(A)=\{(i,j):A_{ij}\neq 0\}$: the support (set of nonzero
  elements) of $A$.
\item $I$: the identity matrix.
\item $\bar{S}$: the complement of an index set $S$.
\end{itemize}

\subsection{Outline}

The remaining of this paper is organized as follows:
Section~\ref{sect:model} introduces the partial Gaussian graphical
model (pGGM) formulation; its statistical property in the high
dimensional setting is analyzed in Section~\ref{sect:analysis}.
Section~\ref{sect:algorithm} presents a coordinate descent algorithm
which can be used to solve pGGM. The extension of the proposed
method to multivariate regression with unknown covariance is
discussed in Section~\ref{sec:Multivariate_Regression}. Monte-Carlo
simulations and experimental results on real data are given in
Section~\ref{sect:expriment}. Finally, we conclude this paper in
Section~\ref{sect:conclusion}.

\section{Sparse Partial Precision Matrix Estimation}
\label{sect:model}

\subsection{Gaussian Graphical Model}

Suppose that two random vectors $Y \in \bR^p$ and $X \in
\bR^q$ are jointly normally distributed with zero-mean, i.e.,
$Z=(Y;X) \sim \cN (0, \Sigma^*)$. Its density is
parameterized by the precision matrix $\Omega^* =
(\Sigma^*)^{-1}\succ 0$ as follows:
\[
\phi(z; \Omega^*) = \frac{1}{\sqrt{(2\pi)^{p+q} (\det
\Omega^*)^{-1}}} \exp\left\{ -\frac{1}{2}z^\top\Omega^* z \right\}.
\]
It is well known that the conditional independence between $Z_i$ and
$Z_j$ given the remaining variables is equivalent to $\Omega^*_{ij}
= 0$. Let $G = (V,E)$ be a graph representing conditional
independence relations between components of $Z$. The vertex set $V$
has $p+q$ elements corresponding to $Z_1=Y_1,...,Z_p
=Y_p,Z_{p+1}=X_1,...,Z_{p+q}=X_q$, and the edge set $E$ consists of
ordered pairs $(i,j)$, where $(i,j) \in E$ if there is an edge
between $Z_i$ and $Z_j$. The edge between $Z_i$ and $Z_j$ is
excluded from $E$ if and only if $Z_i$ and $Z_j$ are independent
given $\{Z_k, k \neq i, j\}$. Thus for normal distributions,
learning the structure of graph is equivalent to estimating the
support of the precision matrix $\Omega^*$.

Suppose we have $n$ independent observations
$\{Z^{(i)}=(Y^{(i)};X^{(i)})\}_{i=1}^n$ from the normal distribution
$\cN(0, \Sigma^*)$. Let $ \Sigma^n = \left[ {
{\begin{array}{*{20}{c}}
   {\Sigma^n_{yy}}   & {\Sigma^n_{yx}}   \\
   {\Sigma^{n\top}_{yx}} & {\Sigma^n_{xx}}   \\
\end{array}}
} \right] $ be the empirical covariance matrix in which
\[
 \Sigma^n_{yy} = \frac{1}{n}\sum_{i=1}^n Y^{(i)} (Y^{(i)})^\top , \quad \Sigma^n_{yx} =
\frac{1}{n} \sum_{i=1}^n Y^{(i)} (Y^{(i)})^\top , \quad
\Sigma^n_{xx} = \frac{1}{n}\sum_{i=1}^n X^{(i)} (X^{(i)})^\top.
\]
The negative of the logarithm of the likelihood function
corresponding to the GGMs is written by
\begin{equation}\label{equat:L_omega}
L(\Omega):=-\log\det \Omega + \langle \Sigma^n, \Omega
\rangle.\nonumber
\end{equation}
It is well-known that $L(\Omega)$ is convex when $\Omega\succ 0$,
which implies that it is jointly convex with respect to the blocks
$\Omega_{yy}$, $\Omega_{yx}$ and $\Omega_{xx}$. The goal of GGMs
learning can be reduced to the problem of estimating the precision
matrix $\Omega^*$ with extra sparsity constraints. In particular,
the following $\ell_1$-regularized maximum-likelihood method is the
most popular formulation to learn sparse precision
matrix~\citep{Banerjee-2008}:
\begin{equation}\label{prob:GGM}
\hat\Omega = \argmin_{\Omega\succ 0} \{L(\Omega) + \lambda_n
|\Omega^{-}|_1\},
\end{equation}
where $\lambda_n$ is the strength parameter of the penalty.

\subsection{Partial Gaussian Graphical Model}

We now present a new maximum-likelihood formulation for the partial
GGM (pGGM) that only aims at estimating the blocks $\Omega^*_{yy}$
and $\Omega^*_{yx}$ instead of estimating the full precision matrix
$\Omega^*$. Without causing confusion, we can write $L(\Omega)$ as
$L(\Omega_{yy}, \Omega_{yx}, \Omega_{xx})$. The basic idea of pGGM
is to eliminate $\Omega_{xx}$ by optimizing $L(\Omega_{yy},
\Omega_{yx}, \Omega_{xx})$ with respect to $\Omega_{xx}$, and this
can be achieved if we \emph{do not impose any sparsity constraint} on $\Omega_{xx}$. As we will
show in the following, this idea allows us to decouple the
estimation of $\Omega_{xx}$ from the estimation of
$\{\Omega_{yy},\Omega_{yx}\}$. This not only allows faster
computation, but also allows us to develop a theoretical convergence analysis for
$\{\Omega_{yy},\Omega_{yx}\}$ without assuming the sparsity of $\Omega_{xx}$.

We introduce a reparameterization $ \tilde\Omega_{xx}:= \Omega_{xx}
- \Omega_{yx}^\top \Omega_{yy}^{-1}\Omega_{yx}$. Note that $\Omega
\succ 0$ implies $\tilde\Omega_{xx} \succ 0$. The following
proposition indicates that with such a reparameterization, $L$ can
be decomposed as the sum of a component only dependent on
$\{\Omega_{yy},\Omega_{yx}\}$ and a component only dependent on
$\tilde\Omega_{xx}$.
\begin{proposition}\label{prop:decomp}
Under the transformation $ \tilde\Omega_{xx}= \Omega_{xx} -
\Omega_{yx}^\top \Omega_{yy}^{-1}\Omega_{yx}$ we have
\begin{equation}\label{equat:L_decompose}
L(\Omega_{yy}, \Omega_{yx}, \Omega_{xx}) = \tilde L(\Omega_{yy},
\Omega_{yx}, \tilde\Omega_{xx}) = \Lpa(\Omega_{yy},\Omega_{yx}) +
H(\tilde\Omega_{xx}),
\end{equation}
where $H(\tilde\Omega_{xx})= -\log\det \tilde\Omega_{xx} +
\Tr(\Sigma^n_{xx}  \tilde\Omega_{xx})$ and
\begin{equation}\label{equat:Lp}
\Lpa(\Omega_{yy},\Omega_{yx}) :=-\log\det(\Omega_{yy}) + \Tr
(\Sigma^n_{yy}\Omega_{yy}) + 2\Tr (\Sigma^{n\top}_{yx}\Omega_{yx}) +
\Tr(\Sigma^n_{xx}\Omega_{yx}^\top\Omega_{yy}^{-1}\Omega_{yx}).
\end{equation}
Moreover $\Lpa(\Omega_{yy},\Omega_{yx})$ is convex.
\end{proposition}
The proof of Proposition~\ref{prop:decomp} is provided in
Appendix~\ref{append:proof_lemma1}. Since both
$\Lpa(\Omega_{yy},\Omega_{yx})$ and $H(\tilde\Omega_{xx})$ are
convex, we have that $\tilde L(\Omega_{yy}, \Omega_{yx},
\tilde\Omega_{xx})$ is jointly convex in
$\{\Omega_{yy},\Omega_{yx},\tilde\Omega_{xx}\}$.

The decomposition formulation~\eqref{equat:L_decompose} is the key
idea key behind our new formulation which decouples the optimization
of $\{\Omega_{yy},\Omega_{yx}\}$ and $\tilde\Omega_{xx}$. In the
high dimensional setting, we consider the following penalized
problem using the reparameterized $\Omega$:
\begin{equation}\label{eq:GGM-pGGM}
\{\hat{\Omega}_{yy}, \hat{\Omega}_{yx}, \tilde\Omega_{xx}\} =
\argmin_{\Omega_{yy}\succ 0,\Omega_{yx},\tilde\Omega_{xx}\succ 0} \{
\tilde L(\Omega_{yy}, \Omega_{yx}, \tilde\Omega_{xx}) +
R(\Omega_{yy},\Omega_{yx}) + P(\tilde\Omega_{xx}) \},
\end{equation}
where $R(\Omega_{yy},\Omega_{yx})$ and $P(\tilde\Omega_{xx})$ are
decoupled regularization terms that can guarantee the problem to be
well-defined. Based on~\eqref{equat:L_decompose},
problem~\eqref{eq:GGM-pGGM} can be decomposed into the following two
separate problems:
\begin{eqnarray}
\{\hat{\Omega}_{yy}, \hat{\Omega}_{yx}\} &=&
\argmin_{\Omega_{yy}\succ 0, \Omega_{yx}} \{\Lpa(\Omega_{yy},
\Omega_{yx}) + R(\Omega_{yy} , \Omega_{yx})\}, \label{prob:MLE_Omega_regularized_general}\\
 \tilde\Omega_{xx}
&=& \argmin_{\tilde\Omega_{xx}\succ 0} \{H(\tilde\Omega_{xx}) +
P(\tilde\Omega_{xx})\}. \nonumber \label{prob:MLE_H}
\end{eqnarray}
We call the first equation specified in ~\eqref{prob:MLE_Omega_regularized_general}
as {\em partial Gaussian Graphical Model} or pGGM, which is the main
formulation proposed in this paper.
If we assume that
both $\Omega^*_{yy}$ and $\Omega^*_{yx}$ are sparse, then we may use
sparsity-inducing penalty $R(\Omega_{yy} , \Omega_{yx})$
in~\eqref{prob:MLE_Omega_regularized_general}. For example, the
following two penalties enforce element-wise and column-wise sparsity
respectively:
\begin{itemize}
  \item Element-wise sparsity-inducing penalty: $R_{e}(\Omega_{yy}, \Omega_{yx}) = \lambda_n |\Omega_{yy}^{-}|_1
  + \rho_n |\Omega_{yx}|_1$.
  \item Column-wise sparsity-inducing penalty: $R_c(\Omega_{yy}, \Omega_{yx}) = \lambda_n |\Omega_{yy}^{-}|_1 +
  \rho_n \|\Omega_{yx}\|_{2,1}$ where $\|\Omega_{yx}\|_{2,1}=\sum_{j=1}^q\|(\Omega_{yx})_{\cdot j}\|$.
\end{itemize}
If we use the element-wise sparsity-inducing penalty, then the
resulting formula is similar to $\ell_1$-penalized full Gaussian
graphical model formulation of (\ref{prob:GGM}). The main difference
is that the pGGM
formulation~\eqref{prob:MLE_Omega_regularized_general} does not
depend on $\Omega_{xx}$, and consequently it does not require the
sparsity assumption on $\Omega_{xx}$. One advantage of pGGM is the
significantly reduced computational complexity when $X$ is high
dimensional. Another important merit of pGGM is that it does not
depend on model assumptions of $\Omega^*_{xx}$ because the
optimization has been decoupled. This is analogous to the situation
of conditional random field~\citep{Lafferty-CRF-2001} where we model
the conditional distribution of $Y$ given $X$ directly, and good
model of the distribution of $X$ is unnecessary or ancillary for
discriminative analysis. In particular, as we will demonstrate in
Section~\ref{ssect:simulation}, the formulation performs well even
if $\Omega^*_{xx}$ is relatively dense compared to $\Omega^*_{yy}$
and $\Omega^*_{yx}$.

% Although we do not aim at estimating $\Omega^*_{xx}$, one may still
% get its estimation by solving the problem~\eqref{prob:MLE_H} with
% proper penalty $P$, and estimate $\Omega_{xx}$ using
% $\tilde\Omega_{xx} - \Omega_{yx}^\top \Omega_{yy}^{-1}\Omega_{yx}$.
% However, one disadvantage of the formulation~\eqref{prob:MLE_H} is
% that it is difficult to enforce sparsity on $\Omega_{xx}$ using $P$.

\section{Theoretical Analysis}
\label{sect:analysis}

We now analyze the estimation error between the estimated precision
matrix blocks $\{\hat{\Omega}_{yy}, \hat{\Omega}_{yx}\}$
in~\eqref{prob:MLE_Omega_regularized_general} and the true blocks
$\{\Omega^*_{yy},\Omega^*_{yx}\}$. Let $S_{yy}:=\supp(\Omega^*_{yy})
\cup\{(i,i):i=1,...,p\}$ and $\bar S_{yy}$ be its complement.
Similarly we define $S_{yx}$ and $\bar S_{yx}$. To simplify
notation, we denote $\Theta = (\Omega_{yy},\Omega_{yx})$, $S =
S_{yy}\cup S_{yx}$ and $\bar S = \bar S_{yy}\cup \bar S_{yx}$. The
error of the first-order Taylor expansion of $\Lpa$ at $\Theta$ in
direction $\Delta\Theta$ is
\[
\Delta \Lpa (\Theta, \Delta\Theta):=\Lpa (\Theta + \Delta\Theta) - \Lpa(\Theta) -
\langle \nabla \Lpa (\Theta), \Delta\Theta\rangle.
\]
We introduce the
concept of local restricted strong convexity to bound $\delta
\Lpa(\Theta,\Delta\Theta)$.
\begin{definition}[\textbf{Local Restricted Strong Convexity}] \label{def:cone}
We define the following quantity which we refer to as local
restricted strong convexity (LRSC) constant at $\Theta$:
\[
\beta(\Theta; r,\alpha) = \inf\left\{ \frac{\Delta \Lpa (\Theta,
    \Delta\Theta)}{\|\Delta\Theta\|^2_{F}} :
0< \|\Delta\Theta\|_{F}\le r, |\Delta\Theta _{\bar S}|_1 \le \alpha|\Delta\Theta _{S}|_1 \right\} ,
\]
where $\alpha= 3 \max\{\lambda_n,\rho_n\}/\min\{\lambda_n,\rho_n\}$.
\end{definition}

As will be described in our main result, the
Theorem~\ref{thrm:F_norm_bound}, that the LRSC condition of $\Lpa$
is required to guarantee the statistical efficiency of pGGM. Before
presenting the theorem, we will first show that when $n$ is
sufficiently large, such a condition holds with high probability
under proper conditions. We require the following assumption.
\begin{assumption} \label{assump:rip}
  Assume that the following conditions hold for some integers $\tilde{s}$:
  \begin{align*}
   \inf \left\{  \frac{u^\top \Sigma^n_{xx} u}{u^\top \Sigma^*_{xx} u} : u \neq 0, \|u\|_0 \leq \tilde{s} \right\}  \geq& 0.5 , \\
    \sup \left\{  \frac{u^\top \Sigma^n_{xx} u}{u^\top \Sigma^*_{xx} u} : u \neq 0, \|u\|_0 \leq \tilde{s} \right\}  \leq& 1.5 , \\
    \frac{\lambda_{\max} \big[\Omega^*_{yx} \Sigma^n_{xx}  (\Omega^*_{yx})^\top\big]}
    {\lambda_{\max} \big[\Omega^*_{yx} \Sigma^*_{xx}  (\Omega^*_{yx})^\top\big]}
    \leq & 1.4 .
  \end{align*}
\end{assumption}

The assumption is similar to the RIP condition in compressed sensing.
The following result is known from the compressed sensing
literature~\citep[see][for example]{Baraniuk-SimpleRIP-2008,Rauhut-2008,Candes-D-RIP-2011}.
\begin{proposition}\label{prop:rip}
There exists absolute constants $c_1$ and $c_2$ such that
Assumption~\ref{assump:rip} holds with probability no less than $1- \exp(-c_2 n)$
when $n \geq c_1 (p + \tilde{s} \log (p+q))$.
\end{proposition}

Assumption~\ref{assump:rip} can be used to obtain a bound on
$\beta(\Theta^*,r,\alpha)$.
\begin{proposition}\label{prop:RSC}
Let
\[
\rho_-= 0.5 \min(\lambda_{\max}(\Omega^*_{yy})^{-1},\lambda_{\min}(\Sigma^*_{xx})) ,
\quad
\rho_+= 1.5 \lambda_{\max}(\Sigma^*_{xx}) .
\]
Assume that Assumption~\ref{assump:rip} holds with
$\tilde{s} = |S| + \lceil 4 (\rho_+/\rho_-) \alpha^2 |S|\rceil$.
If
\[
r \leq
\min\left[
0.5 \lambda_{\min}(\Omega^*_{yy}),
0.13 \sqrt{\lambda_{\max} \big[\Omega^*_{yx} \Sigma^*_{xx}   (\Omega^*_{yx})^\top \big]/\rho_+ }
\right] ,
\]
then we have
\[
\beta(\Theta^*,r,\alpha) \ge
\frac{\rho_-}{40 \lambda_{\max}(\Omega^*_{yy})} \cdot
\min\left[2 ,
\frac{\lambda_{\min}(3\Omega^*_{yy})}{8\lambda_{\max}(\Omega^*_{yx} \Sigma^*_{xx}  (\Omega^*_{yx})^\top)}
\right] .
\]
\end{proposition}

The following definition of $\gamma_n$ is also needed in our analysis.
\begin{definition}
Define
\begin{align*}
A_n =& \Sigma^n_{yy} - \Sigma^*_{yy}
-(\Omega^*_{yy})^{-1}\Omega^*_{yx}(\Sigma^n_{xx}-\Sigma^*_{xx})\Omega^{*\top}_{yx}(\Omega^*_{yy})^{-1}\\
B_n =&  2(\Sigma^n_{yx} - \Sigma^*_{yx} +
(\Omega^*_{yy})^{-1}\Omega^*_{yx}(\Sigma^n_{xx} - \Sigma^*_{xx})) , \\
\gamma_n =& \max\{|A_n|_\infty, |B_n|_\infty\} .
\end{align*}
\end{definition}
We have the following estimate of $\gamma_n$.
\begin{proposition} \label{prop:gamma_convergence}
For any $\eta \in (0,1)$, and given the sample size $n \geq \log (10 (p+q)^2/\eta)$, we have with probability
$1-\eta$:
\[
\gamma_n  \le 16 \sqrt{\log (10 (p+q)^2/\eta) /n}   \;
\left[
\max_{i} (\Sigma^*_{ii})+ \max_i (((\Omega^*_{yy})^{-1}\Omega^*_{yx}\Sigma^*_{xx}\Omega^{*\top}_{yx}(\Omega^*_{yy})^{-1})_{ii}) \right] .
\]
\end{proposition}

The following result bounds the Frobenius norm estimation error in terms of $\gamma_n$.
\begin{theorem}\label{thrm:F_norm_bound}
Let $\hat\Theta=(\hat\Omega_{yy},\hat\Omega_{yx})$ be the global
minimizer of \eqref{prob:MLE_Omega_regularized_general} with
element-wise $\ell_1$-penalty $R_e$. Assume that $\lambda_n,\rho_n
\in[2\gamma_n, c_0\gamma_n]$ for some $c_0\ge2$. We further assume
that $\Lpa$ has LRSC at $\Theta^*=(\Omega^*_{yy},\Omega^*_{yx})$
with constant $\beta(\Theta^*;r,\alpha)>0. $ Consider $r_0,
\beta_0>0$ so that $\beta(\Theta^*;r_0,\alpha) \ge \beta_0$. Define
$\Delta_n = 1.5 c_0 \beta_0^{-1} \gamma_n \sqrt{|S|}$. If $\Delta_n
< r_0$, then
\[
\|\hat\Theta - \Theta^*\|_{F} \le 1.5 c_0 \beta_0^{-1} \gamma_n \sqrt{|S|} .
\]
\end{theorem}

The following corollary is easier to interpret than Theorem~\ref{thrm:F_norm_bound}.
\begin{corollary}
Let $\hat\Theta=(\hat\Omega_{yy},\hat\Omega_{yx})$ be the global
minimizer of \eqref{prob:MLE_Omega_regularized_general} with
element-wise $\ell_1$-penalty $R_e$. Assume that $\lambda_n,\rho_n
\in[2\gamma_n, c_0\gamma_n]$ for some $c_0\ge2$.
Define
\begin{align*}
\beta_0=&
\frac{\rho_-}{40 \lambda_{\max}(\Omega^*_{yy})} \cdot
\min\left[2 ,
\frac{3 \lambda_{\min}(\Omega^*_{yy})}{8\lambda_{\max}(\Omega^*_{yx} \Sigma^*_{xx}  (\Omega^*_{yx})^\top)}
\right] , \\
r_0 = &
\min\left[
0.5 \lambda_{\min}(\Omega^*_{yy}),
0.13 \sqrt{\lambda_{\max} \big[\Omega^*_{yx} \Sigma^*_{xx}   (\Omega^*_{yx})^\top \big]/\rho_+ }
\right] , \\
\gamma_0 =& 16 \left[
\max_{i} (\Sigma^*_{ii})+ \max_i (((\Omega^*_{yy})^{-1}\Omega^*_{yx}\Sigma^*_{xx}\Omega^{*\top}_{yx}(\Omega^*_{yy})^{-1})_{ii}) \right] .
\end{align*}
Let $c_1$ and $c_2$ be absolute constants in Proposition~\ref{prop:rip}.
If $n$ is sufficiently large so that
\[
n > \max\left[c_1 (p + \tilde{s} \log (p+q)), \log (10 (p+q)^2/\eta)
, (1.5 c_0 \gamma_0)^2(r_0 \beta_0)^{-2} |S| \log(10(p+q)^2/\eta)
\right]
\]
with $\tilde{s} = |S| + \lceil 4 (\rho_+/\rho_-) \alpha^2 |S|\rceil $,
then with probability no less than $1- \exp(-c_2 n)-\eta$,
\[
\|\hat\Theta - \Theta^*\|_{F} \le 1.5 c_0 \beta_0^{-1} \gamma_0 \sqrt{|S|\log (10 (p+q)^2/\eta) /n} .
\]
\end{corollary}
\begin{proof}
Since $n \geq c_1 (p + \tilde{s} \log (p+q))$,  with probability no less than $1- \exp(-c_2 n)-\eta$,
both Assumption~\ref{assump:rip} hold and
Proposition~\ref{prop:gamma_convergence} are valid.

Since Assumption~\ref{assump:rip} holds,
Proposition~\ref{prop:RSC} implies
$\beta(\Theta^*,r_0,\alpha) \ge \beta_0$.
Since $n \geq \log (10 (p+q)^2/\eta)$,
Proposition~\ref{prop:gamma_convergence} implies that
$\gamma_n \leq \sqrt{\log (10 (p+q)^2/\eta) /n}   \gamma_0$.
Therefore the assumption of $n$ implies that
$\Delta_n \leq 1.5 c_0 \beta_0^{-1} \gamma_0 \sqrt{|S|\log (10 (p+q)^2/\eta) /n} < r_0$,
and Theorem~\ref{thrm:F_norm_bound} implies that
$\|\hat\Theta - \Theta^*\|_{F} \le \Delta_n$.
\end{proof}

We may assume that $\beta_0$, $r_0$, and $\gamma_0$ to be $O(1)$
constants that depend on $\Omega^*$ and $\Sigma^*$. The corollary
implies that when $n$ is at least the order of $p+ |S|
\log((p+q)/\eta)$, then
\[
\|\hat\Theta - \Theta^*\|_{F} = O(\sqrt{|S|\log ((p+q)/\eta) /n}) .
\]

\section{Numerical Algorithm}
\label{sect:algorithm}

We present a coordinate descent procedure to solve the pGGM
problem~\eqref{prob:MLE_Omega_regularized_general}. The algorithm
alternates between solving the following two subproblems on
$\Omega_{yy}$ and $\Omega_{yx}$ respectively:
\begin{eqnarray}
\Omega_{yy}^{(t+1)}&=& \argmin_{\Omega_{yy}\succ 0}
\left[
\Lpa(\Omega_{yy}, \Omega_{yx}^{(t)}) + R(\Omega_{yy},
\Omega_{yx}^{(t)})
\right] , \label{subprob:omega_yy_t+1} \\
\Omega_{yx}^{(t+1)} &=& \argmin_{\Omega_{yx} }
\left[
\Lpa(\Omega_{yy}^{(t+1)}, \Omega_{yx}) + R(\Omega_{yy}^{(t+1)} ,
\Omega_{yx})
\right] . \label{subprob:omega_xy_t+1}
\end{eqnarray}
Since the objective is convex, it is guaranteed that the above
procedure converges to the global minimum. Let us first consider the
minimization problem~\eqref{subprob:omega_yy_t+1}. This is
equivalent to
\begin{equation}\label{equat:Omg_yy_sub}
\Omega_{yy}^{(t+1)} = \argmin_{\Omega_{yy}\succ 0}
\left[
F^{(t)}(\Omega_{yy})+ R(\Omega_{yy}, \Omega_{yx}^{(t)})
\right]
,
\end{equation}
where
\[
F^{(t)}(\Omega_{yy}):=-\log\det(\Omega_{yy}) +
\Tr(\Sigma^n_{yy}\Omega_{yy}) +
\Tr(\Sigma^n_{xx}(\Omega_{yx}^{(t)})^\top\Omega_{yy}^{-1}\Omega_{yx}^{(t)})
.
\]
In our implementation, the proximal gradient descent
method~\citep{Nesterov-2005,Beck-2009} is utilized to solve the
above composite optimization problem, where the gradient of the
first (smooth) term of \eqref{equat:Omg_yy_sub} is given by
\[
\nabla F^{(t)}(\Omega_{yy}) = -\Omega_{yy}^{-1} + \Sigma^n_{yy} -
\Omega_{yy}^{-1}\Omega_{yx}^{(t)}
\Sigma^n_{xx}(\Omega_{yx}^{(t)})^\top\Omega_{yy}^{-1} .
\]
Next, we consider the minimization
problem~\eqref{subprob:omega_xy_t+1}. This is equivalent to
\begin{equation}\label{equat:Omg_yx_sub}
\Omega_{yx}^{(t+1)} = \argmin
\left[ G^{(t)}(\Omega_{yx})+R(\Omega_{yy}^{(t+1)}, \Omega_{yx})
\right] ,
\end{equation}
where
\[
G^{(t)}(\Omega_{yx}):=\Tr(\Sigma^n_{xx}\Omega_{yx}^\top(\Omega_{yy}^{(t+1)})^{-1}\Omega_{yx})
+ 2\Tr (\Sigma_{yx}^{n\top}\Omega_{yx}) .
\]
Again, we apply the proximal gradient method to solve this
subproblem. Here the gradient of the first (smooth) term of
\eqref{equat:Omg_yx_sub} is given by
\[
\nabla G^{(t)}(\Omega_{yx}) = 2(\Omega_{yy}^{(t+1)})^{-1}\Omega_{yx}
\Sigma^n_{xx} + 2\Sigma^n_{yx} .
\]
The computational complexity in terms of $p$ and $q$ for this
coordinate descent algorithm is as follows: (1) $O(p^3 + p^2q +
pq\min\{n,q\})$ for the subproblem~\eqref{subprob:omega_yy_t+1} due
to the inverse of $\Omega_{yy}$ and the matrix product in the
evaluation of gradient $\nabla F^{(t)}(\Omega_{yy})$; and (2)
$O(p^2q + pq\min\{n,q\})$ for the
subproblem~\eqref{subprob:omega_xy_t+1} from matrix product in
evaluating gradient $\nabla G^{(t)}(\Omega_{yx})$. Therefore, the
overall complexity of the proposed algorithm is $O(p^3 + p^2q +
pq\min\{n,q\})$. This can be compared to the $O((p+q)^3)$ or higher
per iteration complexity required by well known representative
algorithms for full precision matrix
estimation~\citep{Friedman-Glasso-2008,Aspremont-2008,Rothman-2008,Lu-VSM-2009}.
In the high dimensional setups where $q \gg \max\{n, p\}$, the
computational advantage of pGGM over standard GGMs can be
significant.

\section{pGGM for Multivariate Regression with Unknown Covariance}
\label{sec:Multivariate_Regression}

In this section, we show that pGGM provides a convex
formulation for solving the following model of multivariate
regression with unknown noise covariance:
\begin{equation}\label{prob:multivariate_regression}
Y = \Gamma_{yx}^* X + \bar\varepsilon_y,
\end{equation}
where $Y \in\mathbb{R}^p$, $X \in \mathbb{R}^q$, $\Gamma_{yx}^*$ is
a $p \times q$ regression coefficient matrix and the random noise
vector $\bar\varepsilon_y \sim \cN(0, (\bar\Omega_{yy}^*)^{-1})$ is
independent of $X$. Our interest is in the simultaneous estimation
of $\Gamma_{yx}^*$ and $\bar\Omega^*_{yy}$ from observations
$\{Y^{(i)};X^{(i)}\}_{i=1}^n$ in the high-dimensional setting. Note
that for this regression problem we do not have to assume the joint
normality of $(Y;X)$, but rather that the noise term is normal (or
more generally sub-Gaussian). Our discussion in this section is
based on the fact that pGGM is a regularized maximum likelihood
estimator for multivariate regression with Gaussian noise.

\subsection{pGGM as a Conditional Maximum Likelihood Estimator}

We will start our discussion under the joint Gaussian setup, which
provides the connection of the pGGM formulation and multivariate
regression. Let the true covariance matrix $\Sigma^*$ be partitioned
into blocks
\[
\Sigma^* = \left[ { {\begin{array}{cc}
   {\Sigma^*_{yy}}   & {\Sigma^*_{yx}}   \\
   {\Sigma^{*\top}_{yx}} & {\Sigma^*_{xx}}   \\
\end{array}}
} \right].
\]
Here we assume that $(Y;X)$ is jointly normal, the conditional
distribution of $Y$ given $X$, given as follows, remains normal:
\begin{equation}\label{equat:cond_distr_ori}
Y\mid X \sim \cN \left(\Sigma^*_{yx}(\Sigma_{xx}^{*})^{-1} X ,
\Sigma^*_{yy} - \Sigma^*_{yx}(\Sigma^*_{xx})^{-1}\Sigma^{*\top}_{yx}
\right) .
\end{equation}
Now by using algebra for block matrix inversion, we may write the
precision matrix $\Omega^*=(\Sigma^*)^{-1}$ as
\begin{equation}
\Omega^* = \left[ { {\begin{array}{*{20}{c}}
   {\left(\Sigma^*_{yy} - \Sigma_{yx}^{*}(\Sigma^*_{xx})^{-1}\Sigma_{yx}^{*\top}\right)^{-1}} & -\left(\Sigma^*_{yy} - \Sigma_{yx}^{*}(\Sigma^*_{xx})^{-1}\Sigma_{yx}^{*\top}\right)^{-1}\Sigma^*_{yx}(\Sigma^*_{xx})^{-1}   \\
   -(\Sigma^*_{xx})^{-1}\Sigma_{yx}^{*\top}\left(\Sigma^*_{yy} - \Sigma_{yx}^{*}(\Sigma^*_{xx})^{-1}\Sigma_{yx}^{*\top}\right)^{-1}   & \Box   \\
\end{array}}
} \right],\nonumber
\end{equation}
and thus
\begin{equation}\label{equat:omega_sigma_connection}
\Omega^*_{yy} = \left(\Sigma^*_{yy} -
\Sigma_{yx}^*(\Sigma^*_{xx})^{-1}\Sigma^{*\top}_{yx}\right)^{-1},
\quad \Omega^*_{yx} =
-\Omega^*_{yy}\Sigma^*_{yx}(\Sigma^*_{xx})^{-1} .
\end{equation}
Therefore the conditional distribution \eqref{equat:cond_distr_ori}
can be rewritten as:
\begin{equation}\label{equat:cond_distr}
Y \mid X \sim \cN (-(\Omega^*_{yy})^{-1}\Omega^*_{yx} X,
(\Omega^*_{yy})^{-1}). \nonumber
\end{equation}
This can be equivalently expressed as the following multivariate
regression model:
\begin{equation}\label{prob:cond_regression}
Y = -(\Omega^*_{yy})^{-1}\Omega^*_{yx} X + \varepsilon_y,
\end{equation}
where $\varepsilon_y \sim \cN(0, (\Omega^*_{yy})^{-1})$ is
independent of $X$. Note that this model can be regarded as
a reparameterization of the standard multivariate regression model in \eqref{prob:multivariate_regression}.
It is easy to verify that given the observations
$\{Y^{(i)};X^{(i)}\}_{i=1}^n$, the negative of the conditional log-likelihood
function for $\varepsilon_y$ is written by
\[
-\log\det(\Omega^*_{yy}) + \Tr (\Sigma^n_{yy}\Omega_{yy}^*) + 2\Tr
(\Sigma^{n\top}_{yx}\Omega^*_{yx}) +
\Tr(\Sigma^n_{xx}\Omega_{yx}^{*\top}(\Omega^*_{yy})^{-1}\Omega^*_{yx}).
\]
which is exactly $\Lpa(\Omega_{yy}^*, \Omega_{yx}^*)$ given
by~\eqref{equat:Lp}. Therefore, pGGM is essentially a regularized conditional
maximum likelihood estimator for the regression
model~\eqref{prob:cond_regression}.
This implies that we can use pGGM to solve multivariate regression
problem with unknown noise covariance matrix $\Omega_{yy}$.

\subsection{Convexity and cGGM}

We now consider the general multivariate regression
model~\eqref{prob:multivariate_regression} with Gaussian noise. A
more straightforward method for estimating the model parameters
$\{\bar\Omega_{yy}^*,\Gamma_{yx}^*\}$ was considered
by~\citet{Li-cGGM} using the following $\ell_1$-regularized
log-likelihood function associated with $\bar\varepsilon_y$:
\begin{equation}\label{equat:cGGM_objective}
\{\hat\Omega_{yy},\hat\Gamma_{yx}\}=\argmin_{\bar\Omega_{yy} \succ
0, \Gamma_{yx}} \left\{ -\log\det\bar\Omega_{yy} +
\Tr(\Sigma^n_{\Gamma_{yx}}\bar\Omega_{yy}) + \lambda_n
|(\bar\Omega_{yy})^-|_1 + \rho_n |\Gamma_{yx}|_1\right\},
\end{equation}
where $\Sigma^n_{\Gamma_{yx}} = \Sigma^n_{yy} -
\Sigma^n_{yx}\Gamma_{yx}^\top -\Gamma_{yx} \Sigma^{n\top}_{yx}
+\Gamma_{yx} \Sigma^n_{xx}\Gamma_{yx}^\top$.
However, with this formulation, the objective function in
\eqref{equat:cGGM_objective} is not jointly convex in $\Gamma_{yx}$
and $\bar\Omega_{yy}$, although it is convex with respect to
$\Gamma_{yx}$ for any fixed $\bar\Omega_{yy}$, and it is also convex
respective to $\bar\Omega_{yy}$ for any fixed $\Gamma_{yx}$.

In contrast, the expression \eqref{prob:cond_regression} is jointly
convex in $\{\Omega_{yy},\Omega_{yx}\}$, which may be regarded
as a convex reparameterization of \eqref{prob:multivariate_regression}
under the following transformation:
\begin{equation}\label{equat:re_parametrization}
\bar\Omega_{yy} = \Omega_{yy} \quad, \quad \Gamma_{yx} =
-\Omega_{yy}^{-1}\Omega_{yx}. \nonumber
\end{equation}
This transformation yields a one-to-one mapping from
$\{\bar\Omega_{yy},\Gamma_{yx}\}$ to $\{\Omega_{yy}, \Omega_{yx}\}$.
The convexity of \eqref{prob:cond_regression} is desirable both for
optimization and for theoretical analysis which we considered in
Section~\ref{sect:analysis}.

It is worth mentioning that for high dimensional problems,
regularization has to be imposed on the model parameters. With
regularization, the pGGM regression formulation
\eqref{prob:cond_regression} becomes
\eqref{prob:MLE_Omega_regularized_general}, which is different from
the cGGM formulation of \eqref{equat:cGGM_objective}. This is
because for pGGM, the $\ell_1$-norm penalties are imposed on
$\{\Omega_{yy},\Omega_{yx}\}$, and for cGGM, the $\ell_1$-norm
penalties have to be directly imposed on
$\{\bar\Omega_{yy},\Gamma_{yx}\}$. The former has a natural
interpretation in terms of the conditional dependency between the
variables in $X$ and $Y$, while the latter does not have such an
intuitive interpretation.

\subsection{Univariate Case}

As a special case, when the output $Y$ is univariate, pGGM reduces
to a regularized maximum likelihood estimator for high-dimensional
linear regression with unknown variance. In this case, by replacing
the scalar $\Omega_{yy}$ and the row vector $\Omega_{yx}$ with
$\omega$ and $\theta^\top$ respectively in the pGGM formulation
\eqref{prob:MLE_Omega_regularized_general}, with element-wise
$\ell_1$-penalty $R_e$, we arrive at the following estimator:
\begin{equation}\label{prob:MLE_Omega_neighborhood}
\{\hat{\omega}, \hat{\theta}\} = \argmin_{\omega>0, \theta}
\Lpa(\omega, \theta) + \rho \|\theta\|_1,
\end{equation}
where
\[
\Lpa(\omega, \theta):= -\log(\omega) + \Sigma^n_{yy}\omega +
2\theta^\top \Sigma^n_{xy} + \theta^\top \Sigma^n_{xx}\theta/\omega.
\]
As aforementioned that this is identical to a regularized maximum
likelihood estimator for the following linear regression model with
unknown variance:
\begin{equation}\label{prob:univariate_regression}
Y = -\omega^{-1}\theta^\top X + \varepsilon,
\end{equation}
where $\varepsilon \sim \cN(0, \omega^{-1})$ is independent of $X$.
The specific $\ell_1$-penalized maximum likelihood
estimator~\eqref{prob:MLE_Omega_neighborhood} has also been studied
by~\citet{Stadler-2010-mixture} for sparse linear regression with
unknown noise covariance. For multivariate random vector $Y$, pGGM
can be regarded as a multivariate generalization of the method in
\citep{Stadler-2010-mixture}.

For graphical model estimation, pGGM with univariate $Y$ can also be
regarded as a variant of the neighborhood selection
method~\citep{Meinshausen-NSLasso-2006}. Let us write $\Omega_{jj}$
the entry of $\Omega$ at the $j$th row and the $j$th column, and
denote by $\Omega_{j,-j}$ or $\Omega_{-j,j}$ the $j$th row of
$\Omega$ with its $j$th entry removed or the $j$th column with its
$j$th entry removed respectively. In order to recover the non-zero
entries in $\Omega$, \citet{Meinshausen-NSLasso-2006} proposed to
solve for each row $j$ a Lasso problem:
\begin{equation}\label{prob:neighborhood_selection}
\hat\theta =\argmin_{\theta} \theta^\top \Sigma^n_{-j,-j}\theta +
2\theta^\top \Sigma^n_{-j,j} + \rho \|\theta\|_1.
\end{equation}
If we fix $\omega=1$ in~\eqref{prob:MLE_Omega_neighborhood}, then
the resultant estimator is identical
to~\eqref{prob:neighborhood_selection}. For precision matrix
estimation, our formulation \eqref{prob:MLE_Omega_neighborhood} is
different from neighborhood
selection~\eqref{prob:neighborhood_selection} due to the inclusion
of $\omega$ as an unknown parameter. More precisely, the quantity
$\omega^{-1}$ is the noise variance for the corresponding Lasso
regression, and the estimator~\eqref{prob:MLE_Omega_neighborhood}
may be regarded as an extension of neighborhood selection without
knowing the noise variance. For multivariate random vector $Y$, pGGM
can be regarded as a blockwise generalization of neighborhood
selection for graphical model estimation.

For precision matrix estimation, the regression
model~\eqref{prob:univariate_regression} has also been considered
by~\citet{Yuan-JMLR-2010}. However, the author suggested a procedure
to estimate $\theta$ via the
Dantzig-selector~\citep{CandTao:2005:dantzig} followed by a mean
squared error estimator for the variance $\omega^{-1}$. In contrast,
the pGGM based estimator~\eqref{prob:MLE_Omega_neighborhood}
simultaneously estimates the two parameters under a joint convex
optimization framework.

\section{Experiments}
\label{sect:expriment}

In this section, we investigate the empirical performance of the
pGGM estimator on both synthetic and real datasets and compare its
performance to several representative approaches for sparse
precision matrix estimation.

\subsection{Monte Carlo Simulations}
\label{ssect:simulation}

In the Monte Carlo simulation study, we investigate parameter
estimation and support recovery accuracy as well as algorithm
efficiency using synthetic data for which we know the ground truth.
\subsubsection{Data}
Our simulation study employs a precision matrix $\Omega^*$ whose
sub-matrices $\Omega^*_{yy}$ and $\Omega^*_{yx}$ are sparse, while
$\Omega^*_{xx}$ is dense. The matrix is generated as follows: we
first define $\tilde\Omega^* = M + \sigma I$, where each
off-diagonal entry in $M$ is generated independently and equals 1
with probability $P = 0.1$ or 0 with probability $1-P=0.9$. $M$ has
zeros on the diagonal, and $\sigma$ is chosen so that the condition
number of $\Omega^*$ is $p+q$. We then add the $ q \times q$ all-one
matrix to the block $\tilde\Omega^*_{xx}$ and the resultant matrix
is defined as $\Omega^*$. We generate a training sample of size $n$
from $\cN(0, \Sigma^*)$ and an independent sample of size
$n$ from the same distribution for validating the tuning parameters.
The goal is to estimate the sparse blocks
$\{\Omega^*_{yy},\Omega^*_{yx}\}$. We fix $(n,p)=(100,50)$ and
compare the performance under increasing values of $q = 50, 100,
200, 500$, replicated 50 times each.

\subsubsection{Comparing Methods and Evaluation Metrics}

We compare the performance of pGGM to the following three
representative approaches for sparse precision matrix estimation:

\begin{itemize}
\item cGGM for conditional Gaussian graphical model estimation~\citep{Li-cGGM}. After recovering the regression parameters
$\hat\Gamma_{yx}$ and the conditional precision matrix
$\hat\Omega_{yy}$, we estimate the block $\hat\Omega_{yx} = -
\hat\Omega_{yy}\hat\Gamma_{yx}$.
\item GLasso for $\ell_1$-penalized precision matrix estimation~\citep{Friedman-Glasso-2008}. We conventionally apply GLasso to estimate
the full precision matrix $\hat\Omega$.
\item NSLasso for support
recovery~\citep{Meinshausen-NSLasso-2006}. We use a modified version
to recover the supports in the blocks $\Omega^*_{yy}$ and
$\Omega^*_{yx}$ by regressing each $Y_i$ on $Y_{-i}$ and $X$ via the
Lasso. Such a modified neighborhood selection method has also been
adopted by~\citet{Li-cGGM} for their empirical study. Note that this
method does not provide an estimate of the precision matrix.
\end{itemize}
For all methods, we use the validation set to
estimate the values of the regularization parameters.

We measure the parameter estimation quality of
$\hat\Theta=(\hat\Omega_{yy}, \hat\Omega_{yx})$ by its Frobenius
norm distance to $\Theta^*=(\Omega^*_{yy}, \Omega^*_{yx})$. To evaluate
the support recovery performance, we use the F-score from
the information retrieval literature. Note that precision, recall,
and F-scores are standard concepts in information retrieval defined
as follows:
\begin{eqnarray}
\text{Precision} =&\text{TP} /(\text{TP} + \text{FP}) \nonumber \\
\text{Recall} =& \text{TP}/(\text{TP} + \text{FN}) \nonumber\\
\text{F-score} =& \frac{2\cdot \text{Precision} \cdot
\text{Recall}}{\text{Precision} + \text{Recall}}, \nonumber
\end{eqnarray}
where TP stands for true positives (for nonzero entries), and FP and
FN stand for false positives and false negatives. Since one can generally
trade-off precision and recall by increasing one and decreasing the
other, a common practice is to use the F-score as a single metric
to evaluate different methods. The larger the F-score, the better
the support recovery performance.
% The numerical values over $10^{-2}$ in magnitude are considered to be nonzero.

\subsubsection{Results}

Figure~\ref{fig:frobenius}, \ref{fig:fscore}, \ref{fig:cpu} plot the
mean and standard errors of the above metrics as a function of
dimensionality $q$. The results show the following:

\begin{itemize}
  \item Parameter estimation accuracy (see Figure~\ref{fig:frobenius}): pGGM and cGGM perform favorably to
  GLasso. This is expected because GLasso enforces the sparsity of the
  full precision matrix and thus tends to select a smaller
regularization parameter due to the dense structure of block
$\Omega^*_{xx}$. In contrast, pGGM and cGGM exclude $\Omega_{xx}$ in
the model and thus avoid potential under penalization of sparsity.
pGGM and cGGM perform comparably on parameter estimation accuracy.
Note that NSLasso does not estimate the precision matrix.
  \item Support recovery (see Figure~\ref{fig:fscore}): pGGM achieves
    the best performance among all four methods being compared. pGGM outperforms cGGM since the former directly
enforces the sparsity on blocks $\Omega_{yy}$ and $\Omega_{yx}$
while the latter enforces the sparsity of $\Gamma_{yx}
=-\Omega_{yy}^{-1}\Omega_{yx}$ which is not necessarily sparse. GLasso
is inferior due to the under penalization. We also observe that pGGM is
slightly better than NSLasso.
  \item Computational efficiency (see Figure~\ref{fig:cpu}): The pGGM
    and cGGM methods can achieve $\times 100$ speedup over GLasso when $q=500$.
  %The modified NSLasso is faster than GLasso as $q$ increases but inferior to pGGM and  cGGM.
\end{itemize}

\begin{figure}
\centering \subfigure[Frobenius norm loss
($\downarrow$)]{\includegraphics[width=52mm]{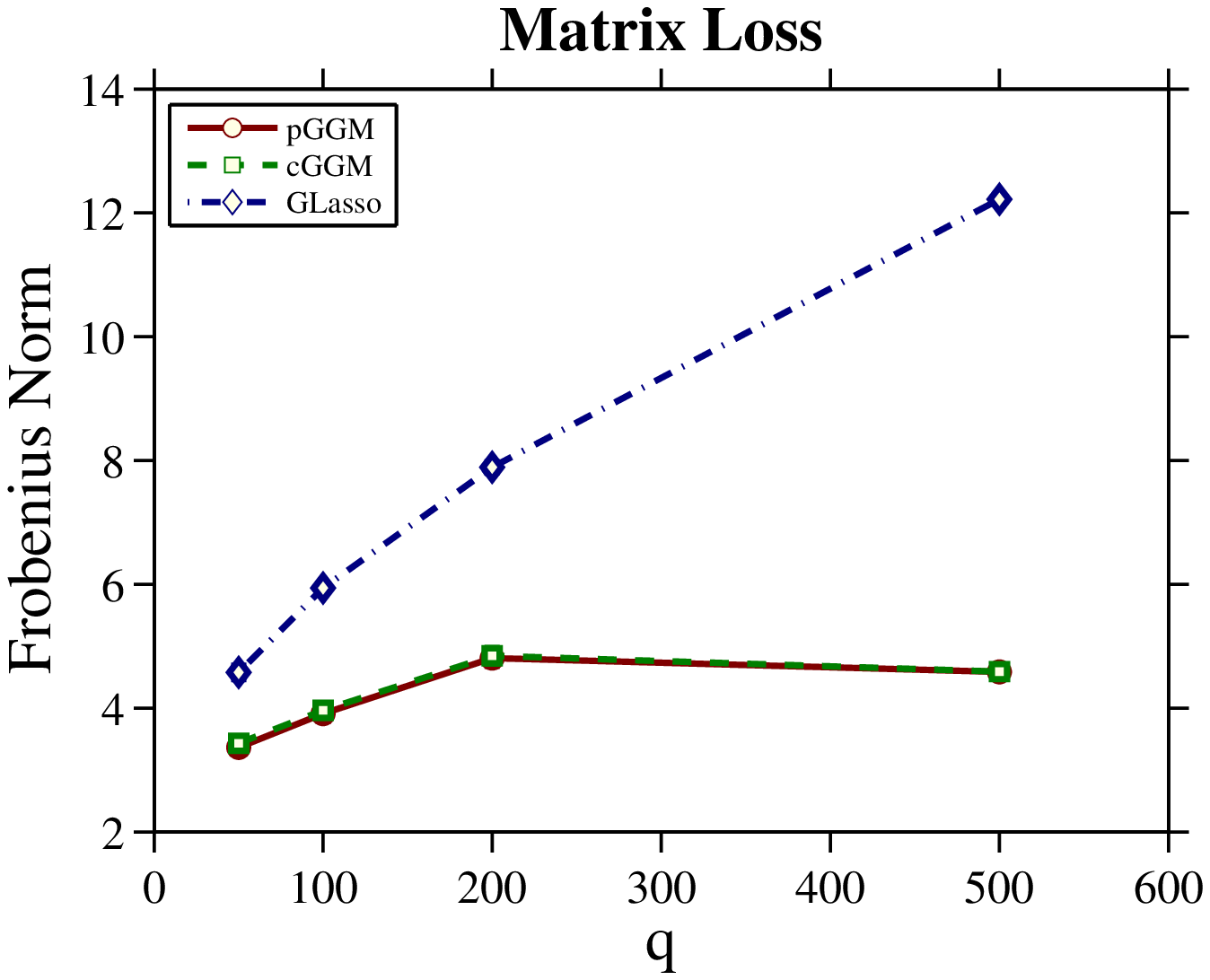}\label{fig:frobenius}}
\subfigure[Support recovery F-score
($\uparrow$)]{\includegraphics[width=52mm]{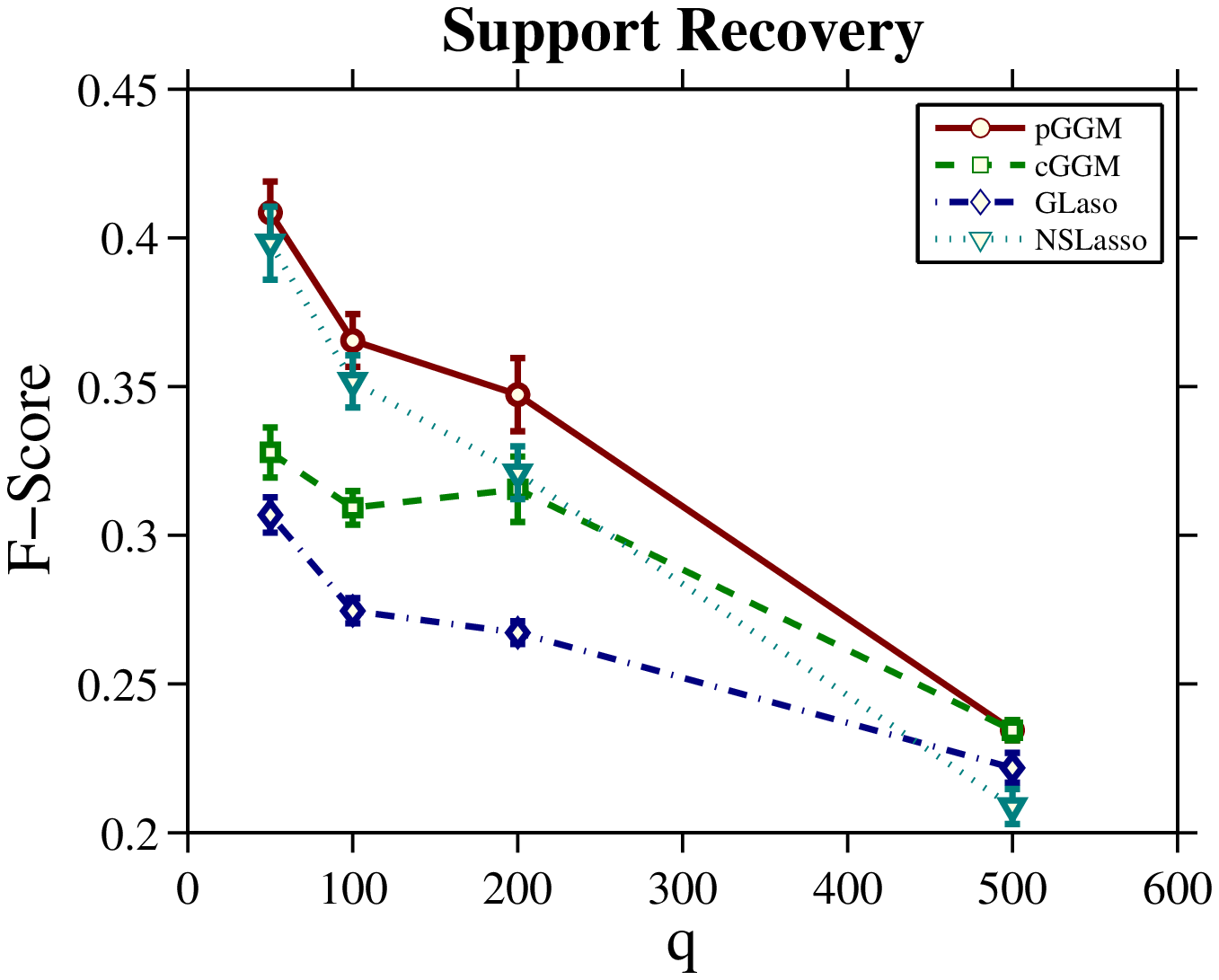}\label{fig:fscore}}
\subfigure[CPU running time
($\downarrow$)]{\includegraphics[width=52mm]{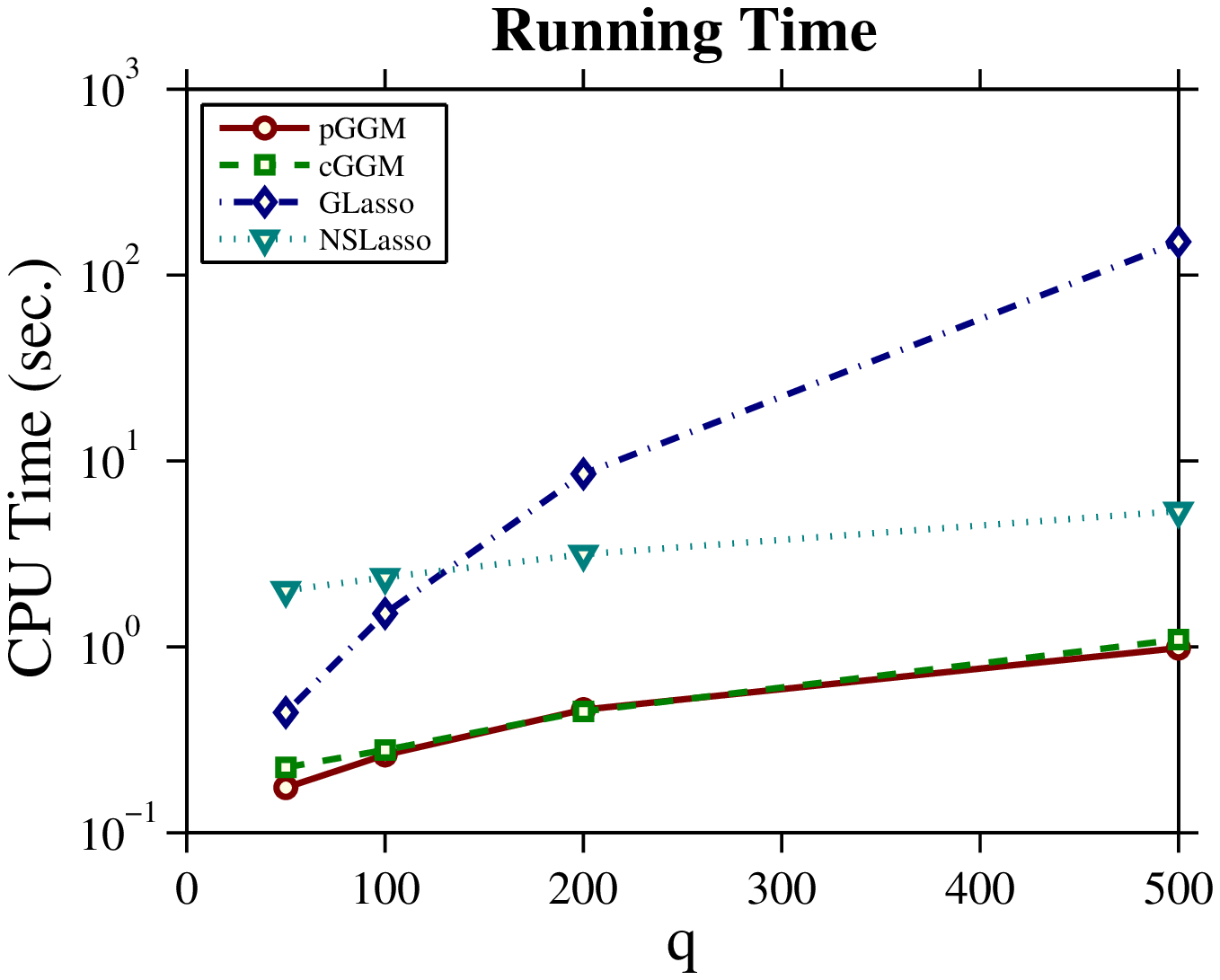}\label{fig:cpu}}
\subfigure[Frobenius norm loss
($\downarrow$)]{\includegraphics[width=52mm]{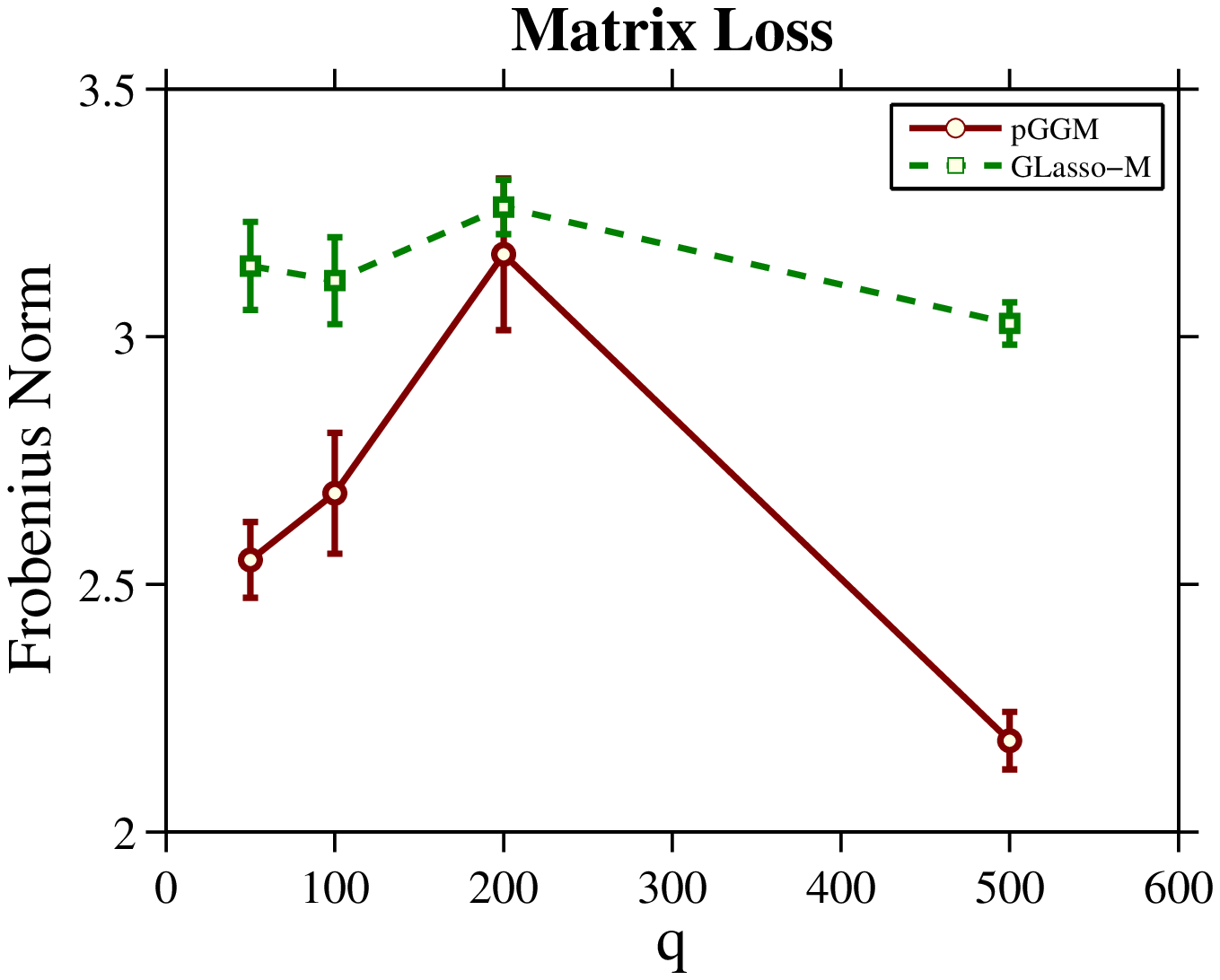}\label{fig:frobenius_y}}
\subfigure[Support recovery F-score
($\uparrow$)]{\includegraphics[width=52mm]{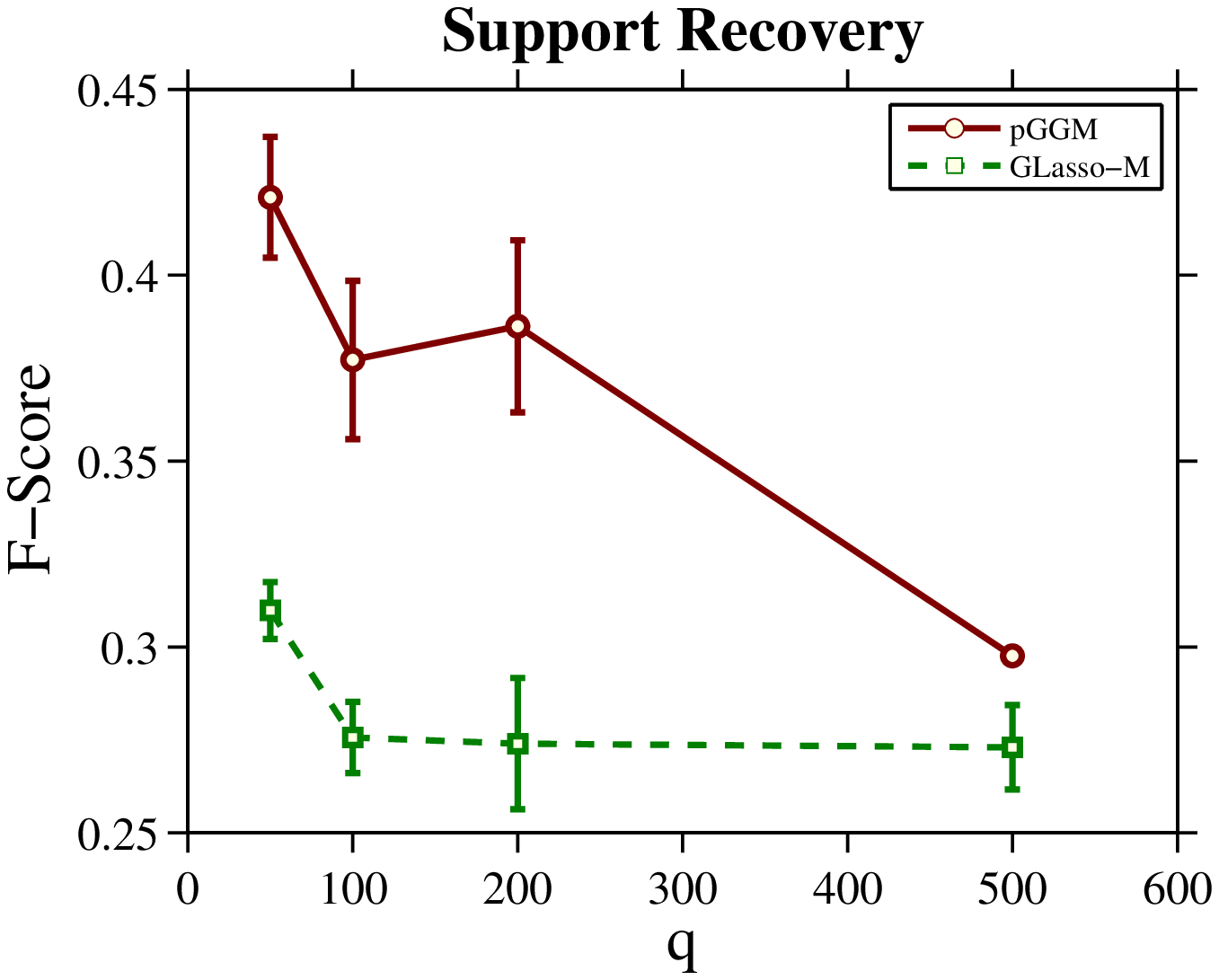}\label{fig:fscore_y}}
\subfigure[CPU running time
($\downarrow$)]{\includegraphics[width=52mm]{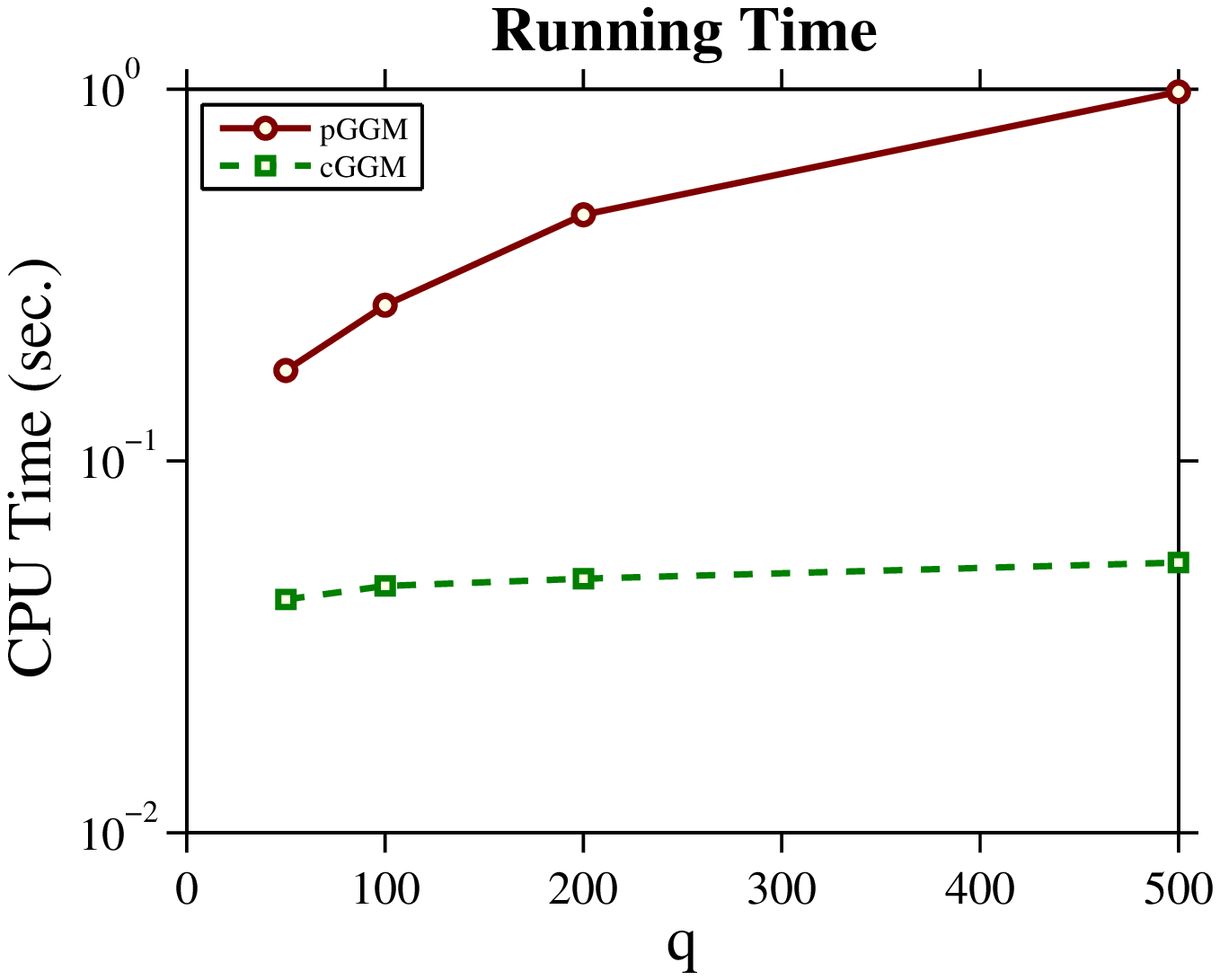}\label{fig:cpu_y}}
\caption{Performance curves on the synthetic data. Top row:
comparison of the estimated blocks
$\{\hat\Omega_{yy},\hat\Omega_{yx}\}$. Bottom row: comparison of the
estimated $\hat\Omega_{yy}$ by pGGM and GLasso-M. The down-arrow
$\downarrow$ means the smaller the better while the up-arrow
$\uparrow$ means the larger the better.\label{fig:synthetic_loss}}
\end{figure}

We further compare pGGM to GLasso applied to the marginal
distribution of $Y$ by ignoring $X$. We call this method as
GLasso-M. The results are plotted in Figure~\ref{fig:frobenius_y},
\ref{fig:fscore_y}, \ref{fig:cpu_y}. It can be observed from these
figures that pGGM consistently outperforms GLasso-M in terms of
parameter estimation and support recovery accuracies.

The detailed performance figures that are used to generate
Figure~\ref{fig:synthetic_loss} are presented in
Appendix~\ref{append:results_montecarlo} in tabular forms, along
with additional performance metrics in spectral norm and matrix
$\ell_1$-norm. The observations using the other norms are consistent
with that of the Frobenius norm.

\subsection{Real Data}

We further study the performance of pGGM on real data.

\subsubsection{Data}
We use three multi-label datasets \textsf{Corel5k},
\textsf{MIRFlickr25k} and \textsf{RCV1-v2} and a stock price dataset
\textsf{S\&P500} for this study. For each dataset, we generate a
training sample for parameter estimation and an independent test
sample for evaluation. Table~\ref{tab:data_statistics} summarizes
some statistics of the data. We next describe the derails of these
datasets.

\noindent\textbf{\textsf{Corel5k}.}  This dataset was first used
in~\citep{Corel5k-2002}. Since then, it has become a standard
benchmark for keyword based image retrieval and image annotation. It
contains around 5,000 images manually annotated with 1 to 5
keywords. The vocabulary contains 260 visual words. The average
number of keywords per sample is 3.4 and the maximum number of
keywords per sample is 5. The data set along with the extracted
visual features are publicly available at
\url{lear.inrialpes.fr/people/guillaumin/data.php}. In our
experiment, we down sample the training data to size 450 for
constructing the Gaussian graphical models of image keywords. For
evaluation purpose, an independent test set of size 450 is selected.
Each image is described by the GIST feature which has dimensionality
512. Our goal is to construct a graphical model for image tags. Note
that the size of label-feature joint variable is $260+512 = 772$,
which allows us to examine the performance when $p + q
> n$.

\noindent\textbf{\textsf{MIRFlickr25k}.} This data contains 25,000
images collected from Flickr over a period of 15 months. The
database is available at~\url{press.liacs.nl/mirflickr/}. The
collection contains highest scored images according to Flickr's
``interestingness'' score. These images were annotated for 24
concepts, including object categories but also more general scene
elements such as \emph{sky}, \emph{water} or \emph{indoor}. For 14
of the 24 concepts, a second and more strict annotation was made.
The vocabulary contains 457 tags. The average number of words per
sample is 2.7 and the maximum words per sample is 32. The data set
along with the extracted visual features are publicly available at
\url{lear.inrialpes.fr/people/guillaumin/data.php}. In our
experiment, we down sample the training set to size 1,250 for
constructing the Gaussian graphical models of image keywords. For
evaluation purpose, an independent test set of size 1,250 is
selected. Each image is described by the GIST feature of dimension
512. Our goal is to construct a graphical model for image tags.

\noindent\textbf{\textsf{RCV1-v2}.} This data set contains newswire
stories from Reuters Ltd~\citet{Lewis-2004}. Several schemes were
utilized to process the documents including removing stopping words,
stemming, and transforming each document into a numerical vector.
There are three sets of categories: \emph{Topics}, \emph{Industries} and
\emph{Regions}. In this paper, we consider the \emph{Topics}
category set, and make use of a subset collection (sample size 3,000,
feature dimension 47,236) of this data
from~\url{www.csie.ntu.edu.tw/~cjlin/libsvmtools/datasets}. We
further down sample the data set to a size of 1,000, and select the top
1,000 words with highest TF-IDF frequencies. For evaluation purpose,
an independent test set of size 1,000 is selected. The vocabulary
contains 103 keywords. The average number of words per sample is 3.3
and the maximum words per sample is 12. Our goal is to construct a
graphical models for these keywords.

\noindent\textbf{\textsf{S\&P500}.} We investigate the historical
prices of \textsf{S\&P500} stocks over 5 years, from January 1, 2007
to January 1, 2012. By taking out the stocks with less than 5 years
of history, we end up with 465 stocks, each having daily closing
prices over 1,260 trading days. The prices are first adjusted for
dividends and splits and the used to calculate daily log returns.
Each day's return can be represented as a point in $\bR^{465}$. For
each day's return, we chose the first 300 as $X$ and the rest $165$
as $Y$. We down sample the data set to size 101. For evaluation
purpose, an independent test set of size 101 is selected. Our goal
is to construct the conditional precision matrix of $Y$ conditioned
on $X$.

\begin{table}
\begin{center}
\caption{Statistics of data. \label{tab:data_statistics}}
\begin{tabular}{c c c c c }
\hline
 & $p$ & $q$ & training size (n) & test size \\
\hline \hline
\textsf{Corel5k}  & 260 & 512 & 450 & 450 \\
\textsf{MIRFlickr25k} & 457 & 512 & 1,250 & 1,250 \\
\textsf{RCV1-v2}  & 103 & 1,000 & 1,000 & 1,000  \\
\textsf{S\&P500}  & 165 & 300 & 101 & 101 \\
\hline
\end{tabular}
\end{center}
\end{table}

\begin{table}
\begin{center}
\caption{Quantitative results on real data
\label{tab:realdata_results}}
\begin{tabular}{c c c c c | c c c c}
\hline
& \multicolumn{4}{c}{$\Lpa$ value on test set}  &
\multicolumn{4}{c}{CPU Time (sec.)on training set} \\
\hline
& pGGM & GLasso & GLasso-M & NSLasso & pGGM & GLasso & GLasso-M & NSLasso\\
\hline \hline
\textsf{Corel5k} &-1.08e3  & -0.63e3  & --- &  --- & 16.63 &  125.74  &  9.07 &  9.06  \\
\textsf{MIRFlickr25k} &-1.99e3  & -1.99e3  &  ---  & --- & 56.93 &  228.71 & 39.74 & 42.89\\
\textsf{RCV1-v2}  &-0.42e3 & -0.39e3  & --- & --- &  3.04 & 421.86 & 1.38 & 75.43 \\
\textsf{S\&P500}  &0.22e3 & 0.24e3 & ---  & --- & 4.83 &  46.65 &  4.28 & 4.29\\
\hline
\end{tabular}
\end{center}
\end{table}

\subsubsection{Methods and Evaluation Metrics}

In these experiments, we compare pGGM to GLasso,  GLasso-M (for
estimating marginal precision matrix using the data component $Y$
only) and NSLasso. Here we focus on convex formulations, and thus
skip cGGM. For all these methods, we use the Bayesian information
criterion (BIC) to select the regularization parameters.

Since there is no ground truth precision matrix, we measure the
quality of $\hat\Theta$ by evaluating the $\Lpa$ objective (recall
its definition in \eqref{equat:Lp}) on the test data. The training
CPU times are also reported. Since the category information of
\textsf{RCV1-v2} and \textsf{S\&P500} are available, we also measure
the precision of the top $k$ links in the constructed conditional
GGM from $\Omega_{yy}$ on these two datasets. A link is regarded as
\emph{true} if and only if it connects two nodes belonging to the
same category. Note that the category information is \emph{not} used
in any of the graphical model learning procedures.

\subsubsection{Results}
%This is partially
%due to that $\Omega_{xx}$ is potentially dense and thus GLasso under penalizes the blocks $\Omega_{yy}$ and $\Omega_{yx}$
%which pGGM can penalize just right.
Table~\ref{tab:realdata_results} tabulates the evaluated $\Lpa$
objectives on the test set and the training time. The key
observations are
\begin{itemize}
  \item In most cases, pGGM outputs smaller $\Lpa$ objective value than
GLasso (note that the $\Lpa$ value cannot be evaluated for GLasso-M
and NSLasso). pGGM runs much faster than GLasso on all these
datasets.
  \item pGGM is slightly slower than NSLasso on \textsf{Corel5k},
\textsf{MIRFlickr25k} and \textsf{S\&P500} where $p \sim q$, but
significantly faster than NSLasso on RCV1-v2 where $p\ll q$.
\end{itemize}
Figure~\ref{fig:precision} shows the precision of top $k$ links in
the conditional graphs as a function of $k$. It can be seen that
pGGM performs favorably in comparison to the other three methods for
identifying correct links on \textsf{RCV1-v2}. On \textsf{S\&P500},
pGGM and GLasso-M have comparable performance,  and both are better
than GLasso and NSLasso. This is because the \textsf{S\&P500} stocks
are weakly correlated and thus the conditional graphical model can
be well approximated by the marginal graphical model.
\begin{figure}
\centering
\includegraphics[width=80mm]{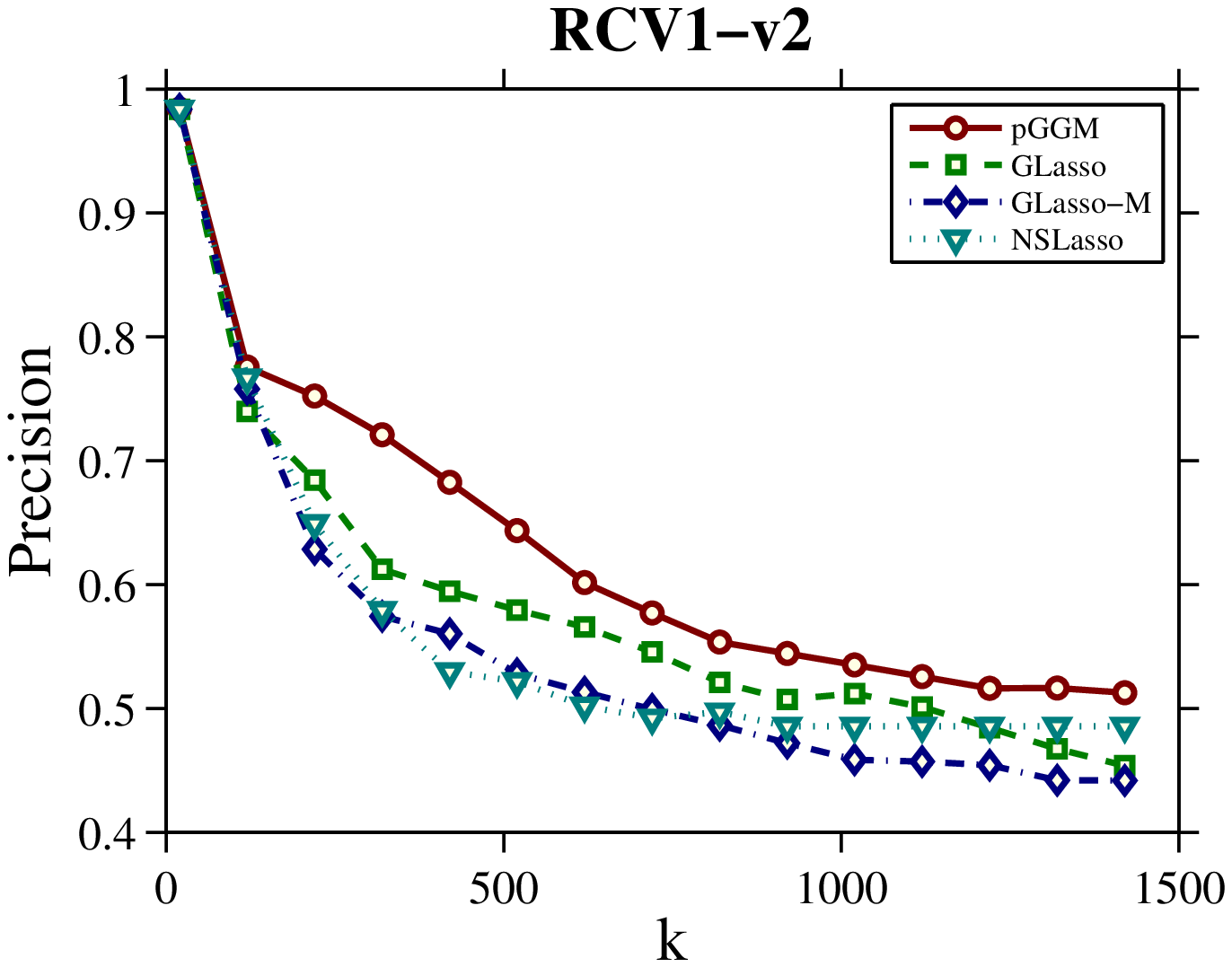}
\includegraphics[width=80mm]{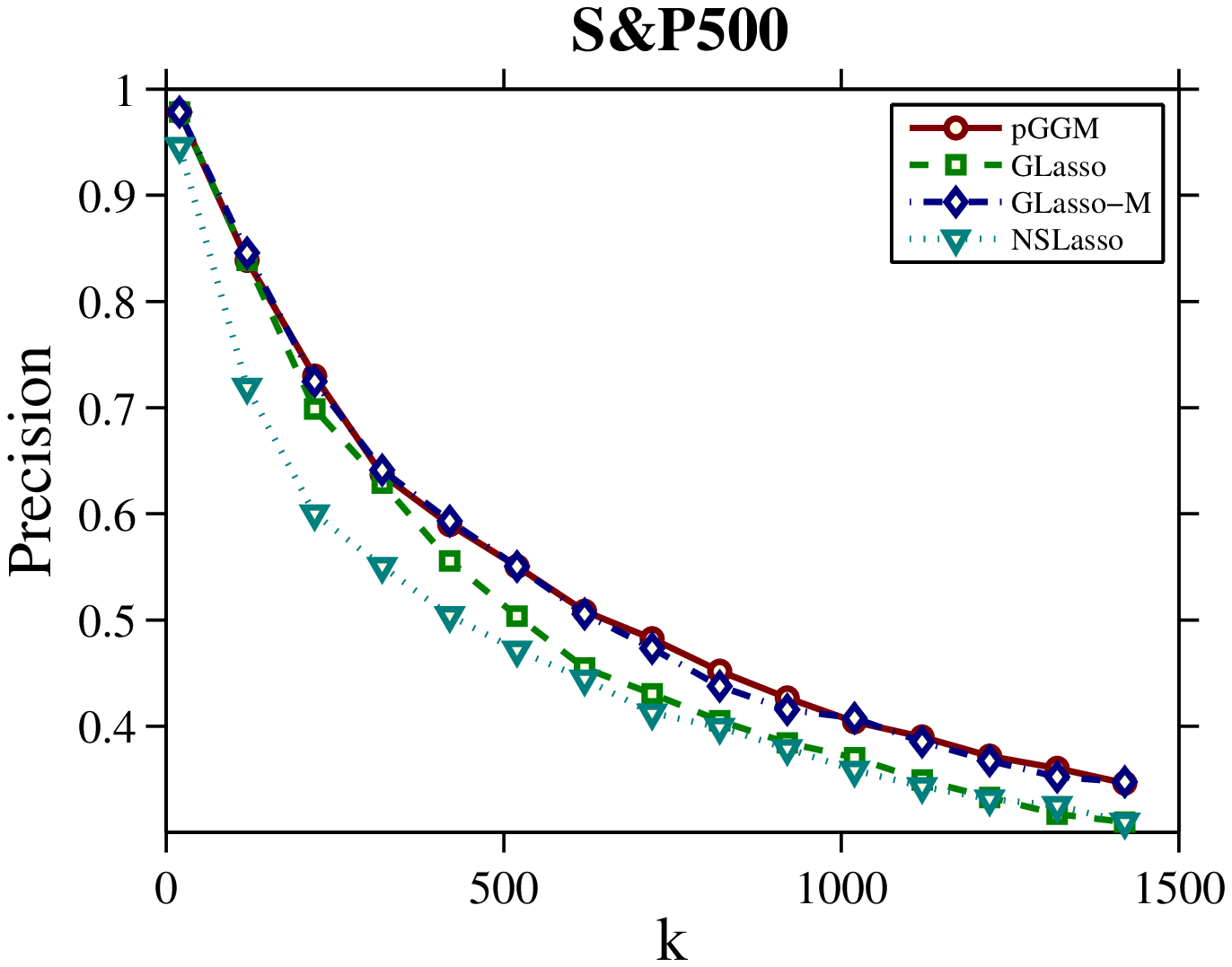}
\caption{Link precision curves on \textsf{RCV1-v2} and
\textsf{S\&P500}. \label{fig:precision}}
\end{figure}

We further evaluate the sparsity of the constructed graphs on these
datasets. The links are identified by $\{(i,j): i\neq j,
|[\hat\Omega_{yy}]_{ij}|\ge \mu\}$ in which $\mu>0$ is a threshold
value. Figure~\ref{fig:num_edges} shows the number of links in
graphs as a function of $\mu$. It can be seen that pGGM, GLasso and
NSLasso tend to output sparser graphical models than GLasso-M. A
potential reason is that GLasso-M ignores the information provided
by $X$, and thus false positive links can be induced. NSLasso
outputs the sparsest network on \textsf{corel5k},
\textsf{MIRFlickr25k} and \textsf{S\&P500}, while pGGM outputs the
sparsest model on \textsf{RCV1-v2}. Note that NSLasso does not
estimate precision matrix. Moreover, pGGM tends to be slightly
sparser than GLasso. These observations are consistent with our
observations on the synthetic data.
\begin{figure}
\centering
\includegraphics[width=80mm]{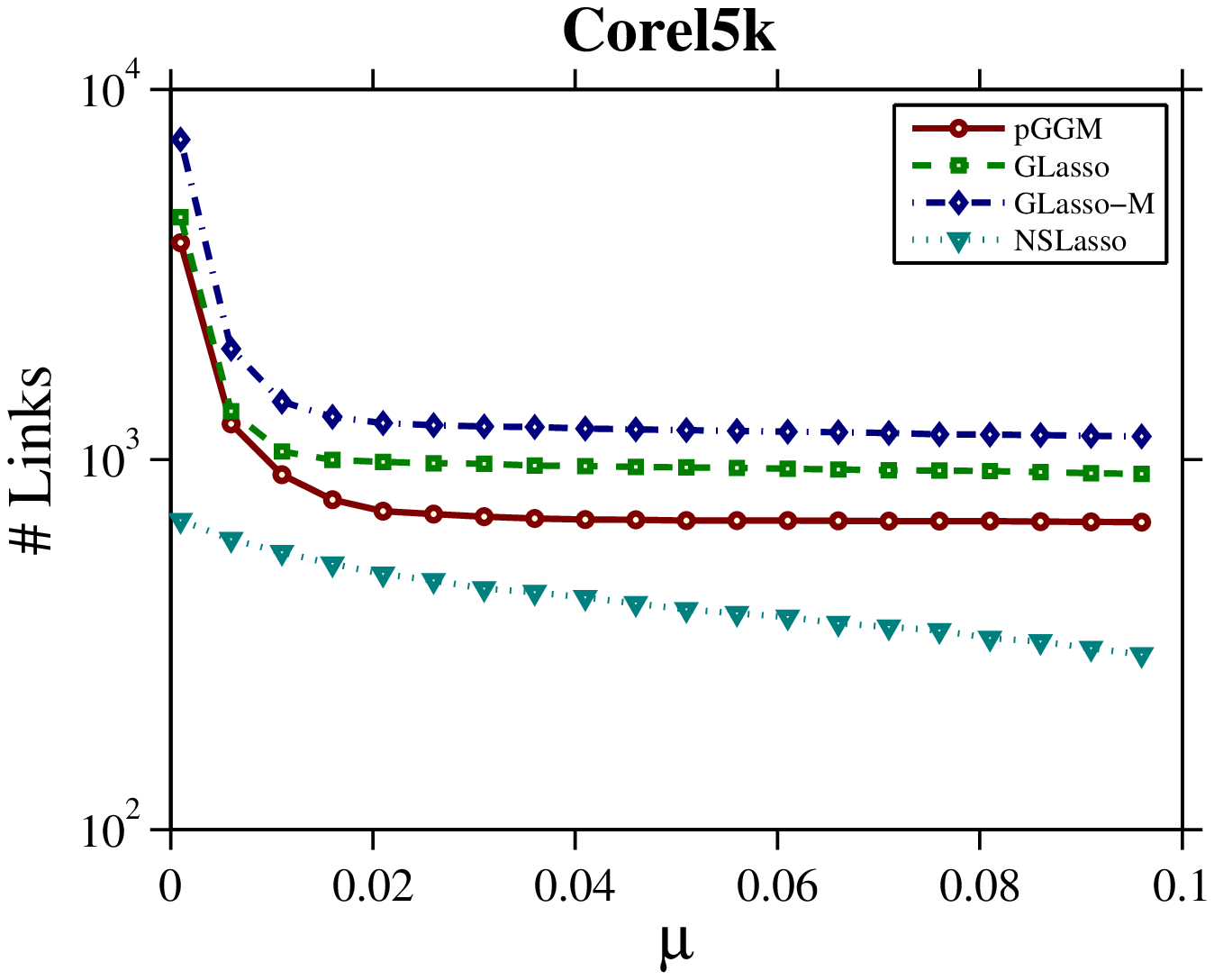}
\includegraphics[width=80mm]{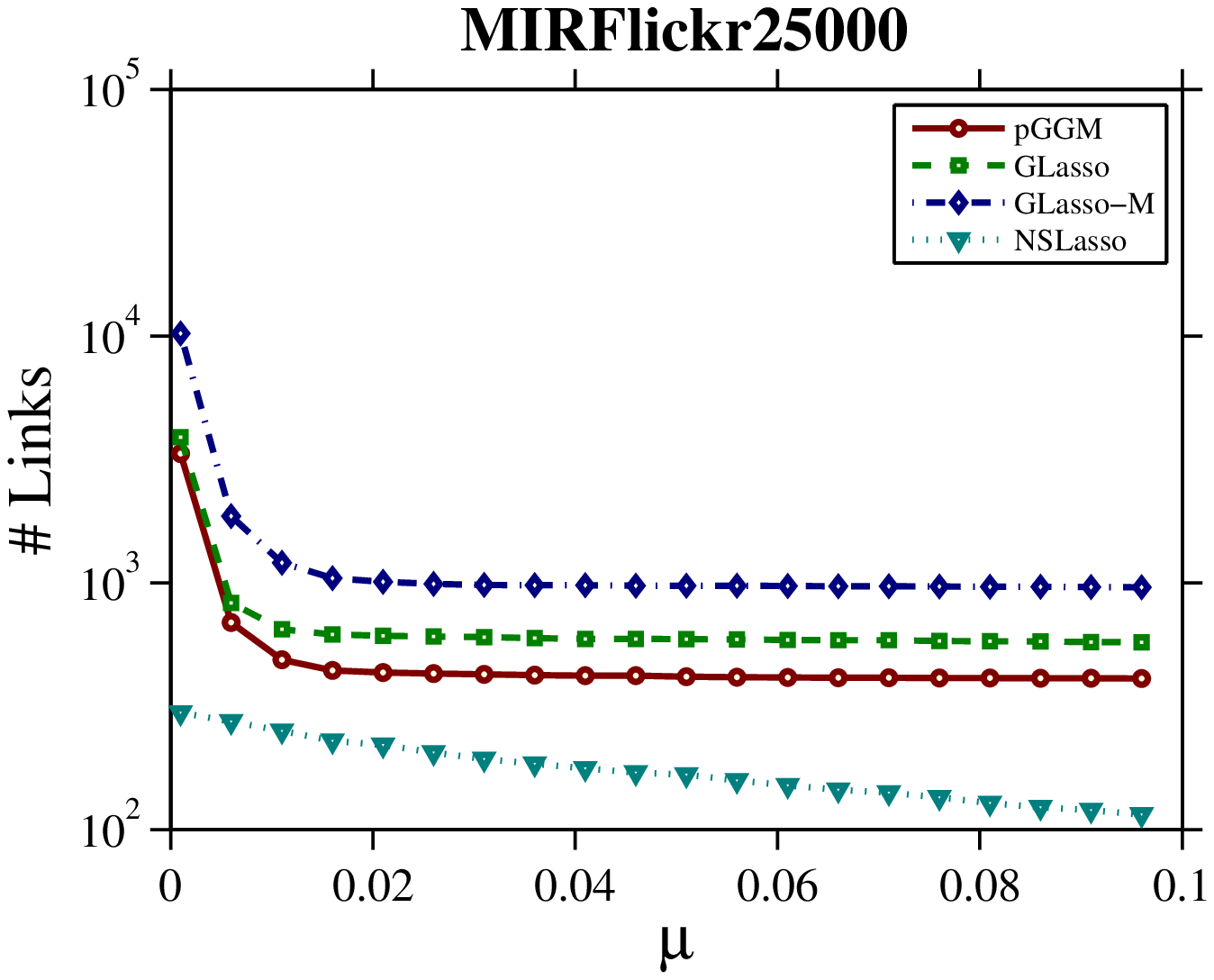}
\includegraphics[width=80mm]{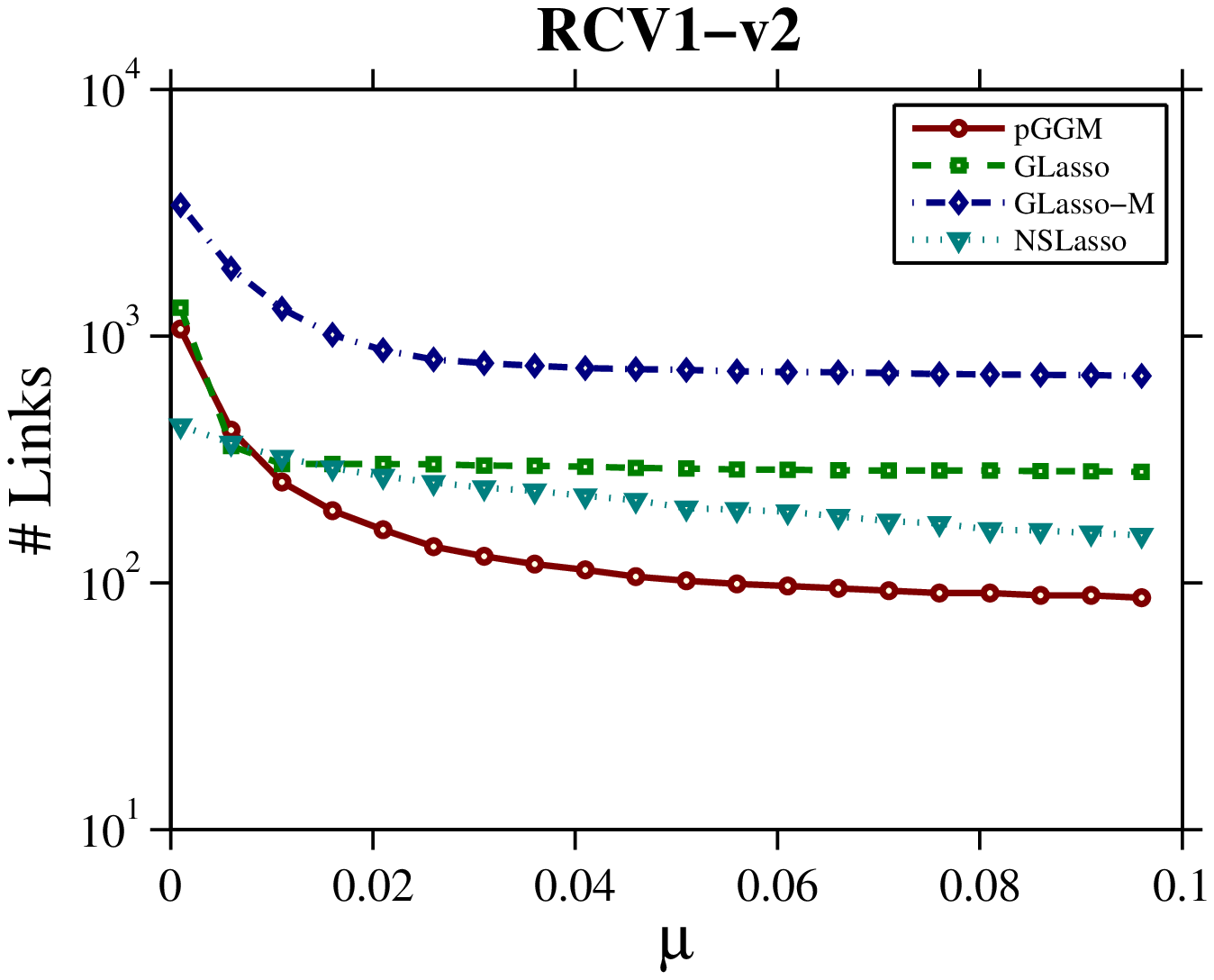}
\includegraphics[width=80mm]{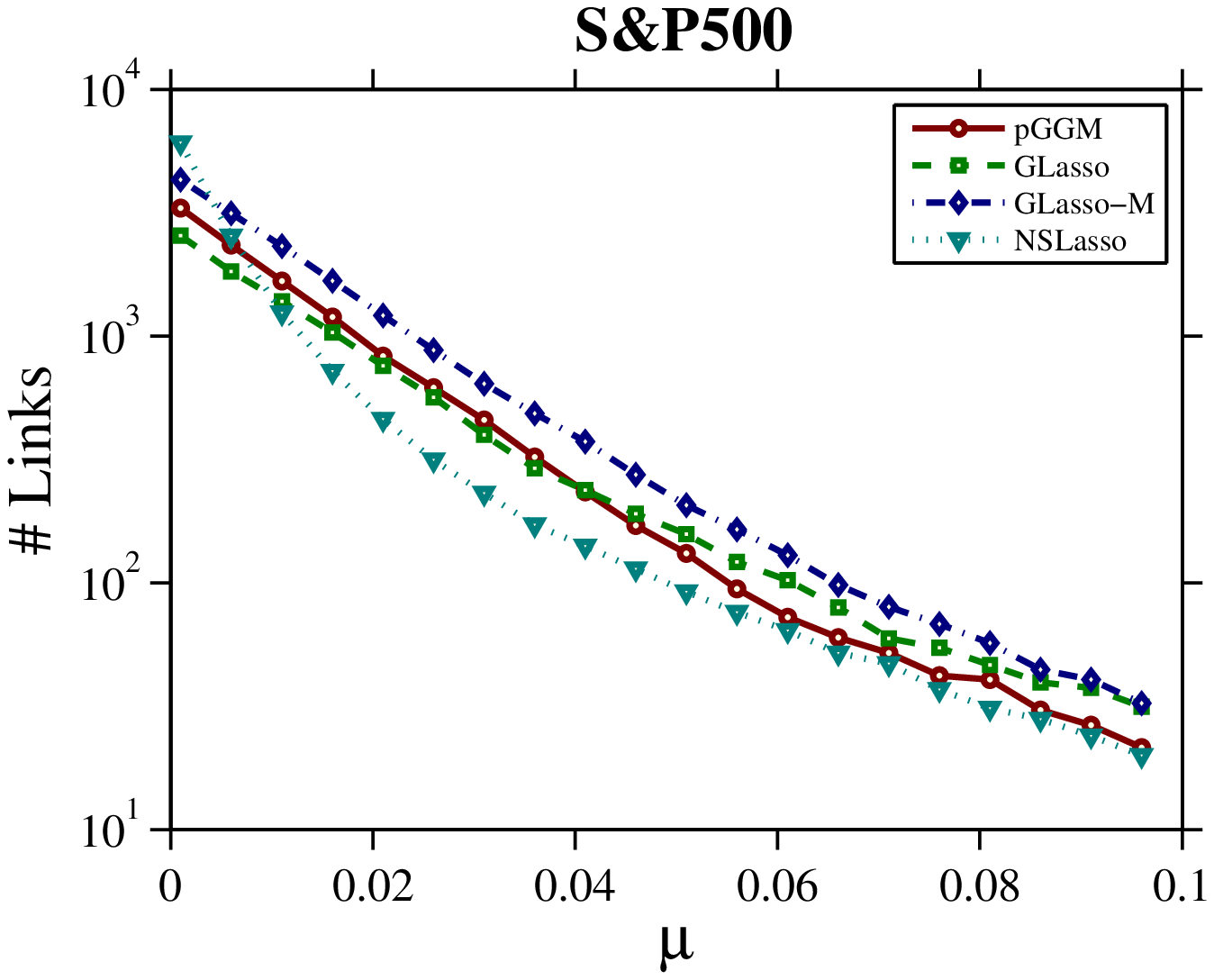}
\caption{Number of links as a function of $\mu$ in the constructed
conditional graphical model. \label{fig:num_edges}}
\end{figure}

Figure~\ref{fig:graph} plots the graphs constructed by using
different estimation methods with $\mu = 0.1$ for \textsf{Corel5k},
\textsf{MIRFlickr25k} and \textsf{RCV1-v2}, and $\mu = 0.05$ for
\textsf{S\&P500}. It can be seen that different methods will
construct different graphs. Figure~\ref{fig:graph_topk} illustrates
in detail the top 50 links in each graph.

\begin{figure}
\centering \subfigure[ \textsf{Corel5k}, $\mu=0.1$. Method(\#
Links): pGGM (677), NSLasso (293), GLasso (909), GLasso-M (1153).
]{\includegraphics[width=35mm]{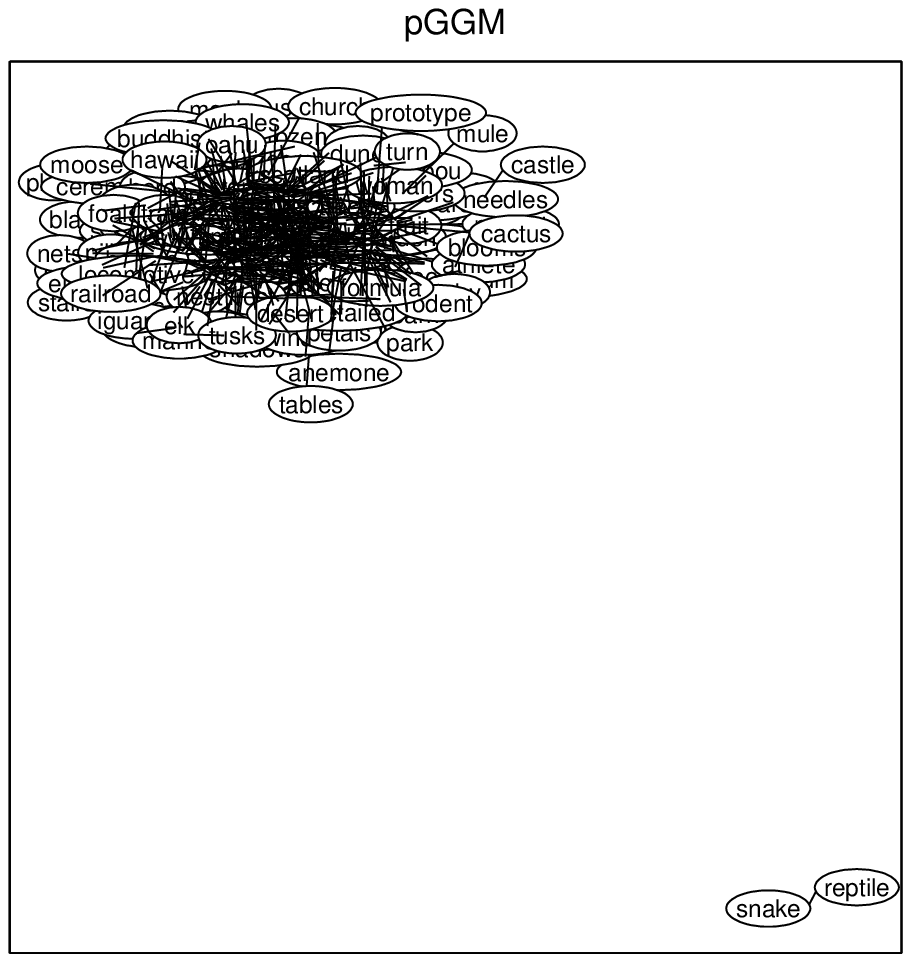}\label{fig:CPME_BCD_Graph_corel5k}
\includegraphics[width=35mm]{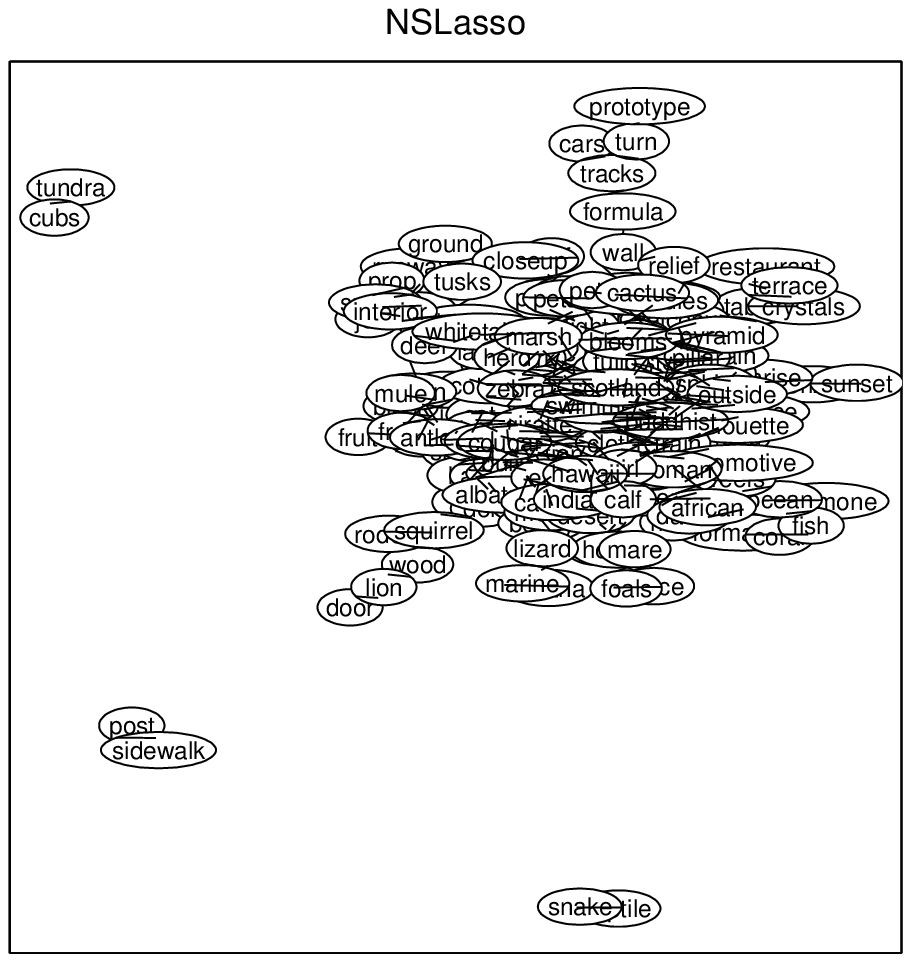}\label{fig:NSLasso_Graph_corel5k}
\includegraphics[width=35mm]{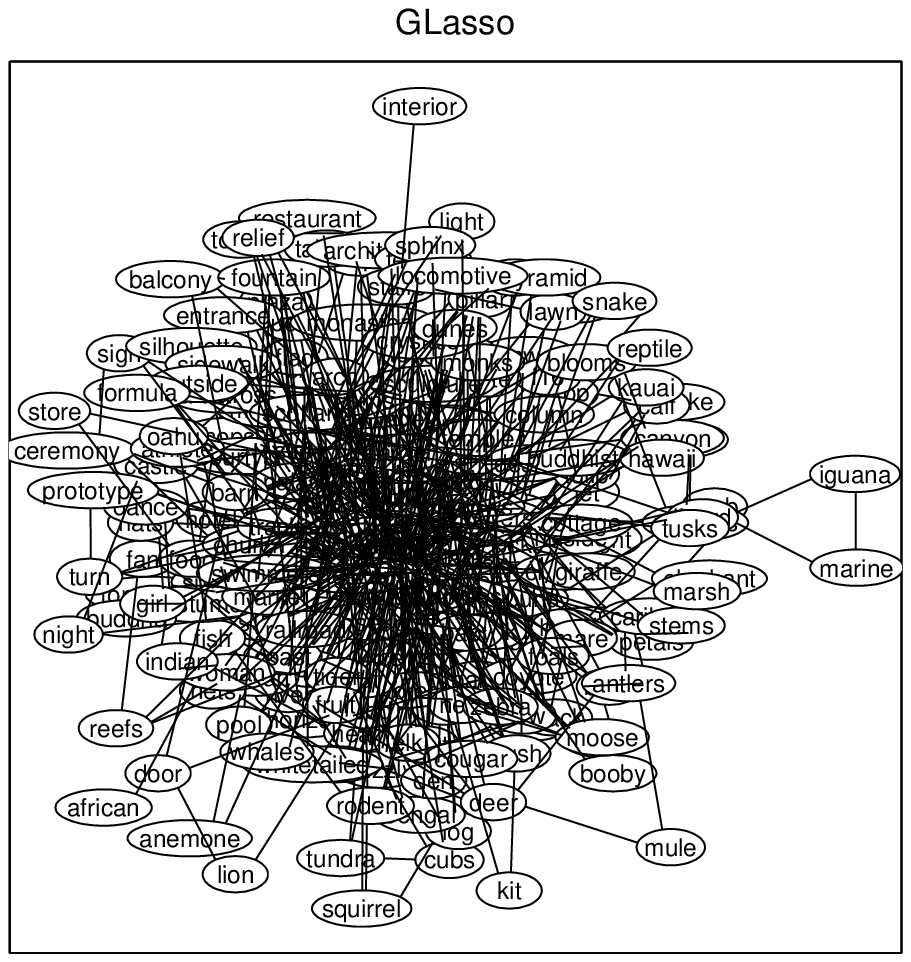}\label{fig:GLasso_Graph_corel5k}
\includegraphics[width=35mm]{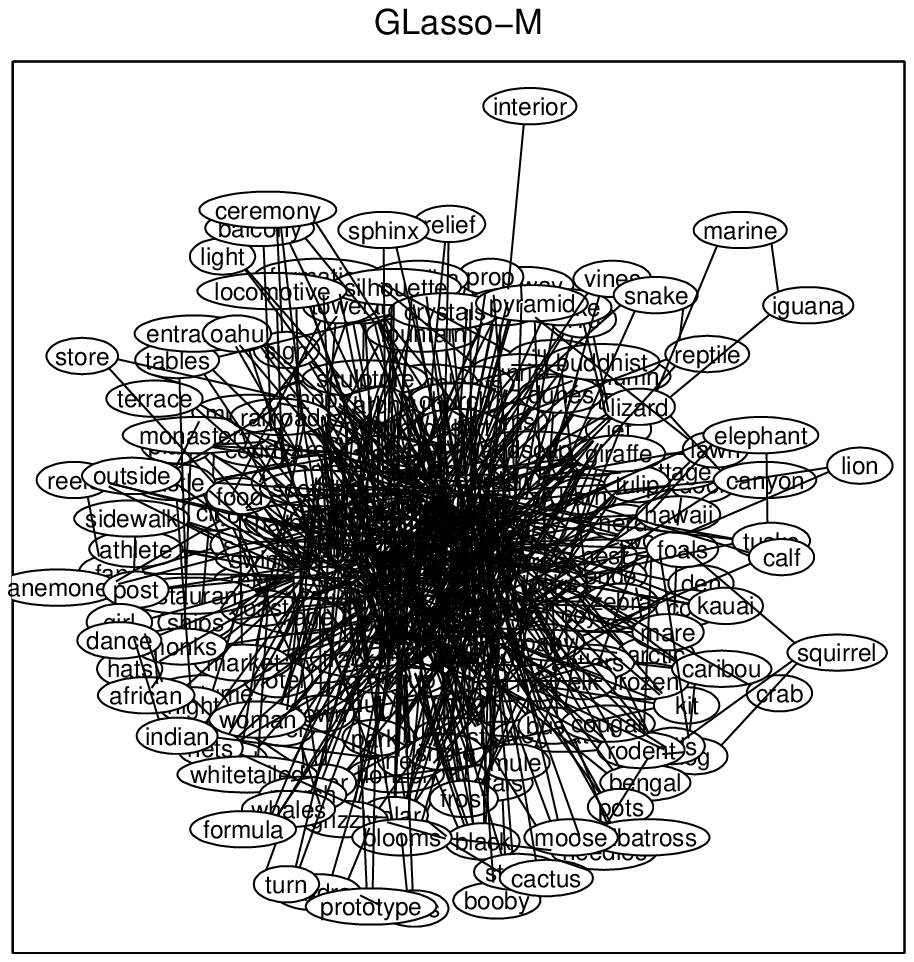}\label{fig:GLasso_OnlyY_Graph_corel5k}
\label{fig:graph_corel5k}} \subfigure[ \textsf{MIRFlicker25k},
$\mu=0.1$. Method(\# Links): pGGM (409), NSLasso (110), GLasso
(573), GLasso-M
(960).]{\includegraphics[width=35mm]{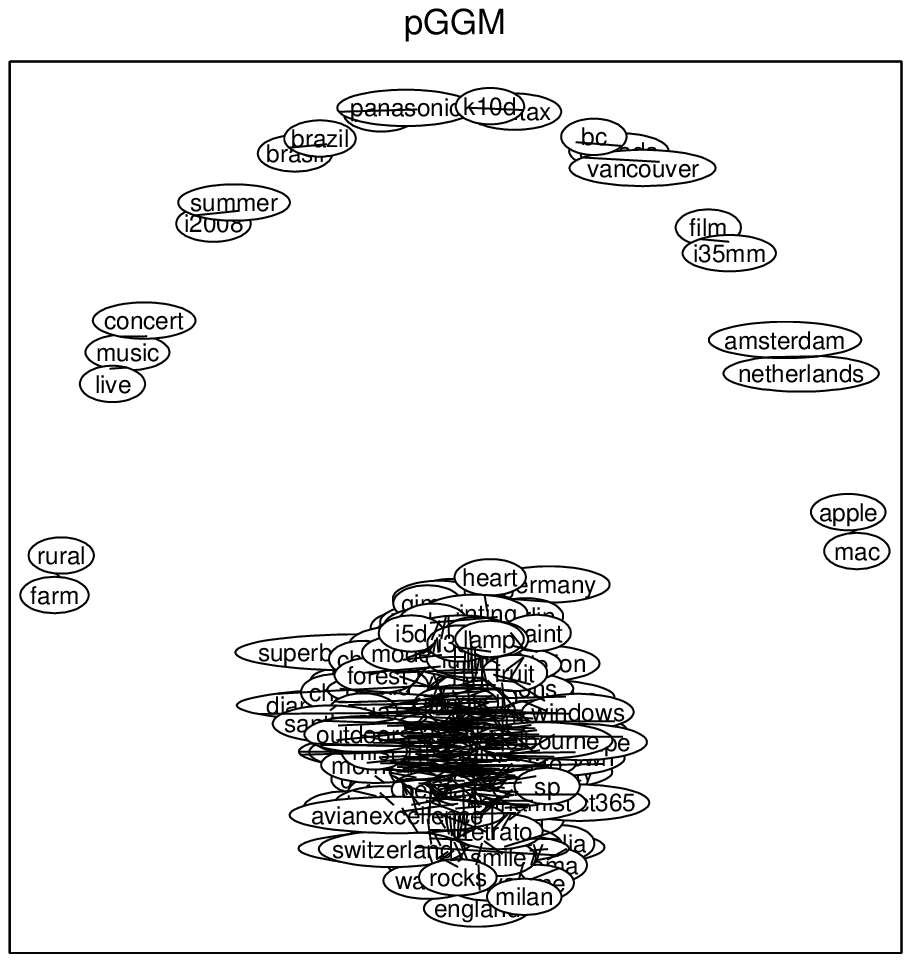}\label{fig:CPME_BCD_Graph_flickr}
\includegraphics[width=35mm]{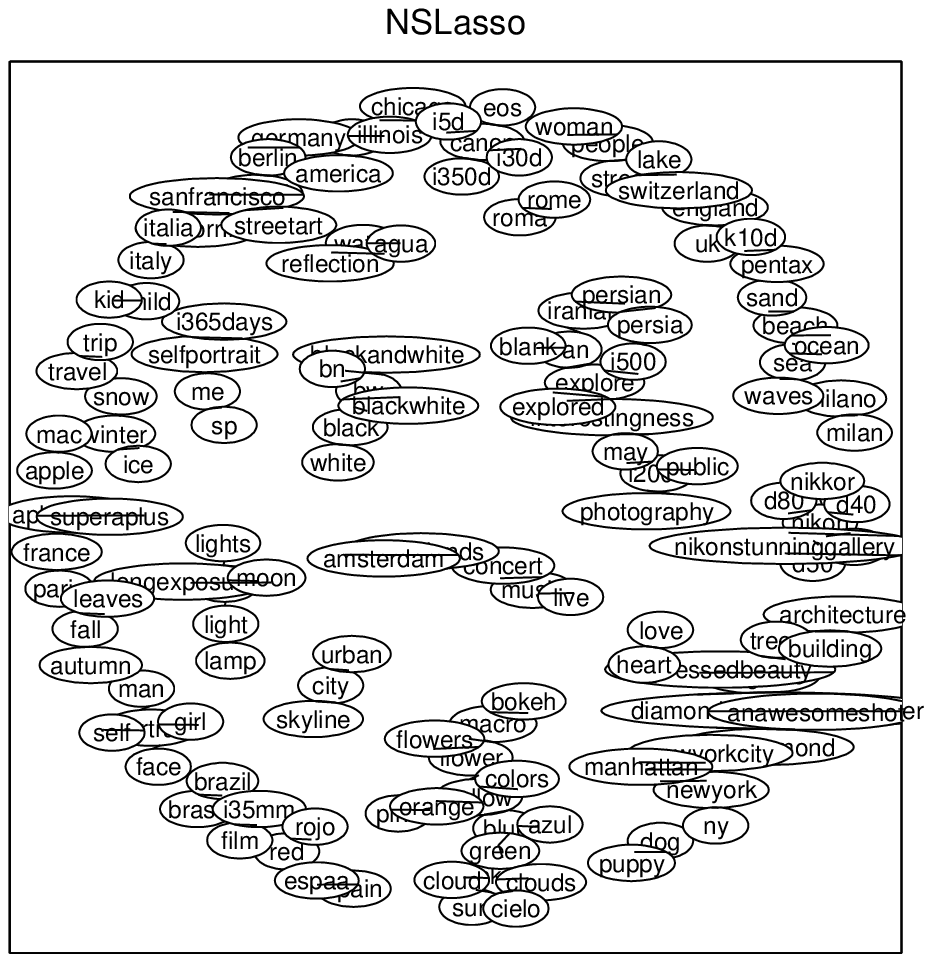}\label{fig:NSLasso_Graph_flickr}
\includegraphics[width=35mm]{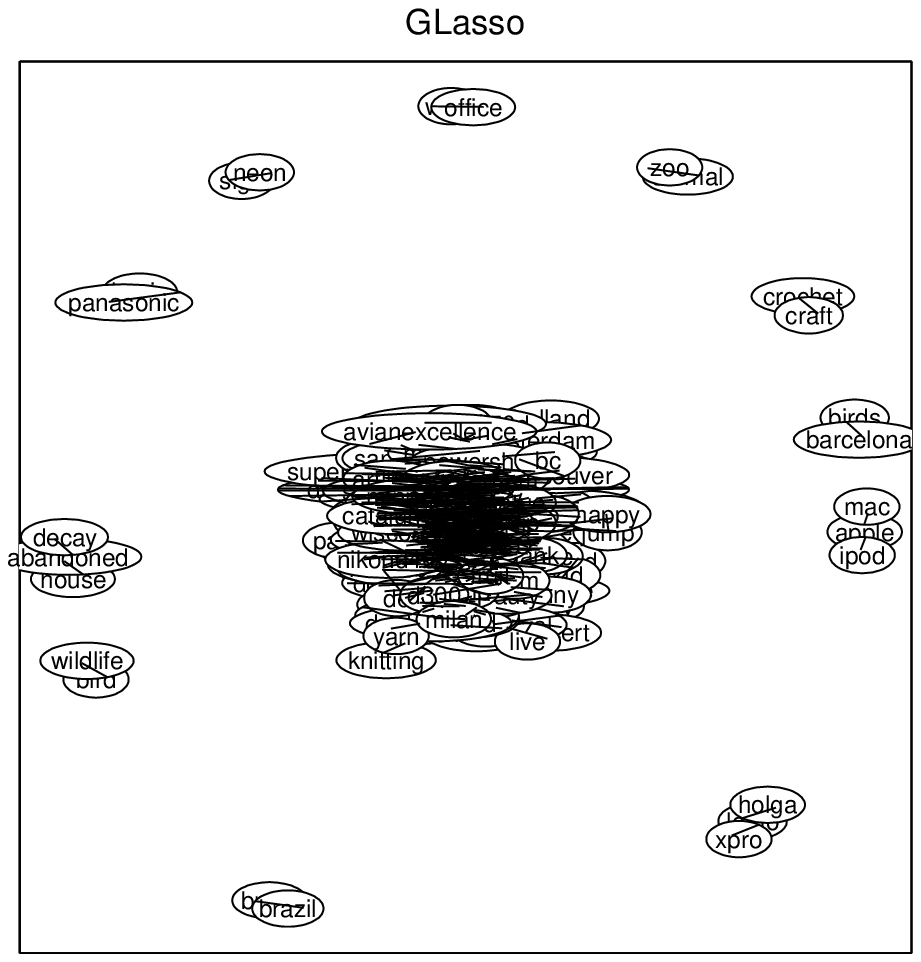}\label{fig:GLasso_Graph_flickr}
\includegraphics[width=35mm]{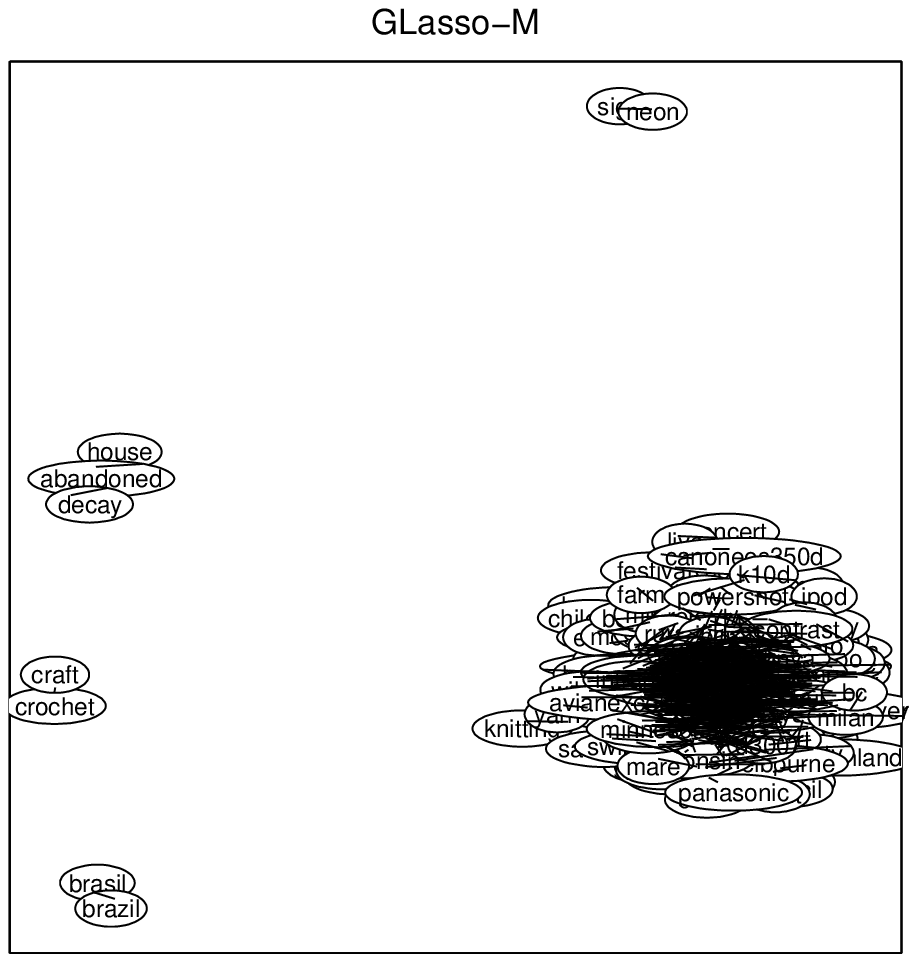}\label{fig:GLasso_OnlyY_Graph_flickr}
\label{fig:graph_flickr}} \subfigure[ \textsf{RCV1-v2}, $\mu=0.1$.
Method(\# Links): pGGM (87), NSLasso (156), GLasso (282), GLasso-M
(688).]{\includegraphics[width=35mm]{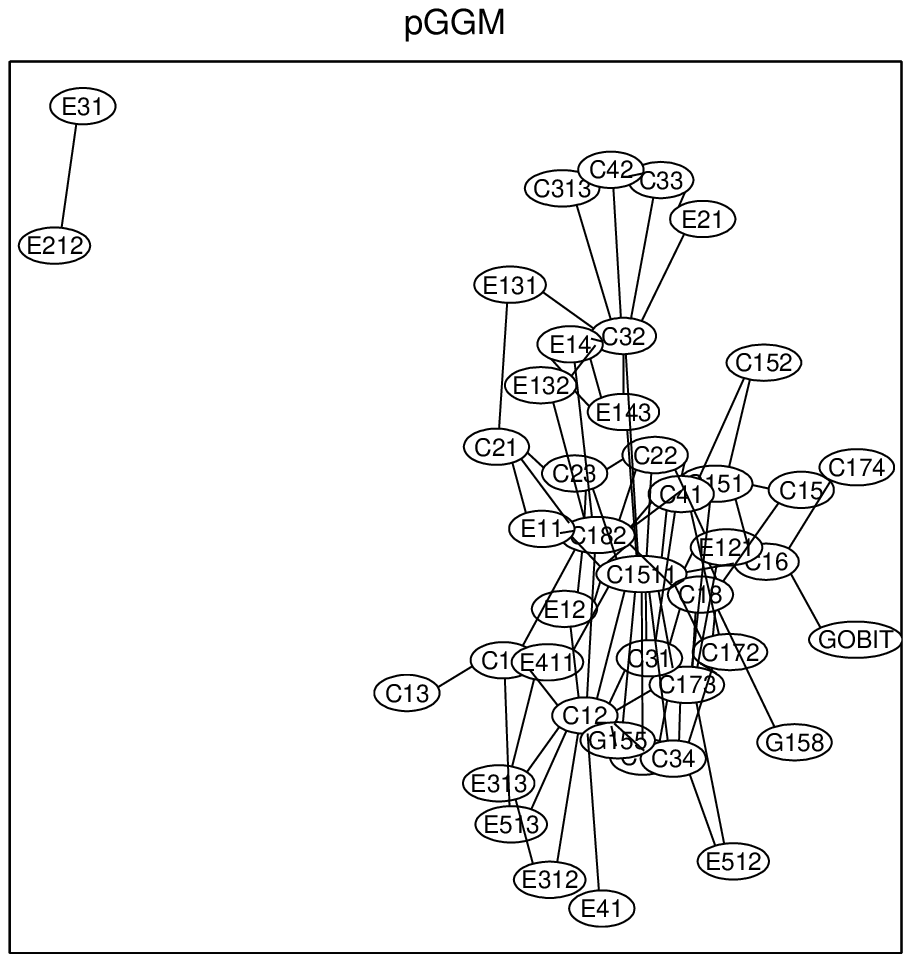}\label{fig:CPME_BCD_Graph_rcv1}
\includegraphics[width=35mm]{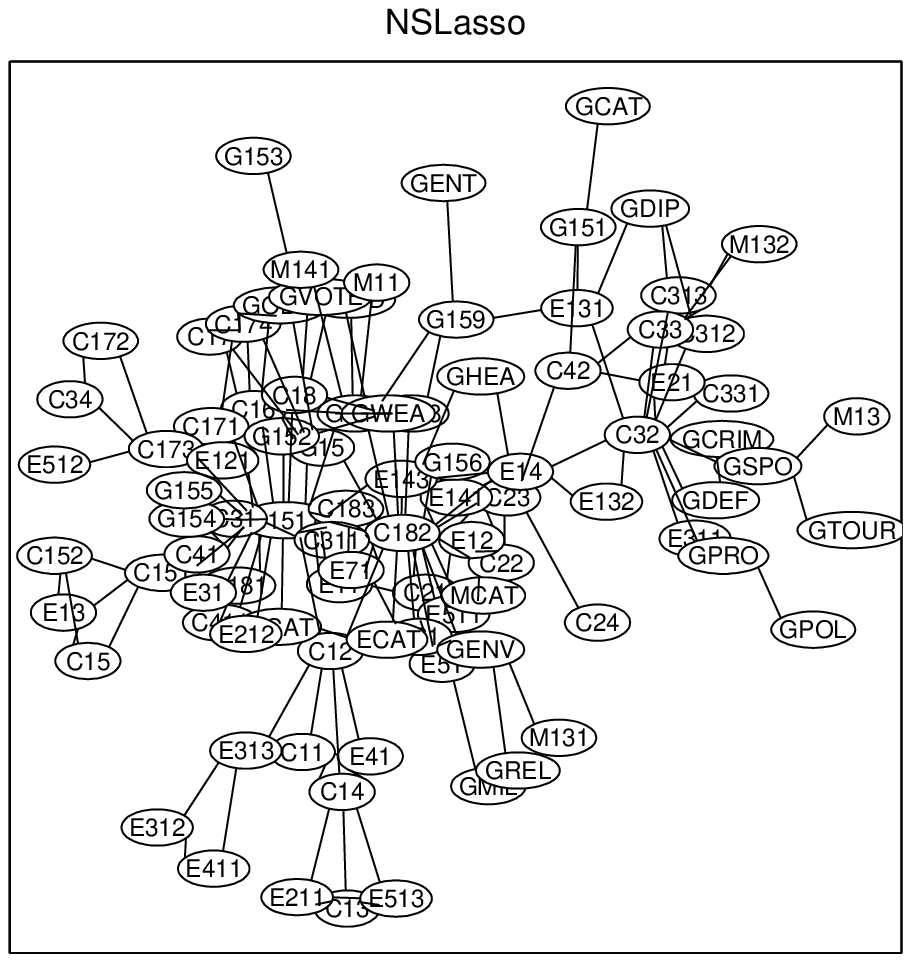}\label{fig:NSLasso_Graph_rcv1}
\includegraphics[width=35mm]{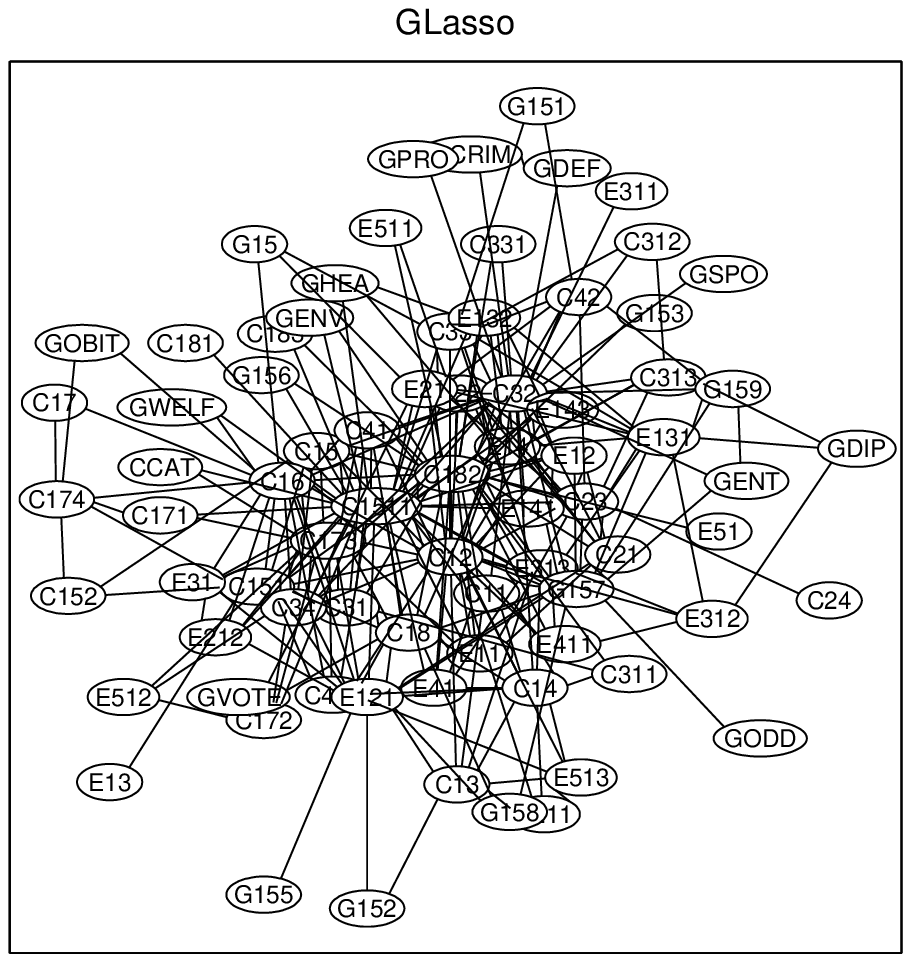}\label{fig:GLasso_Graph_rcv1}
\includegraphics[width=35mm]{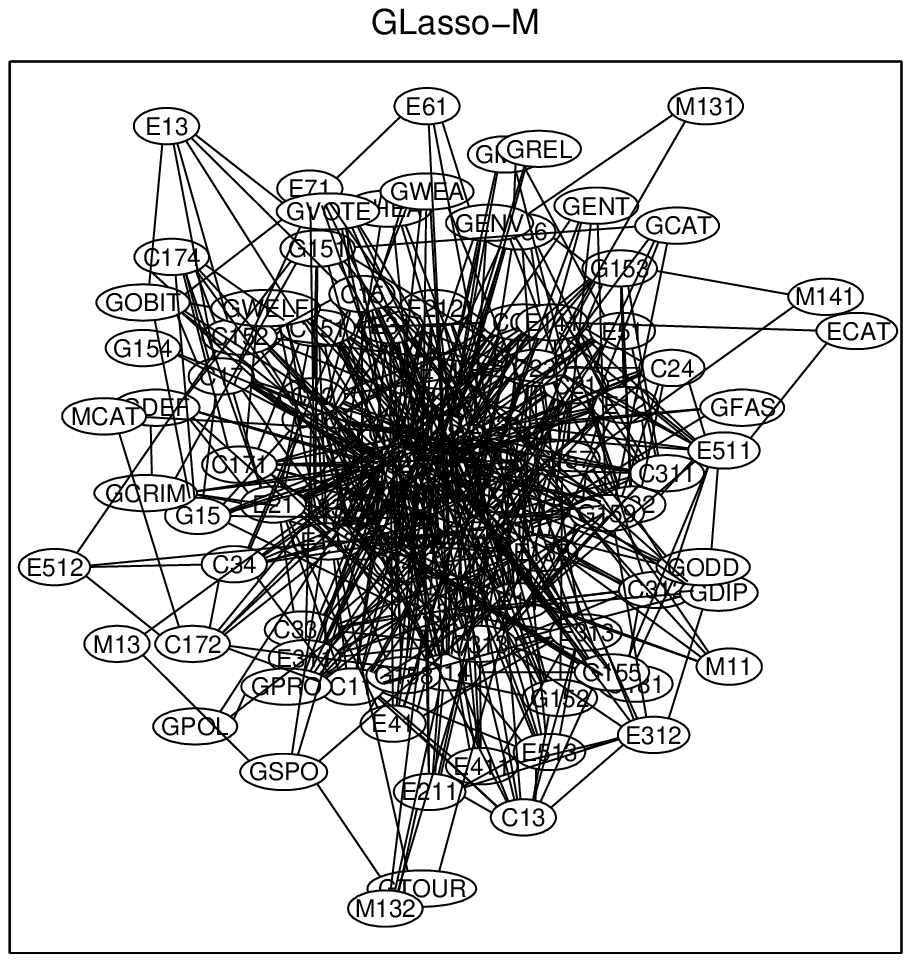}\label{fig:GLasso_OnlyY_Graph_rcv1}
\label{fig:graph_rcv1}} \subfigure[ \textsf{S\&P500}, $\mu=0.05$.
Method(\# Links): pGGM (136), NSLasso (94), GLasso (160), GLasso-M
(221).]{\includegraphics[width=35mm]{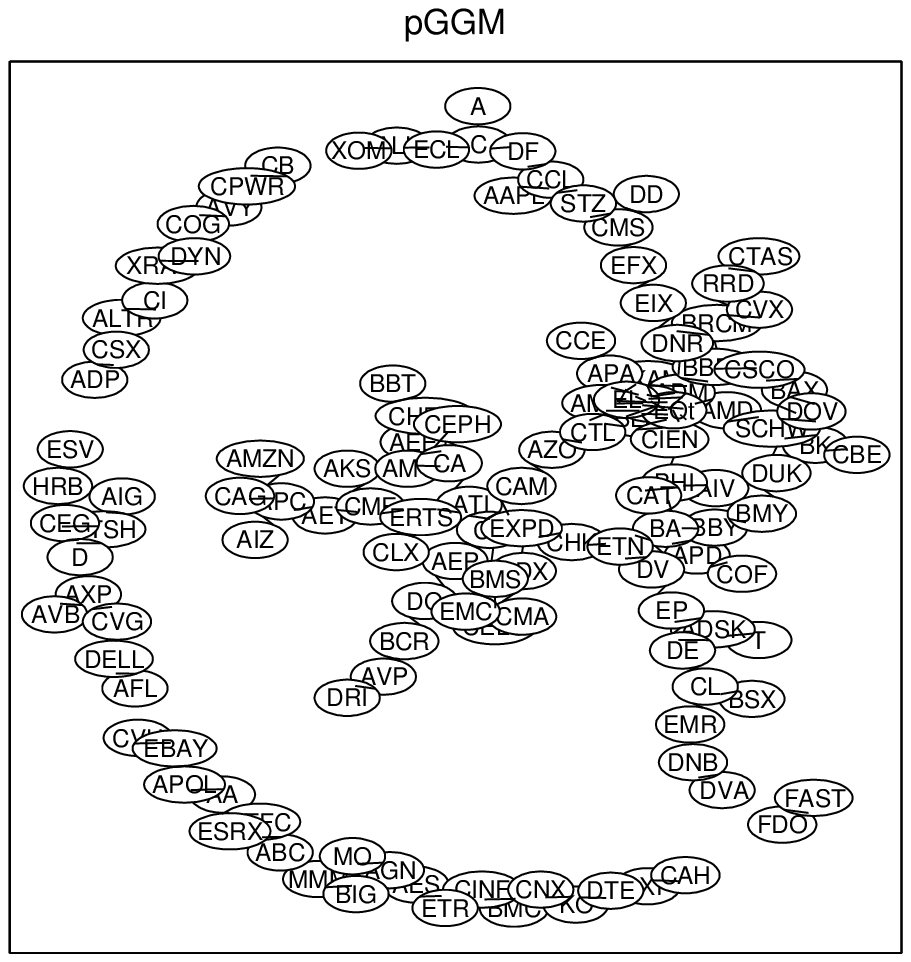}\label{fig:CPME_BCD_Graph_sp500}
\includegraphics[width=35mm]{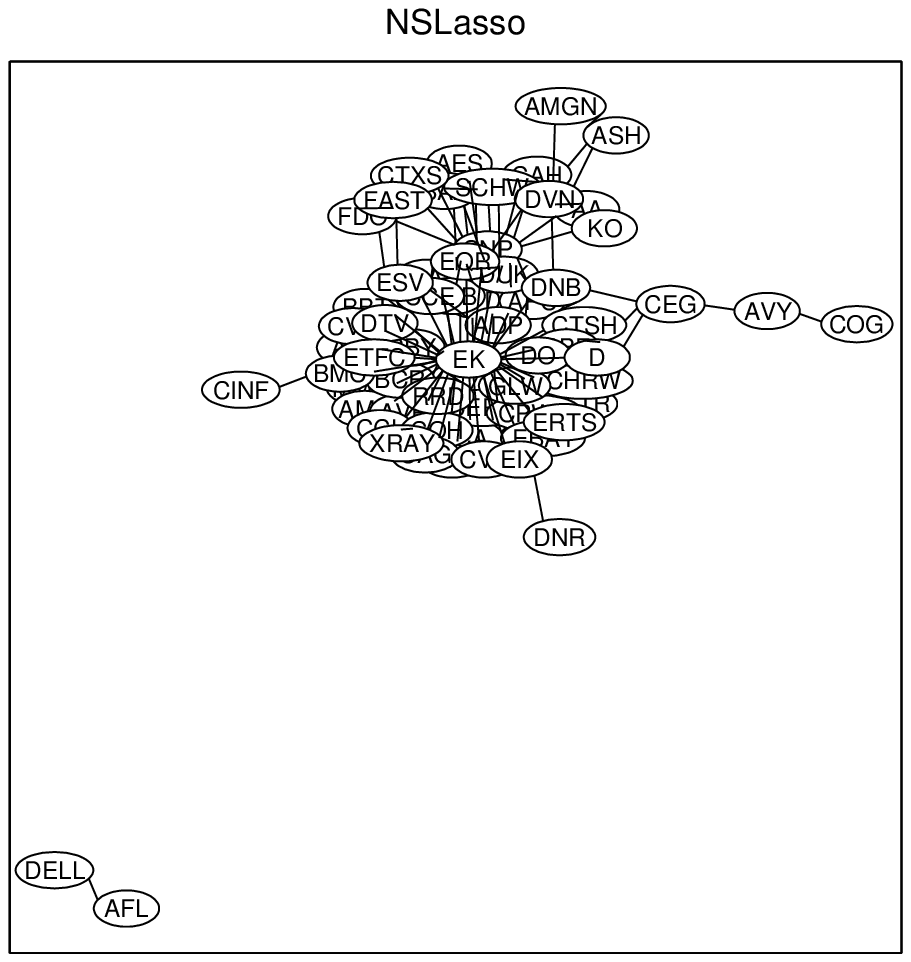}\label{fig:NSLasso_Graph_sp500}
\includegraphics[width=35mm]{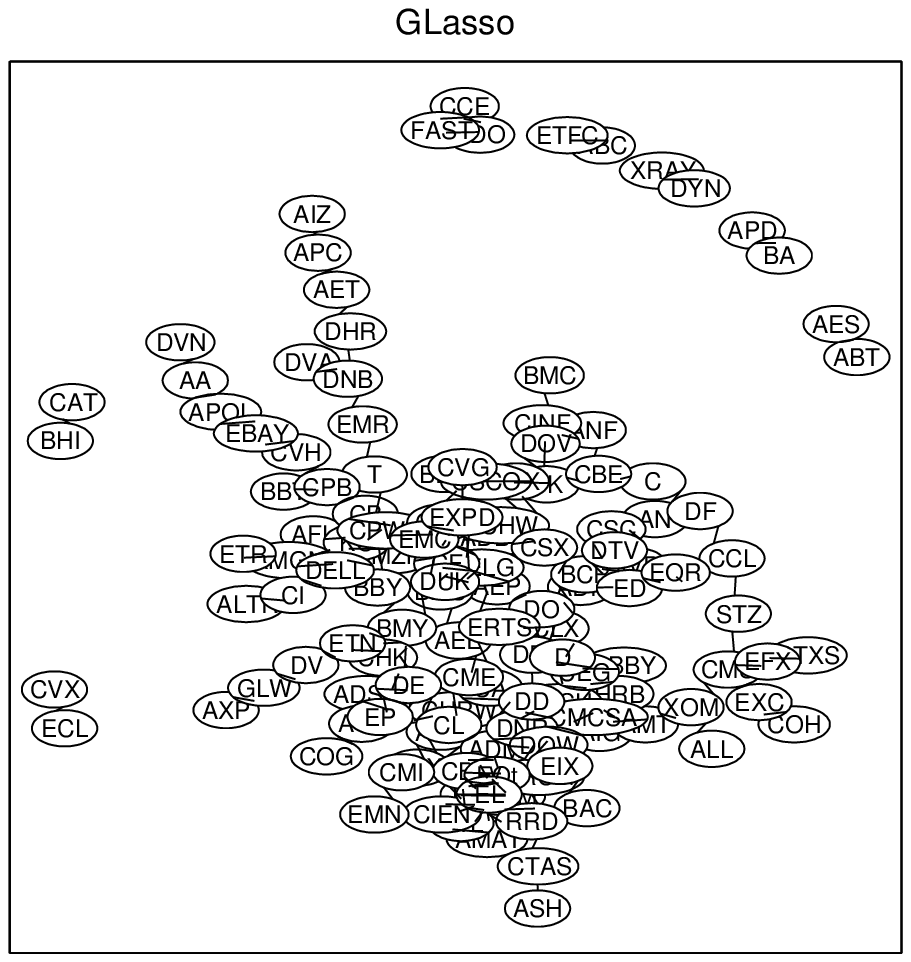}\label{fig:GLasso_Graph_sp500}
\includegraphics[width=35mm]{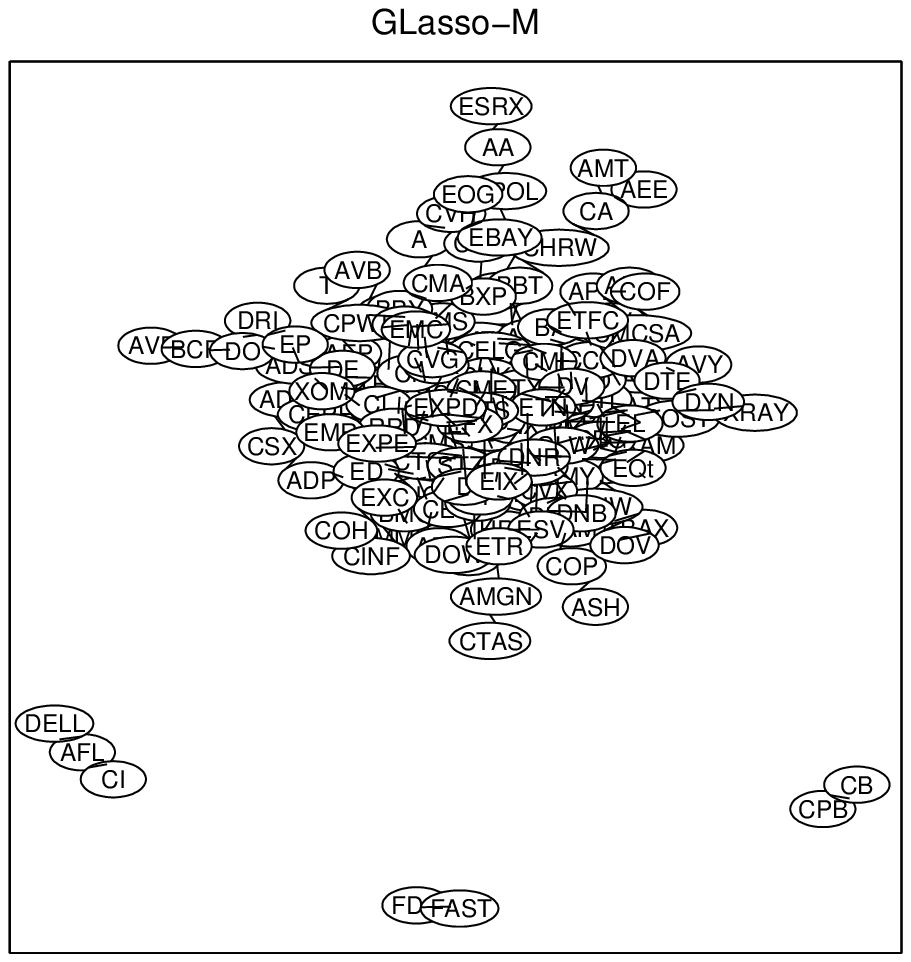}\label{fig:GLasso_OnlyY_Graph_sp500}
\label{fig:graph_sp500}} \caption{Constructed graphs by pGGM,
NSLasso, GLasso and GLasso-M. \label{fig:graph}}
\end{figure}

\begin{figure}
\centering \subfigure[
\textsf{Corel5k}.]{\includegraphics[width=35mm,height=35mm]{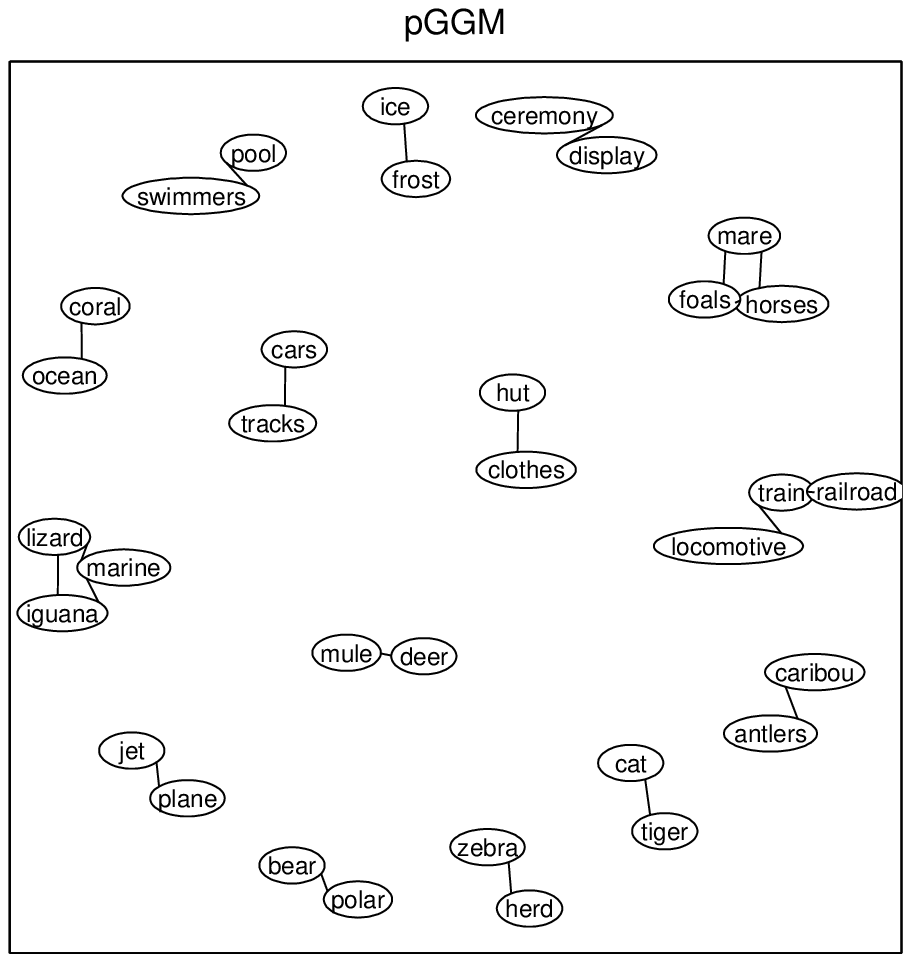}\label{fig:CPME_BCD_Graph_corel5k_topk}
\includegraphics[width=35mm,height=35mm]{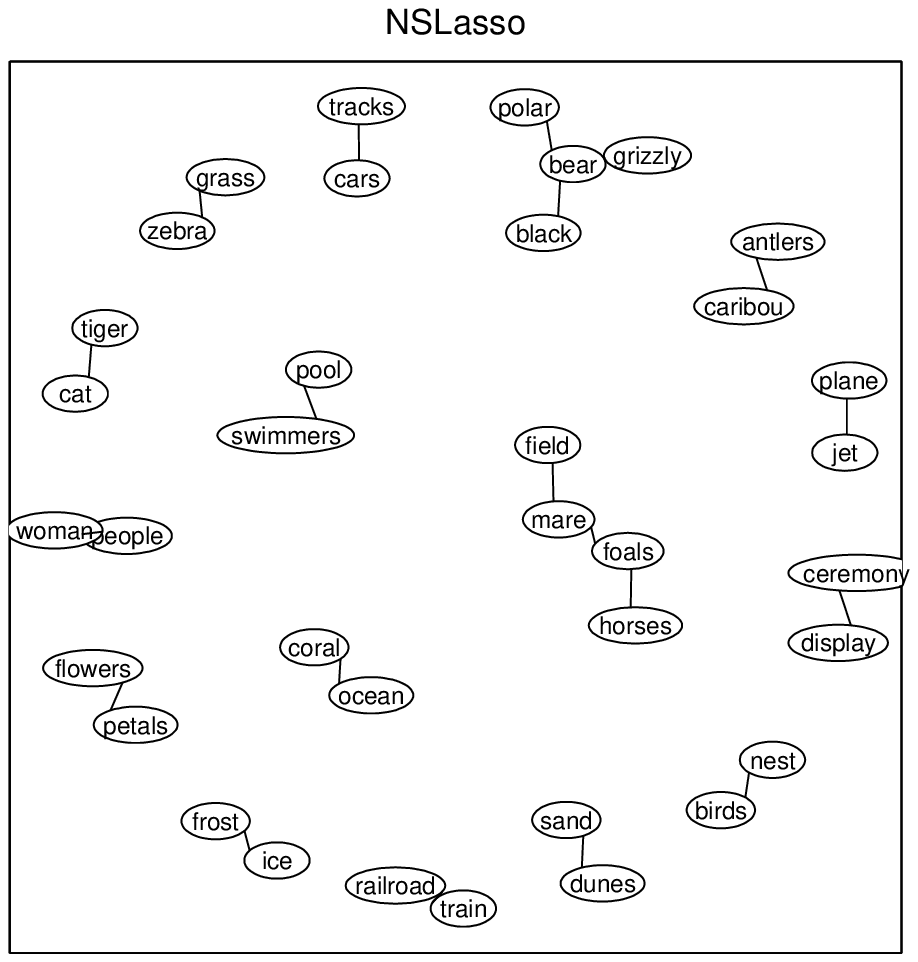}\label{fig:NSLasso_Graph_corel5k_topk}
\includegraphics[width=35mm,height=35mm]{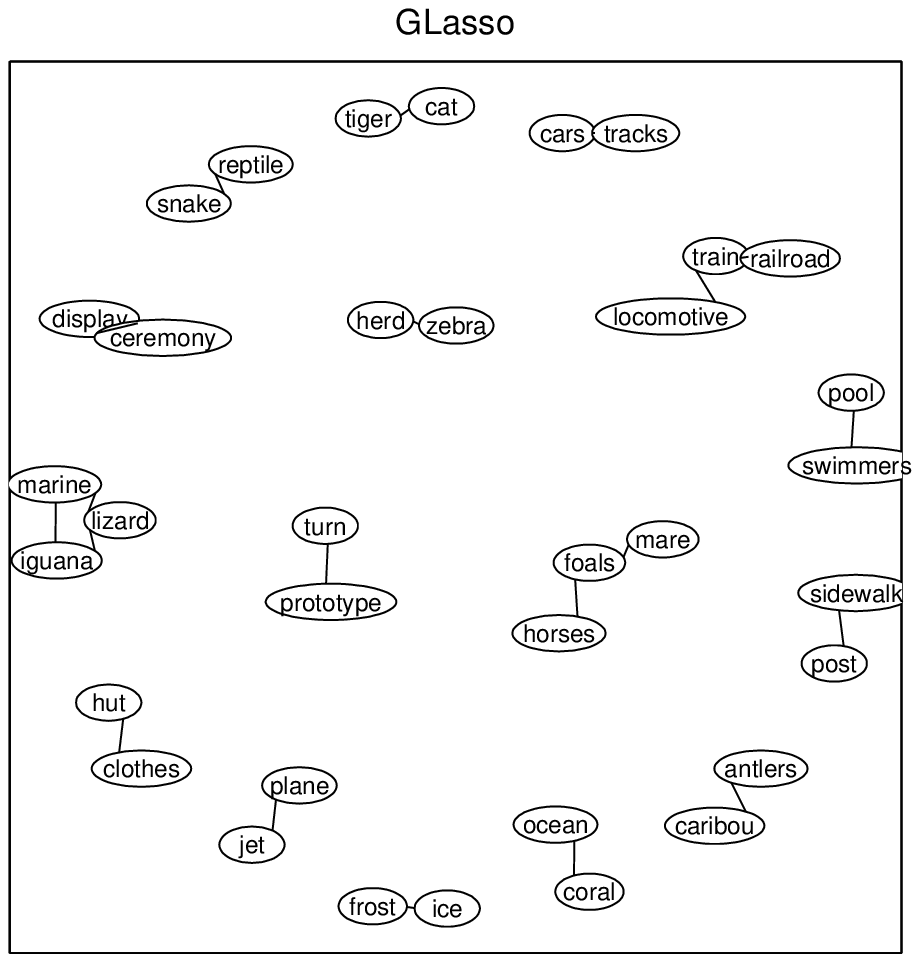}\label{fig:GLasso_Graph_corel5k_topk}
\includegraphics[width=35mm,height=35mm]{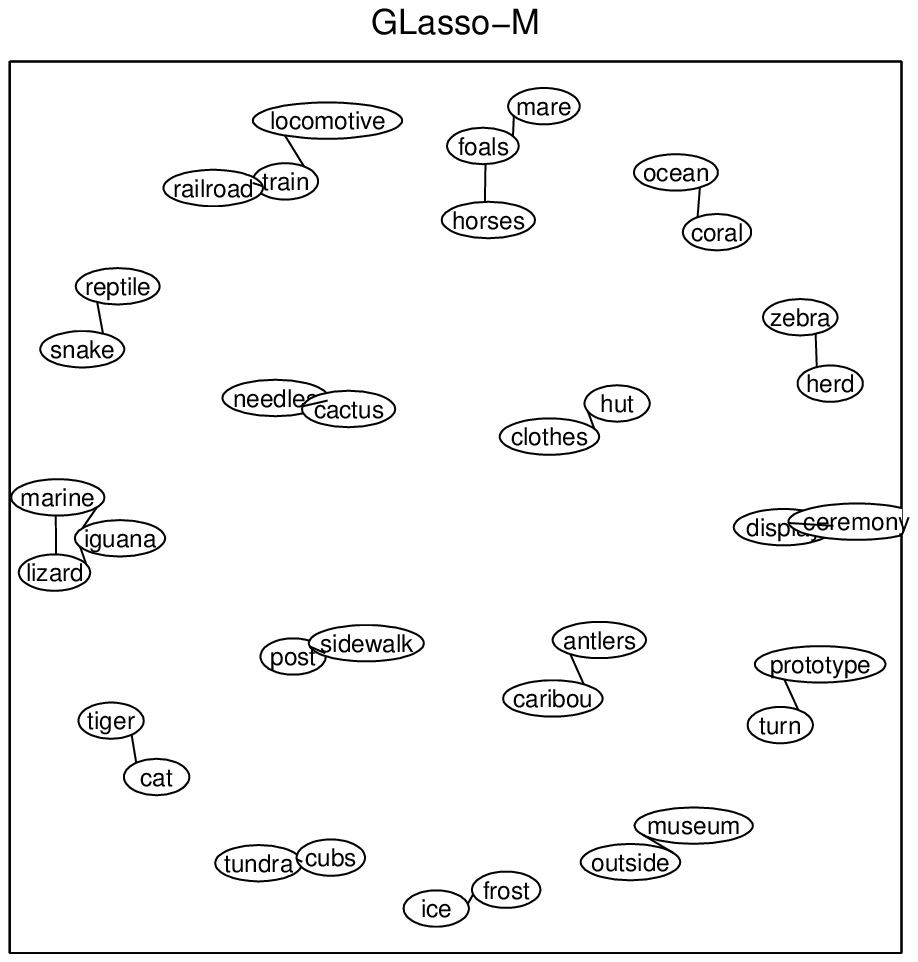}\label{fig:GLasso_OnlyY_Graph_corel5k_topk}
\label{fig:graph_corel5k_topk}} \subfigure[
\textsf{MIRFlicker25k}.]{\includegraphics[width=35mm,height=35mm]{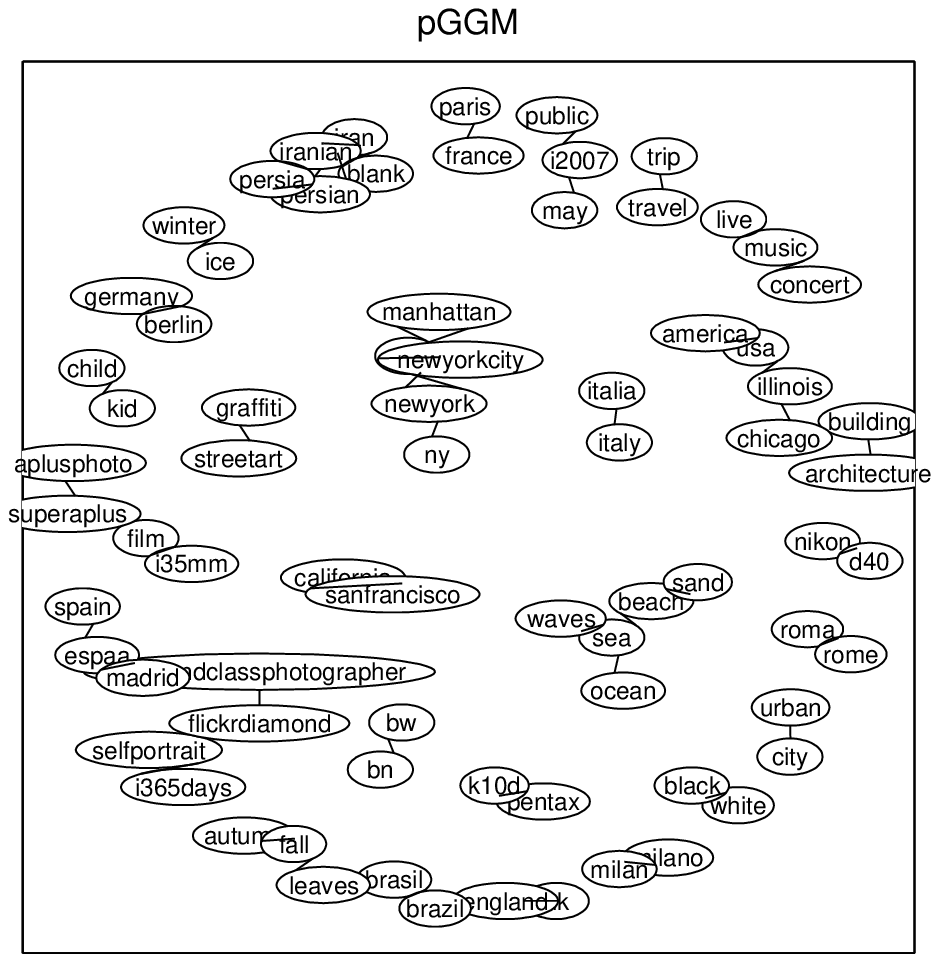}\label{fig:CPME_BCD_Graph_flickr_topk}
\includegraphics[width=35mm,height=35mm]{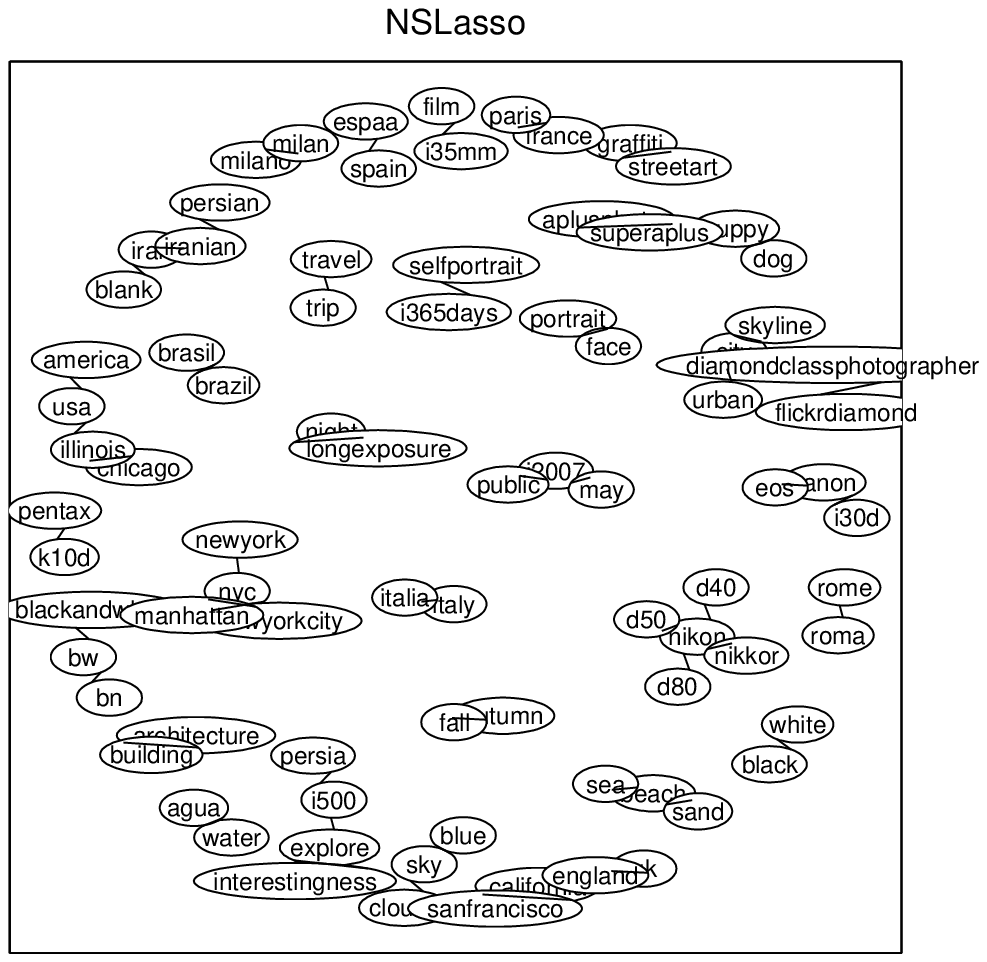}\label{fig:NSLasso_Graph_flickr_topk}
\includegraphics[width=35mm,height=35mm]{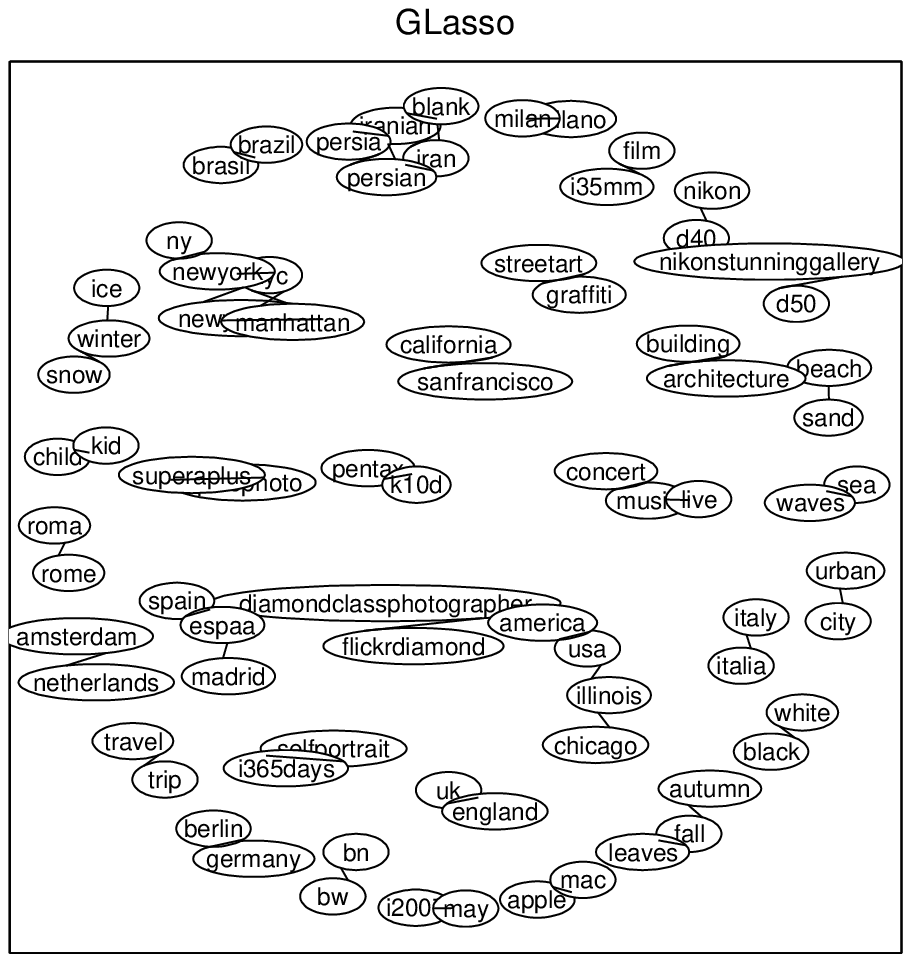}\label{fig:GLasso_Graph_flickr_topk}
\includegraphics[width=36mm,height=35mm]{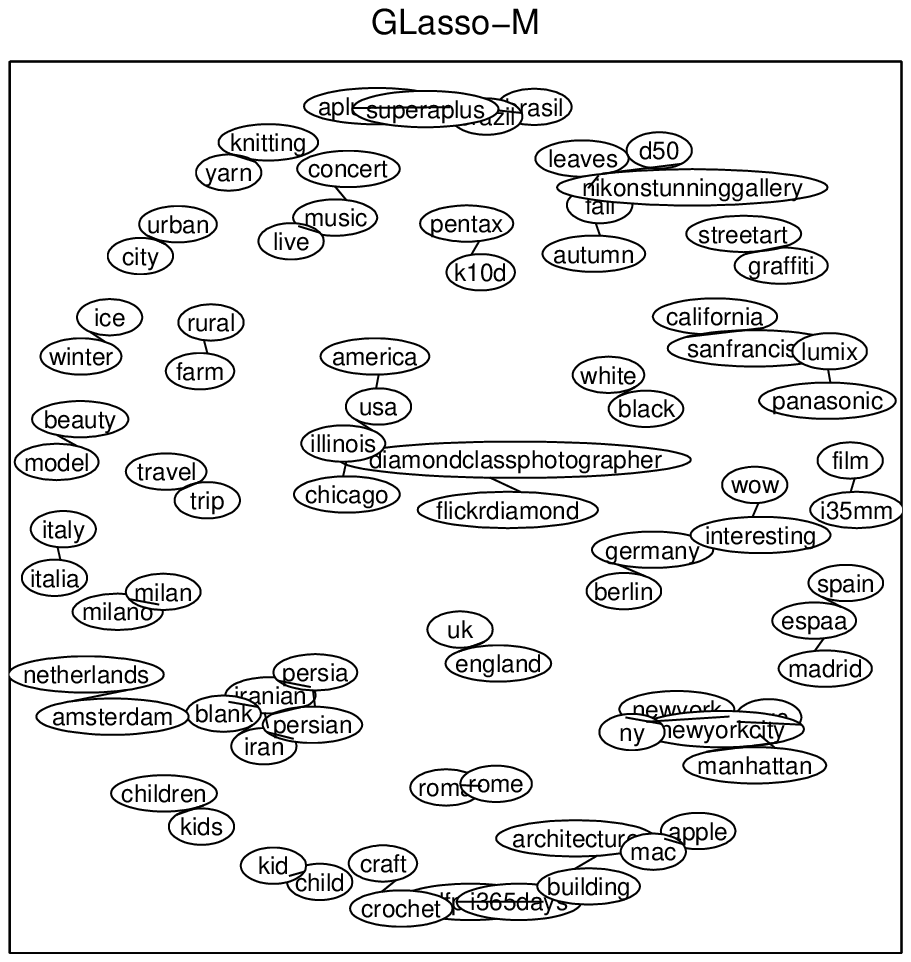}\label{fig:GLasso_OnlyY_Graph_flickr_topk}
\label{fig:graph_flickr_topk}} \subfigure[ \textsf{RCV1-v2}.]
{\includegraphics[width=35mm,height=35mm]{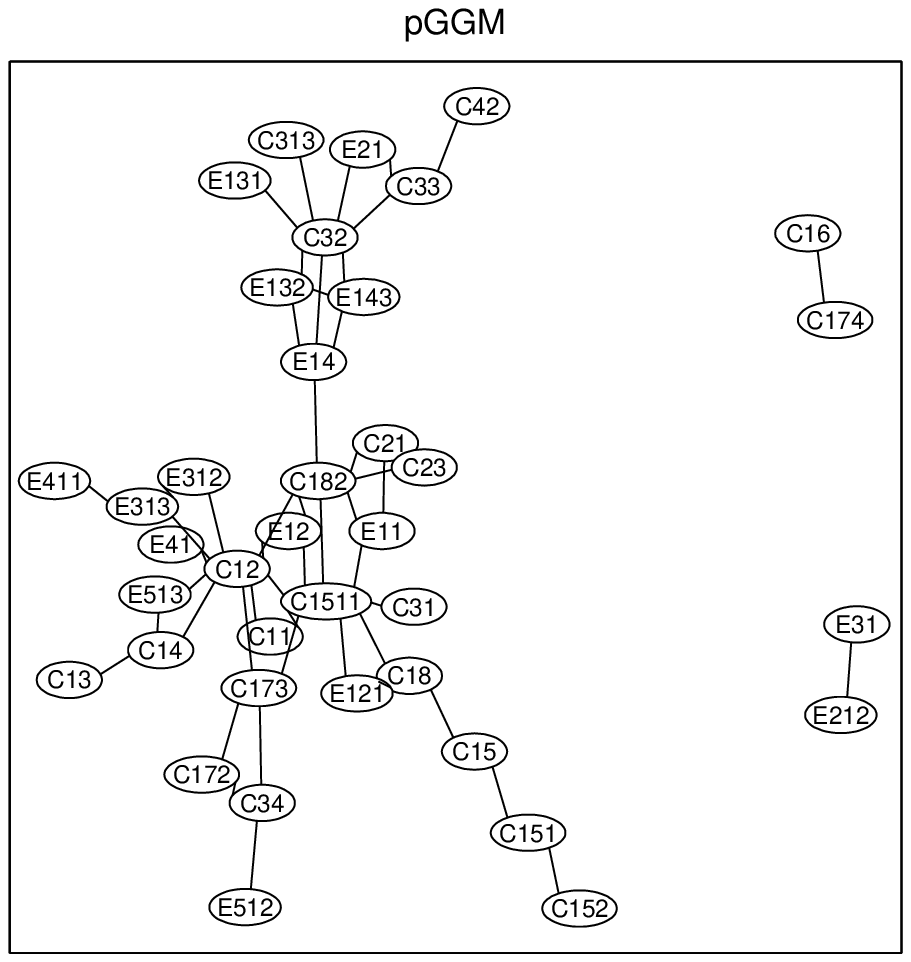}\label{fig:CPME_BCD_Graph_rcv1_topk}
\includegraphics[width=35mm,height=35mm]{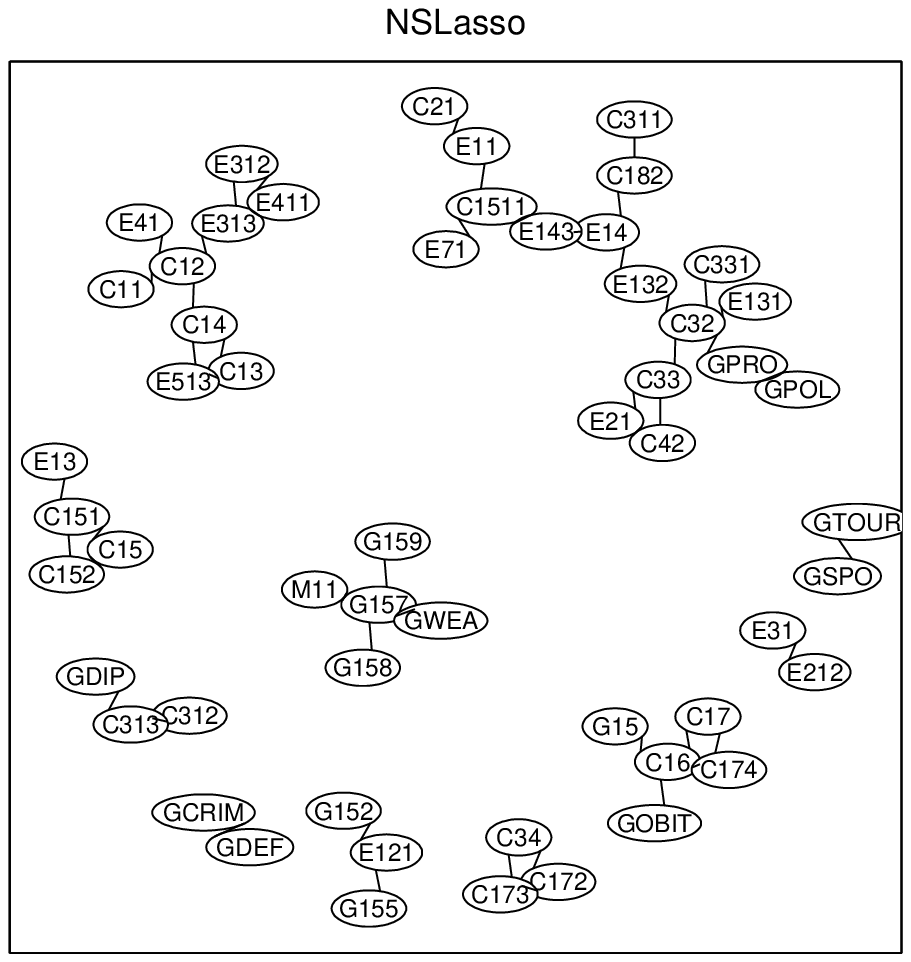}\label{fig:NSLasso_Graph_rcv1_topk}
\includegraphics[width=35mm,height=35mm]{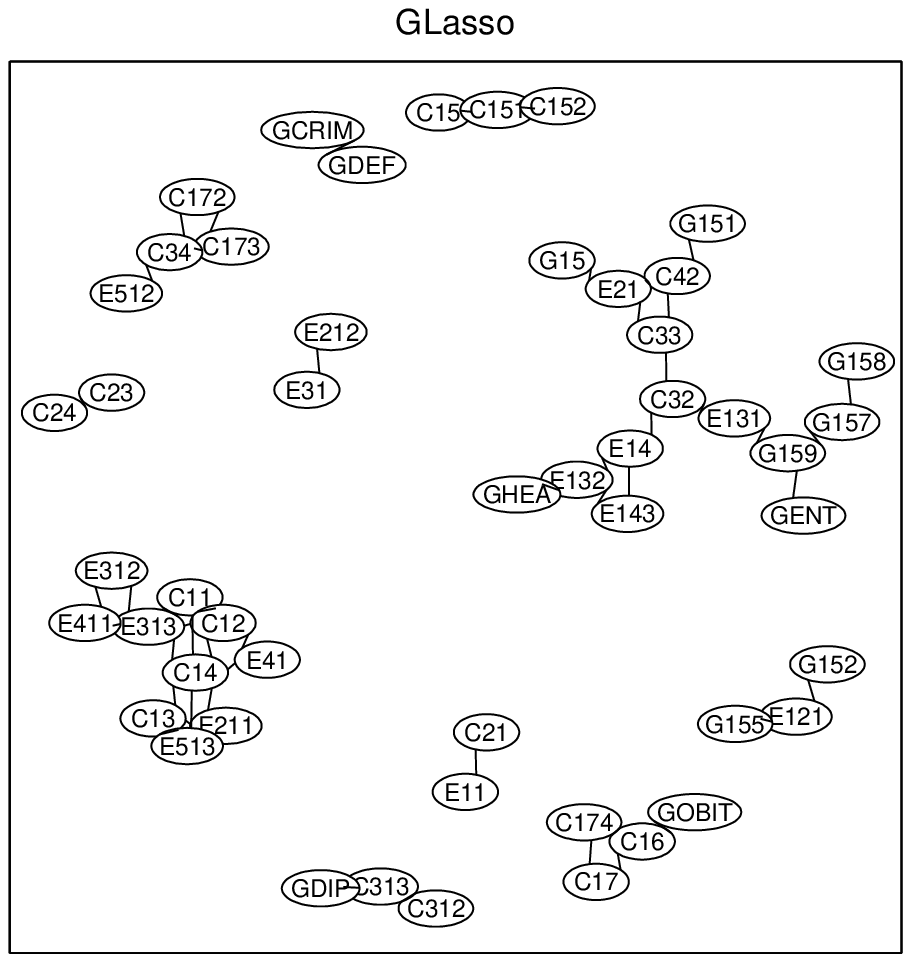}\label{fig:GLasso_Graph_rcv1_topk}
\includegraphics[width=35mm,height=35mm]{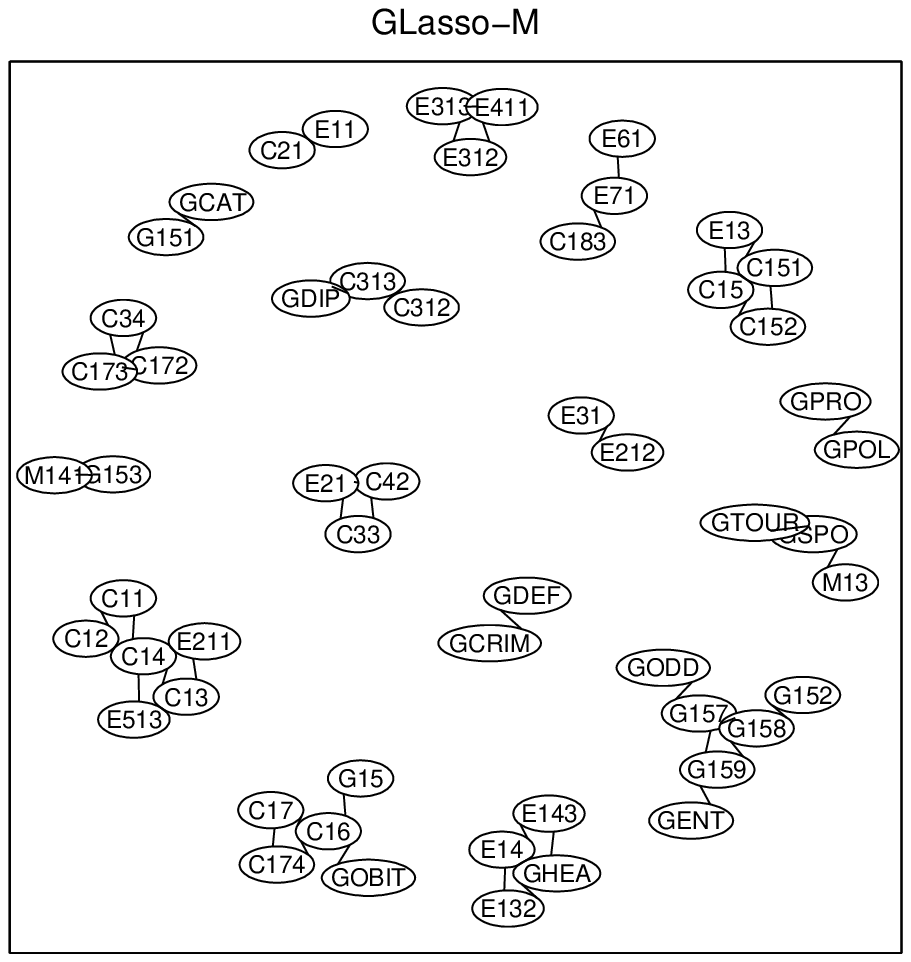}\label{fig:GLasso_OnlyY_Graph_rcv1_topk}
\label{fig:graph_rcv1_topk}} \subfigure[ \textsf{S\&P500}.]
{\includegraphics[width=35mm,height=35mm]{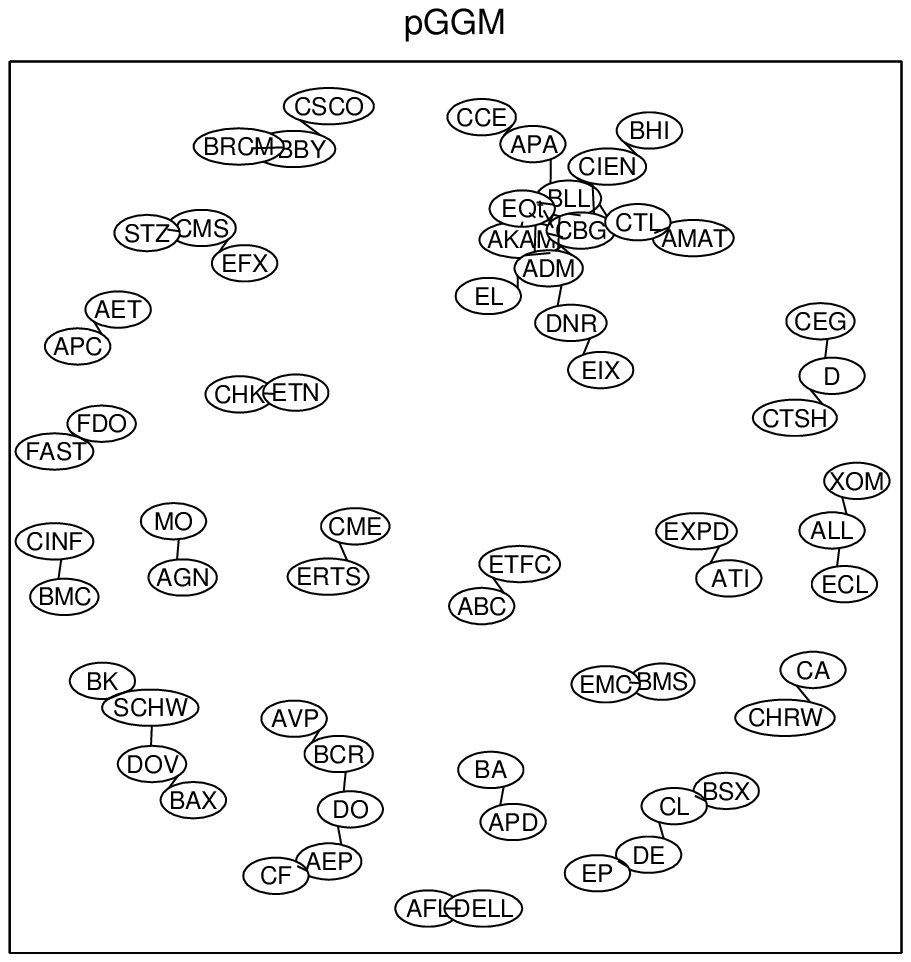}\label{fig:CPME_BCD_Graph_sp500_topk}
\includegraphics[width=35mm,height=35mm]{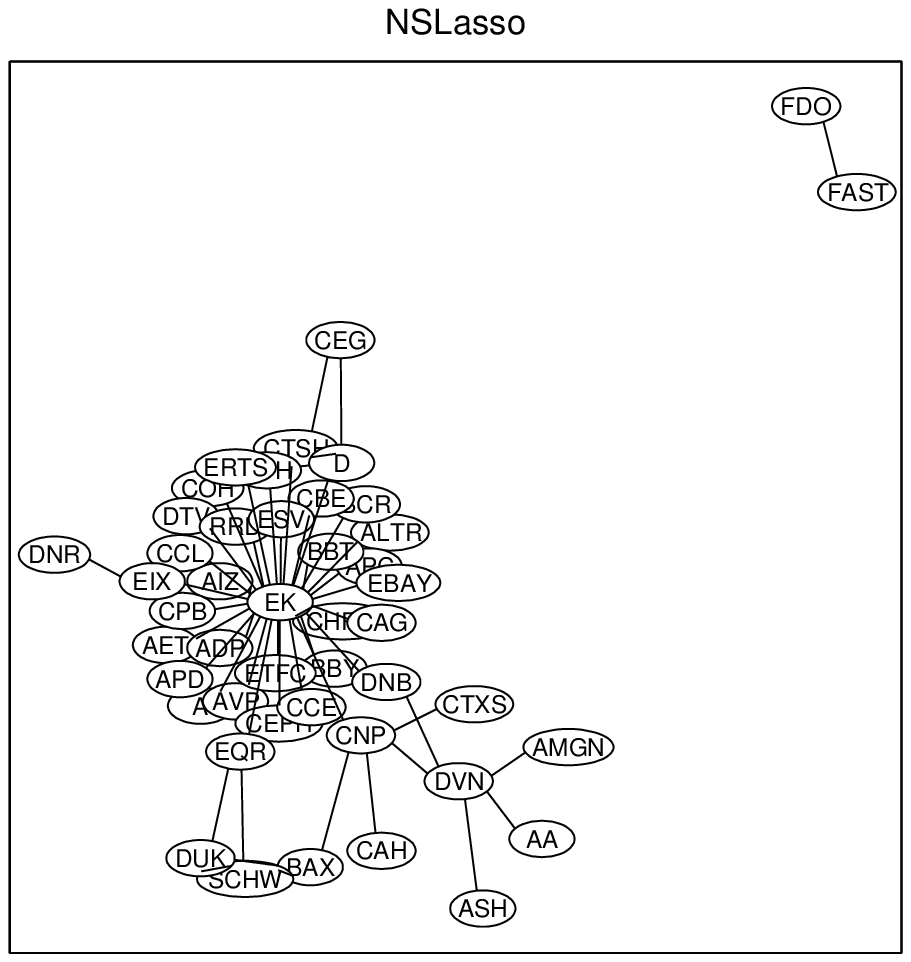}\label{fig:NSLasso_Graph_sp500_topk}
\includegraphics[width=35mm,height=35mm]{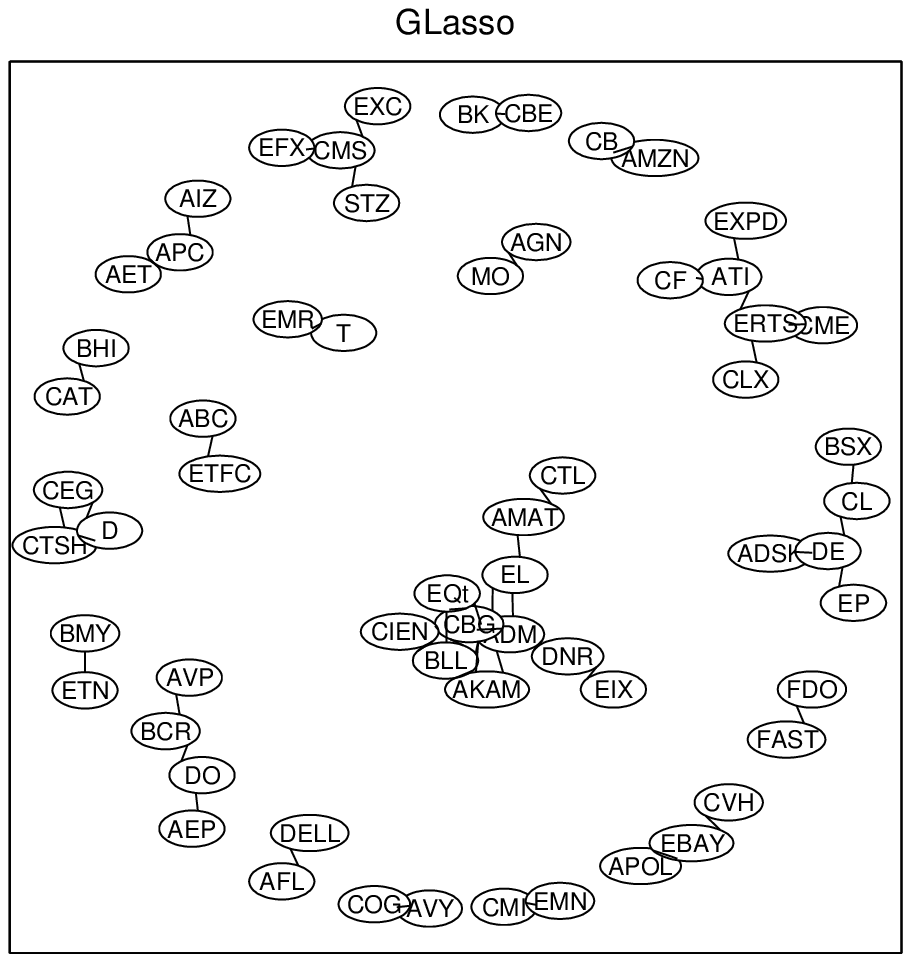}\label{fig:GLasso_Graph_sp500_topk}
\includegraphics[width=35mm,height=35mm]{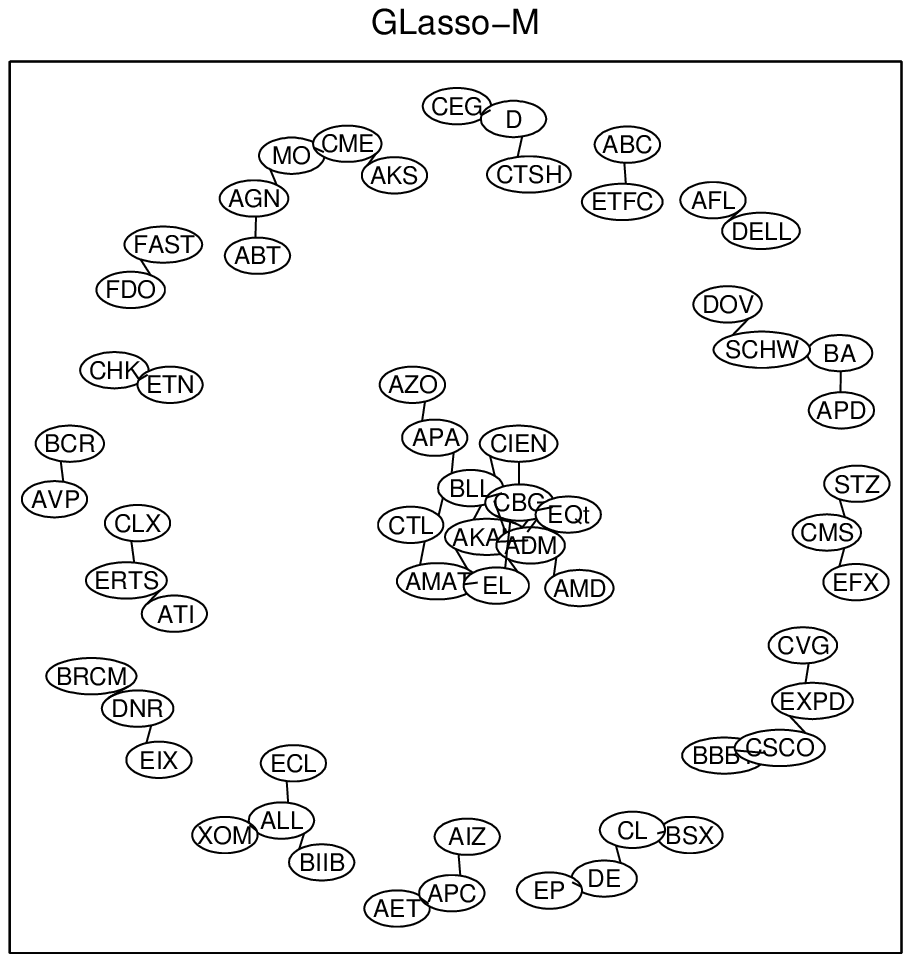}\label{fig:GLasso_OnlyY_Graph_sp500_topk}
\label{fig:graph_sp500_topk}} \caption{The top 50 links in the
constructed graphs by pGGM, NSLasso, GLasso and GLasso-M.
\label{fig:graph_topk}}
\end{figure}

\section{Conclusion}
\label{sect:conclusion}

This paper presents a new formulation pGGM for estimating sparse
partial precision matrix. The advantages of pGGM over prior GGMs and
conditional GGMs include: (i) the formulation is convex; (ii) the
optimization procedure scales well with respect to the component
$X$; (iii) the model has natural interpretation in terms of the
conditional dependency between the variables in $X$ and $Y$; and
(iv) theoretical guarantees on the global solution can be
established without sparsity assumptions on the precision matrix of
$X$. We showed that the rate of convergence of pGGM depends on how
sparse the underlying true partial precision matrix is. Numerical
experiments on several synthetic and real datasets demonstrated the
competitive performance of pGGM compared to the existing approaches.

In the current paper, the pGGM is derived under the assumption that
$(Y;X)$ is jointly normally distributed. As discussed in
Section~\ref{sec:Multivariate_Regression} that pGGM is still valid
in the setting where the joint normality is relaxed to the
conditional normality. We would like to point out that by assuming
the Gaussian copular structure of the random vector, pGGM can be
easily extended to the setting
of~\emph{nonparanormal}~\citep{Liu-Nonparanormal} which is a useful
tool for semiparametric estimation of high dimensional undirected
graphs. We believe that such an extension will broaden the
application range of pGGM in practice.

\bibliographystyle{icml2010}
\bibliography{pGGM}

\begin{thebibliography}{30}
\providecommand{\natexlab}[1]{#1}
\providecommand{\url}[1]{\texttt{#1}}
\expandafter\ifx\csname urlstyle\endcsname\relax
  \providecommand{\doi}[1]{doi: #1}\else
  \providecommand{\doi}{doi: \begingroup \urlstyle{rm}\Url}\fi

\bibitem[Banerjee et~al.(2008)Banerjee, Ghaoui, and
  d¡¯Aspremont]{Banerjee-2008}
Banerjee, O., Ghaoui, L.~El, and d¡¯Aspremont, A.
\newblock Model selection through sparse maximum likelihood estimation for
  multivariate gaussian or binary data.
\newblock \emph{Journal of Machine Learning Research}, 9:\penalty0 485--516,
  2008.

\bibitem[Baraniuk et~al.(2008)Baraniuk, Davenport, DeVore, and
  Wakin]{Baraniuk-SimpleRIP-2008}
Baraniuk, R.~G., Davenport, M., DeVore, R.~A., and Wakin, M.
\newblock A simple proof of the restricted isometry property for random
  matrices.
\newblock \emph{Constructive Approximation}, 2008.

\bibitem[Beck \& Teboulle(2009)Beck and Teboulle]{Beck-2009}
Beck, A. and Teboulle, Marc.
\newblock A fast iterative shrinkage-thresholding algorithm for linear inverse
  problems.
\newblock \emph{SIAM J. Imaging Sci.}, 2(1):\penalty0 183--202, 2009.

\bibitem[Boyed \& Vandenberghe(2004)Boyed and Vandenberghe]{Boyd-2004}
Boyed, S. and Vandenberghe, L.
\newblock \emph{Convex Optimization}.
\newblock Cambridge University Press, 2004.

\bibitem[Cai et~al.(2010)Cai, Li, Liu, and Xie]{Cai-CAPME}
Cai, T., Li, H., Liu, W., and Xie, J.
\newblock Covariate adjusted precision matrix estimation with an application in
  genetical genomics.
\newblock \emph{Biometrika}, 1:\penalty0 1--19, 2010.

\bibitem[Cai et~al.(2011)Cai, liu, and Luo]{Cai-CLIME-2011}
Cai, T., liu, W., and Luo, X.
\newblock A constrained $\ell_1$ minimization approach to sparse precision
  matrix estimation.
\newblock \emph{Journal of the American Statistical Association},
  106(494):\penalty0 594--607, 2011.

\bibitem[Cand\`es et~al.(2011)Cand\`es, Eldarb, Needella, and
  Randallc]{Candes-D-RIP-2011}
Cand\`es, E.~J., Eldarb, Y.~C., Needella, D., and Randallc, P.
\newblock Compressedsensing with coherent and redundantdictionarie.
\newblock \emph{Applied and Computational Harmonic Analysis}, 2011.

\bibitem[Cand\`es \& Tao(2007)Cand\`es and Tao]{CandTao:2005:dantzig}
Cand\`es, Emmanuel and Tao, Terence.
\newblock The {D}antzig selector: statistical estimation when $p$ is much
  larger than $n$.
\newblock \emph{Annals of Statistics}, 2007.

\bibitem[Chandrasekaran et~al.(2010)Chandrasekaran, Parrilo, and
  Willsky]{Venkat-LatentGGM}
Chandrasekaran, V., Parrilo, P., and Willsky, A.
\newblock Latent variable graphical model selection via convex optimization.
\newblock 2010.
\newblock URL \url{http://arXiv:1008.1290v1}.

\bibitem[d'Aspremont et~al.(2008)d'Aspremont, Banerjee, and
  Ghaoui]{Aspremont-2008}
d'Aspremont, A., Banerjee, O., and Ghaoui, L.~E.
\newblock First-order methods for sparse covariance selection.
\newblock \emph{SIAM Journal on Matrix Analysis and its Applications},
  30(1):\penalty0 56--66, 2008.

\bibitem[Dempster(1972)]{Dempster-1972}
Dempster, A.
\newblock Covariance selection.
\newblock \emph{Biometrics}, 28:\penalty0 157--175, 1972.

\bibitem[Duygulu et~al.(2002)Duygulu, Barnard, deFreitas, and
  Forsyth]{Corel5k-2002}
Duygulu, P., Barnard, K., deFreitas, N., and Forsyth, D.
\newblock Object recognition as machine translation: Learning a lexicon for a
  fixed image vocabulary.
\newblock In \emph{ECCV}, 2002.

\bibitem[Fan et~al.(2009)Fan, Feng, and Wu]{Fan-SCAD-2009}
Fan, J., Feng, Y., and Wu, Y.
\newblock Network exploration via the adaptive lasso and scad penalties.
\newblock \emph{The Annals of Applied Statistics}, 3(2):\penalty0 521--541,
  2009.

\bibitem[Friedman et~al.(2008)Friedman, Hastie, and
  Tibshirani]{Friedman-Glasso-2008}
Friedman, J., Hastie, T., and Tibshirani, R.
\newblock Sparse inverse covariance estimation with the graphical lasso.
\newblock \emph{Biostatistics}, 9(3):\penalty0 432--441, 2008.

\bibitem[Jansen \& Nap(2001)Jansen and Nap]{Jansen-eQTL}
Jansen, R. and Nap, J.
\newblock Genetic genomics: the added value from segregation.
\newblock \emph{Trends in Genetics}, 17(7):\penalty0 388--392, 2001.

\bibitem[Johnson et~al.(2012)Johnson, Jalali, and
  Ravikumar]{Johnson-AISTAT-2012}
Johnson, C., Jalali, A., and Ravikumar, P.
\newblock High-dimensional sparse inverse covariance estimation using greedy
  methods.
\newblock In \emph{AISTAT}, 2012.

\bibitem[Lafferty et~al.(2001)Lafferty, McCallum, and
  Pereira]{Lafferty-CRF-2001}
Lafferty, J., McCallum, A., and Pereira, F.
\newblock Conditional random fields: Probabilistic models for segmenting and
  labeling sequence data.
\newblock In \emph{ICML}, pp.\  282--289, 2001.

\bibitem[Laurent \& Massart(2000)Laurent and Massart]{LauMas00}
Laurent, B. and Massart, P.
\newblock Adaptive estimation of a quadratic functional by model selection.
\newblock \emph{The Annals of Statistics}, 28\penalty0 (5):\penalty0
  1302--1338, 2000.

\bibitem[Lewis et~al.(2004)Lewis, Yang, Rose, and Li]{Lewis-2004}
Lewis, D.D., Yang, Y., Rose, T.G., and Li, F.
\newblock Rcv1: A new benchmark collection for text categorization research.
\newblock \emph{Journal of Machine Learning Research}, 5:\penalty0 361--397,
  2004.

\bibitem[Liu et~al.(2009)Liu, Lafferty, and Wasserman]{Liu-Nonparanormal}
Liu, H., Lafferty, J., and Wasserman, L.
\newblock The nonparanormal: Semiparametric estimation of high dimensional
  undirected graphs.
\newblock \emph{Journal of Machine Learning Research}, 10:\penalty0 2295--2328,
  2009.

\bibitem[Lu(2009)]{Lu-VSM-2009}
Lu, Z.
\newblock Smooth optimization approach for sparse covariance selection.
\newblock \emph{SIAM Journal on Optimization}, 19(4):\penalty0 1807--1827,
  2009.

\bibitem[Meinshausen \& B{\"u}hlmann(2006)Meinshausen and
  B{\"u}hlmann]{Meinshausen-NSLasso-2006}
Meinshausen, N. and B{\"u}hlmann, P.
\newblock High-dimensional graphs and variable selection with the lasso.
\newblock \emph{Annals of Statistics}, 34(3):\penalty0 1436--1462, 2006.

\bibitem[Nesterov(2005)]{Nesterov-2005}
Nesterov, Yu.
\newblock Smooth minimization of non-smooth functions.
\newblock \emph{Mathematical Programming}, 103(1):\penalty0 127--152, 2005.

\bibitem[Rauhut et~al.(2008)Rauhut, Schnass, and Vandergheynst]{Rauhut-2008}
Rauhut, H., Schnass, K., and Vandergheynst, P.
\newblock Compressed sensing and redundant dictionaries.
\newblock \emph{IEEE Transactions on Inform. Theory}, 2008.

\bibitem[Ravikumar et~al.(2011)Ravikumar, Wainwright, Raskutti, and
  Yu]{Ravikumar-EJS-2011}
Ravikumar, P., Wainwright, M.~J., Raskutti, G., and Yu, B.
\newblock High-dimensional covariance estimation by minimizing
  $\ell_1$-penalized log-determinant divergence.
\newblock \emph{Electronic Journal of Statistics}, 5:\penalty0 935--980, 2011.

\bibitem[Rothman et~al.(2008)Rothman, Bickel, Levina, and Zhu]{Rothman-2008}
Rothman, A.~J., Bickel, P.~J., Levina, E., and Zhu, J.
\newblock Sparse permutation invariant covariance estimation.
\newblock \emph{Electronic Journal of Statistics}, 2:\penalty0 494--515, 2008.

\bibitem[St{\"a}dler et~al.(2010)St{\"a}dler, B{\"u}hlmann, and
  Geer]{Stadler-2010-mixture}
St{\"a}dler, N., B{\"u}hlmann, P., and Geer, S. Van~De.
\newblock $\ell_1$-penalization for mixture regression models.
\newblock \emph{TEST}, 19(2):\penalty0 209--256, 2010.

\bibitem[Yin \& Li(2011)Yin and Li]{Li-cGGM}
Yin, J. and Li, H.
\newblock A sparse conditional gaussian graphical model for analysis of general
  genomics data.
\newblock \emph{The Annals of Applied Statistics}, 5:\penalty0 2630--2650,
  2011.

\bibitem[Yuan(2010)]{Yuan-JMLR-2010}
Yuan, M.
\newblock High dimensional inverse covariance matrix estimation via linear
  programming.
\newblock \emph{Journal of Machine Learning Research}, 11:\penalty0 2261--2286,
  2010.

\bibitem[Yuan \& Lin(2007)Yuan and Lin]{Yuan-Lin-2007}
Yuan, M. and Lin, Y.
\newblock Model selection and estimation in the gaussian graphical model.
\newblock \emph{Biometrika}, 94(1):\penalty0 19--35, 2007.

\end{thebibliography}

\appendix

\section{Technical Proofs}

\subsection{Proof of Proposition~\ref{prop:decomp}}
\label{append:proof_lemma1}
\begin{proof}
Using the following well known fact of block matrix determinant
\[
\det  \left(\left[ { {\begin{array}{*{20}{c}}
   {A}   & {B^\top}   \\
   {B} & {C}   \\
\end{array}}
} \right]\right) = \det(A) \det(C - B A^{-1} B^\top)
\]
and simple algebra, we obtain that
\begin{equation}\label{equat:L_decompose_lemma1}
L(\Omega_{yy}, \Omega_{yx}, \Omega_{xx}) =
\Lpa(\Omega_{yy},\Omega_{yx}) -\log\det(\Omega_{xx}-\Omega_{yx}^\top
\Omega_{yy}^{-1}\Omega_{yx}) + \Tr(\Sigma^n_{xx}  (\Omega_{xx} -
\Omega_{yx}^\top \Omega_{yy}^{-1}\Omega_{yx})),
\end{equation}
where
\[
\Lpa(\Omega_{yy},\Omega_{yx}) =-\log\det(\Omega_{yy}) + \Tr
(\Sigma^n_{yy}\Omega_{yy}) + 2\Tr (\Sigma^{n\top}_{yx}\Omega_{yx}) +
\Tr(\Sigma^n_{xx}\Omega_{yx}^\top\Omega_{yy}^{-1}\Omega_{yx}).
\]
The claim \eqref{equat:L_decompose} follows immediately from the
re-parametrization of $ \tilde\Omega_{xx}= \Omega_{xx} -
\Omega_{yx}^\top \Omega_{yy}^{-1}\Omega_{yx}$.

We next show that $\Lpa(\Omega_{yy}, \Omega_{yx})$ is convex. Note
that when $\Sigma_{xx}^n \succ 0$, by minimizing both sides of
\eqref{equat:L_decompose_lemma1} over $\Omega_{xx}$, which is
achieved at $\Omega_{xx}= (\Sigma_{xx}^n)^{-1} + \Omega_{yx}^\top \Omega_{yy}^{-1}\Omega_{yx}$,
we know that up
to an additive constant, $\Lpa$ is the pointwise minimum of $L$ over
$\Omega_{xx}$. Since the
pointwise minimization of a convex objective function with a part of
the parameters is convex with respect to the other parameters~\citep[see, e.g.,][]{Boyd-2004},
we immediately obtain the convexity of $\Lpa$. In the high-dimensional
case where $n < q$, we only have $\Sigma^n_{xx}\succeq 0$ and thus
the minimization over $\Omega_{xx}$ is not well-defined. To show the
convexity in general case, we may replace $\Sigma_{xx}$ by
$\Sigma_{xx} + \lambda I$ for some $\lambda >0$, and the resulting
partial GMM formula:
\[
\Lpa^\lambda(\Omega_{yy},\Omega_{yx})=-\log\det(\Omega_{yy}) + \Tr
(\Sigma^n_{yy}\Omega_{yy}) + 2\Tr (\Sigma^{n\top}_{yx}\Omega_{yx}) +
\Tr((\Sigma^n_{xx}+ \lambda I)\Omega_{yx}^\top\Omega_{yy}^{-1}\Omega_{yx})
\]
is convex in $(\Omega_{yy},\Omega_{yx})$ by the previous argument. Now, let $\lambda
\to 0^+$, we have
$\Lpa^\lambda(\Omega_{yy},\Omega_{yx}) \to
\Lpa(\Omega_{yy},\Omega_{yx})$, which immediately implies the
convexity of $\Lpa(\cdot,\cdot)$.
\end{proof}

\subsection{Proof of Proposition~\ref{prop:RSC}}

\begin{lemma} \label{lem:shift}
Assume the conditions of the proposition hold.
Then for any matrix $ V=(V_{yy},V_{yx}) \in \bR^{p \times p} \times \bR^{p \times q}$ such that $|V_{\bar{S}}|_1 \leq \alpha |V_{S}|_1$,
we have
\[
\Tr(V \tilde{\Sigma} V^\top) \geq \frac{\rho_-}{5}  \| V\|_F^2 , \qquad
\text{ where } \quad
\tilde{\Sigma} = \left[ \begin{array}{cc} \Omega_{yy}^{-1} & 0 \\ 0 & \Sigma^n_{xx} \end{array} \right] .
\]
Moreover, we have
\[
\Tr(V_{yx} \Sigma^n_{xx} V_{yx}^\top) \leq 2.25 \rho_+ \|V\|_F^2 .
\]
\end{lemma}
\begin{proof}
In the following, we let $s=|S|$ and $s'=\tilde{s}-s \geq 4
(\rho_+/\rho_-) \alpha^2 s$. Since $r \leq
\lambda_{\max}(\Omega^*_{yy})$, we know that
$\lambda_{\max}(\Omega_{yy})^{-1}\geq \rho_-$. Indeed,
$\lambda_{\max}(\Omega_{yy})\le\lambda_{\max}(\Omega^*_{yy})+\lambda_{\max}(\Delta\Omega_{yy})\le
\lambda_{\max}(\Omega^*_{yy}) +r \le
2\lambda_{\max}(\Omega^*_{yy})$, which from the definition of
$\rho^-$ implies that $\lambda_{\max}(\Omega_{yy})^{-1}\geq \rho_-$.
Therefore for any $U \in\bR^{p \times (p+q)}$ such that $|U|_0 \leq
s+s'$, the conditions of Assumption~\ref{assump:rip} imply that
\[
\Tr(U \tilde{\Sigma} U^\top) \geq \rho_- \|U\|_F^2 .
\]
We order the elements of $V_{\bar{S}}$ in descending order of absolute values.
Let $V^{(0)}=V_{S}$ which contains $s$ nonzero values, and $V^{(k)}$ contains (at most) $s'$ nonzero values of $V_{\bar{S}}$ with $(ks' -s'+1)$-th to $(ks')$-th largest absolute values.
It follows that $\|V^{(k+1)}\|_F \leq \sqrt{|V^{(k+1)}|_\infty |V^{(k+1)}|_1} \leq |V^{(k)}|_1/\sqrt{s'}$ for all $k \geq 1$.
Therefore we have
\[
a_0 = \Tr ((V^{(0)}+V^{(1)}) \tilde{\Sigma} (V^{(0)}+V^{(1)})^\top) \geq \rho_- \|V^{(0)}+V^{(1)}\|_F^2
\]
and
\begin{align*}
a_1 =& \big| \Tr ((V^{(0)}+V^{(1)}) \tilde{\Sigma} \sum_{k \geq 1} V^{(k+1)\top})  \big| \\
\leq & \sqrt{a_0}
\sqrt{\rho_+} \sum_{k \geq 1} \|V^{(k+1)}\|_F \\
\leq& \sqrt{a_0 \rho_+} \sum_{k \geq 1} |V^{(k)}|_1/\sqrt{s'} \\
\leq& \alpha \sqrt{a_0 \rho_+} |V_{S}|_1/\sqrt{s'} \leq \alpha \sqrt{a_0 \rho_+} \|V^{(0)}+V^{(1)}\|_F \sqrt{s/s'} .
\end{align*}

Note that
$\Tr(V \tilde{\Sigma} V^\top) \geq a_0 - 2 a_1 + a_2$, where
\[
a_2 =  \Tr\left( \left(\sum_{k \geq 1} V^{(k+1)}\right) \tilde{\Sigma} \left(\sum_{k \geq 1} V^{(k+1)}\right)^\top \right) .
\]
The semi-positive-definiteness of $\tilde{\Sigma}$ implies that
$\min_\mu [a_0 + 2 \mu a_1 + \mu^2 a_2] \geq 0$, which implies that $a_1^2 \leq a_0 a_2$.
Therefore
\begin{align*}
\Tr(V^\top \tilde{\Sigma} V) \geq& a_0 - 2 a_1 + a_2 \geq a_0 - 2 a_1 + a_1^2/ a_0 \\
\geq& \rho_- \|V^{(0)}+V^{(1)}\|_F^2  (1- \alpha \sqrt{(\rho_+/\rho_-) (s/s')})^2
\geq \rho_- \|V^{(0)}+V^{(1)}\|_F^2  / 4 ,
\end{align*}
where the last inequality is due to the definition of $s'$ that
implies that $\alpha \sqrt{(\rho_+/\rho_-) (s/s')} \leq 0.5$.

Moreover we have
\begin{align*}
\|V\|_F^2 =& \|V^{(0)}+V^{(1)}\|_F^2 + \sum_{k \geq 1} \|V^{(k+1)}\|_F^2 \\
\leq& \|V^{(0)}+V^{(1)}\|_F^2 + \sum_{k \geq 1} \|V^{(k)}\|_1^2 /s' \\
\leq& \|V^{(0)}+V^{(1)}\|_F^2 + \|V^{(1)}\|_1 \|V_{\bar{S}}\|_1 /s' \\
\leq& \|V^{(0)}+V^{(1)}\|_F^2 + \alpha \|V^{(1)}\|_2 \|V_{S}\|_2 \sqrt{s/s'} \\
\leq & (1+ 0.5 \alpha \sqrt{s /s'})\|V^{(0)}+V^{(1)}\|_F^2
\leq 1.25 \|V^{(0)}+V^{(1)}\|_F^2 .
\end{align*}
By combining the previous two displayed inequalities, we obtain the first desired bound.

To prove the second bound, we define
\[
\tilde{\Sigma}' = \left[ \begin{array}{cc} 0_{p \times p} & 0 \\ 0 & \Sigma^n_{xx} \end{array} \right] .
\]
Therefore for any $U \in\bR^{p \times (p+q)}$ such that $|U|_0 \leq s+s'$,
the conditions of Assumption~\ref{assump:rip} imply that
\[
\Tr(U \tilde{\Sigma}' U^\top) \leq \rho_+ \|U\|_F^2 .
\]
Therefore we have
\[
a_0' = \Tr ((V^{(0)}+V^{(1)}) \tilde{\Sigma}' (V^{(0)}+V^{(1)})^\top) \leq \rho_+ \|V\|_F^2
\]
and
\begin{align*}
a_2'=&\Tr\left( \left(\sum_{k \geq 1} V^{(k+1)}\right) \tilde{\Sigma}' \left(\sum_{k \geq 1} V^{(k+1)}\right)^\top \right) \\
\leq& \sum_{k \geq 1} \sum_{k' \geq 1} \Tr(V^{(k+1)} \tilde{\Sigma}' V^{(k'+1)\top}) \\
\leq& \rho_+ \sum_{k \geq 1} \sum_{k' \geq 1} \|V^{(k+1)}\|_F \|V^{(k'+1)}\|_F \\
\leq& \rho_+ \sum_{k \geq 1} \sum_{k' \geq 1} |V^{(k)}|_1 |V^{(k')}|_1 /s'
\leq  \rho_+ |V_{\bar{S}}|_1^2 /s' \\
\leq& \alpha^2  \rho_+ |V_{S}|_1^2 /s'
\leq \alpha^2  \rho_+ \|V\|_F^2 (s/s') .
\end{align*}
Therefore we obtain (using $\alpha^2 (s/s') \leq 0.25$)
\[
\Tr(V_{yx} \Sigma^n_{xx} V_{yx}^\top) \leq a_0' + 2 \sqrt{a_0' a_2'} + a_2'
\leq 1.5 a_0' + 3 a_2'
\leq (1.5 + 3/4) \rho_+ \|V\|_F^2 = 2.25 \|V\|_F^2 .
\]
This completes the proof.
\end{proof}

\begin{lemma} \label{lem:vartheta}
Let
\[
\vartheta=
\min\left[\frac{2}{3} ,
\frac{\lambda_{\min}(\Omega^*_{yy})}{8\lambda_{\max}(\Omega^*_{yx} \Sigma^*_{xx}  (\Omega^*_{yx})^\top)}
\right] ,
\]
 then we have
 \[
 \lambda_{\max}(\Omega_{yy}^{-1}\Omega_{yx} \Sigma^n_{xx}  \Omega_{yx}^\top) \leq 1/(2\vartheta) .
 \]
\end{lemma}
\begin{proof}
Let $\sigma(A)$ be the largest singular value of a matrix $A$, then
$\sigma(A)=\sqrt{\lambda_{\max}(A^\top A)}$. Therefore we have
  \begin{align*}
    \sqrt{\lambda_{\max} \big[\Omega_{yx} \Sigma^n_{xx}  \Omega_{yx}^\top\big]}
    =&\sigma (\Omega_{yx} (\Sigma^n_{xx})^{1/2}) \\
    \leq &\sigma (\Omega^*_{yx} (\Sigma^n_{xx})^{1/2}) + \sigma (\Delta \Omega_{yx} (\Sigma^n_{xx})^{1/2}) \\
    \leq &\sigma (\Omega^*_{yx} (\Sigma^n_{xx})^{1/2}) + \sqrt{\Tr (\Delta \Omega_{yx} \Sigma^n_{xx} \Delta \Omega_{yx}^\top)} \\
    \leq &    \sqrt{\lambda_{\max} \big[\Omega^*_{yx} \Sigma^n_{xx}  (\Omega^*_{yx})^\top\big]} + 1.5 \sqrt{\rho_+} \|\Delta \Omega\|_F \\
    \leq & 1.4 \sqrt{\lambda_{\max} \big[\Omega^*_{yx} \Sigma^*_{xx}  (\Omega^*_{yx})^\top\big]} ,
  \end{align*}
  where the third inequality uses the second inequality of
  Lemma~\ref{lem:shift}, and the last inequality uses the third inequality
  of Assumption~\ref{assump:rip} and
  $\|\Delta \Omega\|_F \leq r \leq 0.13 \sqrt{\lambda_{\max} \big[\Omega^*_{yx} \Sigma^*_{xx}  (\Omega^*_{yx})^\top\big]/\rho_+}$.
  This implies
  \[
  \lambda_{\max} \big[\Omega_{yx} \Sigma^n_{xx}  \Omega_{yx}^\top\big]
  \leq 2 \lambda_{\max} \big[\Omega^*_{yx} \Sigma^*_{xx}  (\Omega^*_{yx})^\top\big] .
  \]
  Since the assumption of $r \leq \lambda_{\min}(\Omega^*_{yy})/2$
  also implies that
  \[
  \lambda_{\min}(\Omega_{yy}) \ge
  \lambda_{\min}(\Omega^*_{yy})-\lambda_{\min}(\Delta\Omega_{yy})
  \ge \lambda_{\min}(\Omega^*_{yy}) - r \ge
  \lambda_{\min}(\Omega^*_{yy})/2 .
  \]
  Therefore we have
  \[
  \lambda_{\max}(\Omega_{yy}^{-1}\Omega_{yx} \Sigma^n_{xx}  \Omega_{yx}^\top) \leq
  \frac{\lambda_{\max}(\Omega_{yx} \Sigma^n_{xx}  \Omega_{yx}^\top)}{\lambda_{\min}(\Omega_{yy})}
  \leq
  \frac{4 \lambda_{\max} \big[\Omega^*_{yx} \Sigma^*_{xx}  (\Omega^*_{yx})^\top\big]}{\lambda_{\min}(\Omega^*_{yy})}
  = 1/(2 \vartheta) ,
  \]
  which leads to the desired bound.
\end{proof}

\begin{proof}[Proof of Proposition~\ref{prop:RSC}]
For any $s \in (0,1)$, we define for convenience that
\[
\Omega_{yy} = \Omega^*_{yy} + s\Delta\Omega_{yy}, \quad \Omega_{yx}
= \Omega^*_{yx} + s\Delta\Omega_{yx} ,
\]
and consider the function $f(s)$ defined as
\[
f(s):= -\log\det(\Omega_{yy}) + \Tr (\Sigma^n_{yy}\Omega_{yy}) +
2\Tr (\Sigma^{n\top}_{yx}\Omega_{yx}) +
\Tr(\Sigma^n_{xx}\Omega_{yx}^\top\Omega_{yy}^{-1}\Omega_{yx}).
\]
It can be verified that
\begin{align*}
f'(s) = &
- \Tr (\Omega_{yy}^{-1} \Delta {\Omega}_{yy})
+ \Tr(\Sigma^{n}_{yy}\Delta \Omega_{yy}) \\
&
+ 2 \Tr(\Sigma^{n\top}_{yx}\Delta \Omega_{yx})
+ 2\Tr(\Sigma^n_{xx}\Omega_{yx}^\top\Omega_{yy}^{-1}\Delta \Omega_{yx})
- \Tr(\Sigma^n_{xx}\Omega_{yx}^\top\Omega_{yy}^{-1} \Delta \Omega_{yy}\Omega_{yy}^{-1} \Omega_{yx
})
\end{align*}
and
\begin{align*}
f''(s)
=&
 \Tr (\Omega_{yy}^{-1} \Delta {\Omega}_{yy}\Omega_{yy}^{-1} \Delta {\Omega}_{yy})
+ 2\Tr(\Sigma^n_{xx}\Delta \Omega_{yx}^\top\Omega_{yy}^{-1}\Delta \Omega_{yx})
- 4\Tr(\Sigma^n_{xx} \Omega_{yx}^\top\Omega_{yy}^{-1}\Delta \Omega_{yy}\Omega_{yy}^{-1}\Delta \Omega_{yx})\\
&
+  2 \Tr(\Sigma^n_{xx}\Omega_{yx}^\top\Omega_{yy}^{-1} \Delta \Omega_{yy}\Omega_{yy}^{-1}\Delta \Omega_{yy}\Omega_{yy}^{-1} \Omega_{yx}) .
\end{align*}
We obtain from Taylor expansion that
  \[
  \Delta \Lpa(\Theta^*, \Delta\Theta)= \frac{1}{2}f''(s), \quad
  \text{for some } s \in (0,1).
  \]
  This implies that
\begin{eqnarray*}
f''(s) &=&
\Tr(\Omega_{yy}^{-1}\Delta\Omega_{yy}\Omega_{yy}^{-1}\Delta\Omega_{yy})
+
2\Tr(\Sigma^n_{xx}\Delta\Omega^\top_{yx}\Omega^{-1}_{yy}\Delta\Omega_{yx})
-
4\Tr(\Sigma^n_{xx}\Omega^\top_{yx}\Omega^{-1}_{yy}\Delta\Omega_{yy}\Omega_{yy}^{-1}\Delta\Omega_{yx})
\\
&&
+2\Tr(\Sigma_{xx}^n\Omega^\top_{yx}\Omega_{yy}^{-1}\Delta\Omega_{yy}\Omega^{-1}_{yy}\Delta\Omega_{yy}\Omega_{yy}^{-1}\Omega_{yx})
\\
&=&
\Tr(\Omega_{yy}^{-1}\Delta\Omega_{yy}\Omega_{yy}^{-1}\Delta\Omega_{yy})
+
2\Tr(\Sigma^n_{xx}\Delta\Omega^\top_{yx}\Omega^{-1}_{yy}\Delta\Omega_{yx})-
4\Tr(\Sigma^n_{xx}\Omega^\top_{yx}\Omega^{-1}_{yy}\Delta\Omega_{yy}\Omega_{yy}^{-1}\Delta\Omega_{yx})
\\
&&
+(2+\vartheta)\Tr(\Sigma_{xx}^n\Omega^\top_{yx}\Omega_{yy}^{-1}\Delta\Omega_{yy}\Omega^{-1}_{yy}\Delta\Omega_{yy}\Omega_{yy}^{-1}\Omega_{yx})
- \vartheta\Tr(\Sigma_{xx}^n\Omega^\top_{yx}\Omega_{yy}^{-1}\Delta\Omega_{yy}\Omega^{-1}_{yy}\Delta\Omega_{yy}\Omega_{yy}^{-1}\Omega_{yx})
\\
&\ge&
\Tr(\Omega_{yy}^{-1}\Delta\Omega_{yy}\Omega_{yy}^{-1}\Delta\Omega_{yy})
+
\frac{2\vartheta}{2+\vartheta}\Tr(\Sigma^n_{xx}\Delta\Omega^\top_{yx}\Omega^{-1}_{yy}\Delta\Omega_{yx})
\\
&& -\vartheta
\Tr(\Omega_{yy}^{-1/2}\Omega_{yx}\Sigma_{xx}^n\Omega^\top_{yx}\Omega_{yy}^{-1/2} \; \Omega_{yy}^{-1/2}\Delta\Omega_{yy}\Omega^{-1}_{yy}\Delta\Omega_{yy}\Omega_{yy}^{-1/2})
\\
&\ge&
\Tr(\Omega_{yy}^{-1}\Delta\Omega_{yy}\Omega_{yy}^{-1}\Delta\Omega_{yy})
+
\frac{2\vartheta}{2+\vartheta}\Tr(\Sigma^n_{xx}\Delta\Omega^\top_{yx}\Omega^{-1}_{yy}\Delta\Omega_{yx})
\\
&&-\vartheta\lambda_{\max}(\Omega_{yy}^{-1/2}\Omega_{yx}\Sigma_{xx}^n\Omega^\top_{yx}\Omega_{yy}^{-1/2})
\Tr(\Omega_{yy}^{-1/2}\Delta\Omega_{yy}\Omega^{-1}_{yy}\Delta\Omega_{yy}\Omega_{yy}^{-1/2})
\\
&\ge& 0.5
\Tr(\Omega_{yy}^{-1}\Delta\Omega_{yy}\Omega_{yy}^{-1}\Delta\Omega_{yy})
+
\frac{2\vartheta}{2+\vartheta}\Tr(\Sigma^n_{xx}\Delta\Omega^\top_{yx}\Omega^{-1}_{yy}\Delta\Omega_{yx})
,
\end{eqnarray*}
where we have used the trace equality $\Tr(AB)=\Tr(BA)$ throughout the derivations.
The first inequality is due to the trace
inequality $(2/(2+\vartheta))\Tr(A^\top A) - 4 \Tr(A^\top
B)+(2+\vartheta)\Tr(B^\top B) \geq 0$;
the second inequality uses $\Tr(AB) \leq \lambda_{\max}(A) \Tr(B)$ for
symmetric positive semidefinite matrices $A$ and $B$;
and the last inequality is due to $\lambda_{\max}(\Omega_{yy}^{-1/2}\Omega_{yx} \Sigma^n_{xx}  \Omega_{yx}^\top \Omega_{yy}^{-1/2}) \vartheta \leq 1/2$
(Lemma~\ref{lem:vartheta}).

Since $\vartheta \leq 2/3$, we have $0.5 \geq 2\vartheta/(2+\vartheta)$.
Therefore
\begin{align*}
2 \Delta \Lpa(\Theta^*, \Delta\Theta)
=& f''(s) \\
\geq &
\frac{2\vartheta}{2+\vartheta}
\left[
\Tr(\Omega_{yy}^{-1}\Delta\Omega_{yy}\Omega_{yy}^{-1}\Delta\Omega_{yy})
+
\Tr(\Sigma^n_{xx}\Delta\Omega^\top_{yx}\Omega^{-1}_{yy}\Delta\Omega_{yx})
\right]
\\
\geq &
\frac{2\vartheta}{2+\vartheta}
\lambda_{\max}^{-1}(\Omega_{yy})
\left[
\Tr(\Delta\Omega_{yy}\Omega_{yy}^{-1}\Delta\Omega_{yy})
+
\Tr(\Delta\Omega_{yx}\Sigma^n_{xx}\Delta\Omega^\top_{yx})
\right]
\\
\geq &
\frac{2\vartheta   \lambda_{\max}^{-1}(\Omega_{yy}) \rho_-}{5(2+\vartheta)}
\|\Delta\Theta\|_F^2 ,
\end{align*}
where the second inequality uses $\Tr(AB) \geq \lambda_{\min}(A) \Tr(B)$ for symmetric positive semidefinite matrices $A$ and $B$;
and the last inequality follows from Lemma~\ref{lem:shift}.
We complete the proof by noticing $5(2+\vartheta) \leq 40/3$.
\end{proof}

\subsection{Proof of Proposition~\ref{prop:gamma_convergence}}

We will employ the following tail-bound for $\chi^2$ random variable, due to \citet{LauMas00}.
\begin{lemma}\label{lemma:xi2_bound}
Consider independent Gaussian random variables $z_1,\ldots,z_n \sim
\mathcal {N}(0,\sigma^2)$. We have for all $t >0$:
\[
\Pr\left[ \sum_{\ell=1}^n z_\ell^2 \ge n \sigma^2 + 2 \sigma^2
  \sqrt{n t} + 2 \sigma^2 t \right] \le e^{-t}
\]
and
\[
\Pr\left[ \sum_{\ell=1}^n z_\ell^2 \le n \sigma^2 - 2 \sigma^2
  \sqrt{n t} \right] \le e^{-t} .
\]
\end{lemma}
The following lemma is a consequence of
Lemma~\ref{lemma:xi2_bound} when applied to the covariance of multivariate Gaussian distribution.
\begin{lemma}\label{lemma:cov_bound}
Consider the covariance matrix $\Sigma^*$ of a $d$-dimensional
Gaussian random vector and its sample covariance $\Sigma^n$ from $n$
i.i.d. Gaussian random vectors from $\mathcal {N}(0,\Sigma^*)$. For
any $\eta \in (0,1)$ and any deterministic $d' \times d$ matrix $A$.
Let
\[
\sigma^2 = \max_{ij} \left[ (A\Sigma^* A^\top)_{ii} + 2 |(A\Sigma^*)_{ij}| + (\Sigma^*)_{jj} \right] ,
\]
then with probability at least $1-\eta$ for any $\eta \in (0,1)$, we have
\[
|A(\Sigma^n - \Sigma^*)|_\infty \le 2 \sigma^2 \sqrt{\ln (4 d d'/\eta) /n} ,
\]
provided that $n \geq \ln (4 d d'/\eta)$.
\end{lemma}
\begin{proof}
Consider the multivariate Gaussian random vector $X^{(1)},\dots, X^{(n)}\sim \cN(0, \Sigma^*)$.

Given any index pair $(i,j)$, let $z^{(\ell)}= (A X^{(\ell)})_i +
X^{(\ell)}_j$. We have $z^{(\ell)} \sim \cN(0,(A\Sigma^*
A^\top)_{ii} + 2 (A\Sigma^*)_{ij} + (\Sigma^*)_{jj})$. We thus
obtain from Lemma~\ref{lemma:xi2_bound} that for $t \leq n$: with
probability at least $1-2 e^{-t}$,
\[
\left| n^{-1} \sum_{\ell=1}^n (A X^{(\ell)})_i + X^{(\ell)}_j)^2 - [(A\Sigma^* A^\top)_{ii} + 2 (A\Sigma^*)_{ij} + (\Sigma^*)_{jj}] \right| \leq 4 \sigma^2 \sqrt{t/n} .
\]
Similarly, we have for $t \leq n$: with probability at least $1-
2e^{-t}$,
\[
\left| n^{-1} \sum_{\ell=1}^n (A X^{(\ell)})_i - X^{(\ell)}_j)^2 - [(A\Sigma^* A^\top)_{ii} - 2 (A\Sigma^*)_{ij} + (\Sigma^*)_{jj}] \right| \leq 4 \sigma^2 \sqrt{t/n} .
\]
Taking union bound, and adding the previous two inequalities, we
obtain that with probability at least $1-4e^{-t}$:
\begin{align*}
& \left| \left[n^{-1} \sum_{\ell=1}^n (A X^{(\ell)})_i + X^{(\ell)}_j)^2 - [(A\Sigma^* A^\top)_{ii} + 2 (A\Sigma^*)_{ij} + (\Sigma^*)_{jj}] \right] \right. \\
& \left. -\left[ n^{-1} \sum_{\ell=1}^n (A X^{(\ell)})_i - X^{(\ell)}_j)^2 - [(A\Sigma^* A^\top)_{ii} - 2 (A\Sigma^*)_{ij} + (\Sigma^*)_{jj}]  \right] \right|
\leq 8 \sigma^2 \sqrt{t/n} .
\end{align*}
This simplifies to
$|A(\Sigma^n - \Sigma^*)_{ij}| \leq 2 \sigma^2 \sqrt{t/n}$.
Now by taking union bound over $i=1,\ldots,d'$ and $j=1,\ldots,d$, and set
$\eta= 4 d d' e^{-t}$, we obtain the desired bound.
\end{proof}

Note that in Lemma~\ref{lemma:cov_bound}, we have
$\sigma^2 \leq 2 \max_i (A \Sigma^* A^\top)_{ii} + 2 \max_i (\Sigma^*)_{ii}$.
It implies that with probability $1-\eta$:
\begin{equation} \label{eq:cov_bound}
|A(\Sigma^n - \Sigma^*)|_\infty \le 4
[\max_i (A \Sigma^* A^\top)_{ii} + \max_i (\Sigma^*)_{ii}]
\sqrt{\ln (4 d d'/\eta) /n}
\end{equation}
when $n \geq \ln (4 d d'/\eta)$.

\begin{proof}[Proof of Proposition~\ref{prop:gamma_convergence}]
For any $\eta \in (0,1)$ such that $n \geq \ln (10 (p+q)^2/\eta)$,
 we obtain from \eqref{eq:cov_bound} with $A=I$ that with probability $1-0.4\eta$:
\[
|\Sigma^n - \Sigma^*|_\infty\le 8 \max_i (\Sigma^*)_{ii} \sqrt{\ln (10 (p+q)^2/\eta) /n}  .
\]
Let $\tilde{A}=(\Omega^*_{yy})^{-1}\Omega^*_{yx}$. We may also apply
\eqref{eq:cov_bound} to the Gaussian covariance matrix $\tilde{A}
\Sigma^*_{xx} \tilde{A}^\top$ and $A=I$ to obtain that with
probability $1-0.4\eta$:
\[
|\tilde{A} \Sigma^n_{xx} \tilde{A}^\top - \tilde{A} \Sigma^*_{xx}
\tilde{A}^\top |_\infty\le 8 \max_i (\tilde{A} \Sigma^*_{xx}
\tilde{A}^\top)_{ii} \sqrt{\ln (10 q^2/\eta) /n}  .
\]
Similarly, we may also apply \eqref{eq:cov_bound} to the Gaussian covariance matrix $\Sigma^*$ with $A=\tilde{A}$ to obtain that with probability $1-0.2\eta$:
\[
|\tilde{A} \Sigma^n_{xx} - \tilde{A} \Sigma^*_{xx}|_\infty \le 8
\max_i (\tilde{A} \Sigma^*_{xx} \tilde{A}^\top)_{ii} \sqrt{\ln (20 p
q/\eta) /n} .
\]
Taking union bound with the previous three inequalities, we have with probability $1-\eta$:
\[
A_n \leq |\Sigma^n - \Sigma^*|_\infty + |\tilde{A} \Sigma^n_{xx}
\tilde{A}^\top - \tilde{A} \Sigma^*_{xx} \tilde{A}^\top |_\infty
\leq 8 K_* \sqrt{\ln (10 (p+q)^2/\eta) /n}
\]
and
\[
0.5 B_n \leq |\Sigma^n - \Sigma^*|_\infty + |\tilde{A} \Sigma^n_{xx}
- \tilde{A} \Sigma^*_{xx} |_\infty \leq 8 K_* \sqrt{\ln (10
(p+q)^2/\eta) /n}   ,
\]
where
\[
K_* = \max_{i} (\Sigma^*_{ii})+ \max_i (((\Omega^*_{yy})^{-1}\Omega^*_{yx}\Sigma^*_{xx}\Omega^{*\top}_{yx}(\Omega^*_{yy})^{-1})_{ii}) .
\]
This completes the proof.
\end{proof}

\subsection{Proof of Theorem~\ref{thrm:F_norm_bound}}

For convenience, we will introduce the following notations:
\[
\Delta\Omega_{yy}:=\hat\Omega_{yy} - \Omega^*_{yy} , \qquad
\Delta\Omega_{yx}:=\hat\Omega_{yx} - \Omega^*_{yx} ,
\]
and $\Delta \Theta = \hat\Theta - \Theta^* = (\Delta\Omega_{yy},\Delta\Omega_{yx})$.

We first introduce the following lemma which shows that error is in
the cone of Definition~\ref{def:cone}.
\begin{lemma}\label{lemma:solution_set}
Assume that
$\min\{\lambda_n,\rho_n \} \ge 2\gamma_n$. Then the error
$\Delta \Theta$ satisfies
$|\Delta \Theta_{\bar{S}}|_1 \leq \alpha |\Delta \Theta_{S}|_1$.
\end{lemma}
\begin{proof}
Since $(\Omega^*_{yy})_{\bar{S}_{yy}}=0$, we have
\begin{eqnarray}\label{equat:abs_Omega_yy_diff}
|(\Omega^*_{yy} + \Delta\Omega_{yy})^-|_1 - |(\Omega^*_{yy})^-|_1
&=& |(\Omega^*_{yy} + \Delta\Omega_{yy})^-_{S_{yy}}|_1 + |(\Omega^*_{yy} + \Delta\Omega_{yy})^-_{\bar{S}_{yy}}|_1 - |(\Omega^*_{yy})^-|_1 \nonumber \\
&=&|(\Omega^*_{yy} + \Delta\Omega_{yy})^-_{S_{yy}}|_1 + |(\Delta\Omega_{yy})_{\bar{S}^-_{yy}}|_1 - |(\Omega^*_{yy})^-|_1 \nonumber \\
&\ge& |(\Delta\Omega_{yy})^-_{\bar S_{yy}}|_1 - |(\Delta\Omega_{yy})^-_{S_{yy}}|_1 \nonumber \\
&\ge& |(\Delta\Omega_{yy})_{\bar S_{yy}}|_1 -
|(\Delta\Omega_{yy})_{S_{yy}}|_1.
\end{eqnarray}
Similarly we have
\begin{equation}\label{equat:abs_Omega_yx_diff}
|\Omega^*_{yx} + \Delta\Omega_{yx}|_1 - |\Omega^*_{yx}|_1 \ge
|(\Delta\Omega_{yx})_{\bar S_{yx}}|_1 -
|(\Delta\Omega_{yx})_{S_{yx}}|_1.
\end{equation}

We define the function $f(s)$ as in the proof of Proposition~\ref{prop:RSC}.
From the convexity of the loss $\Lpa$ we have
\[
\Lpa (\hat\Theta) - \Lpa(\Theta^*) = f(1)- f(0) \ge f'(0) = \Tr(A^\top_n
\Delta\Omega_{yy}) + \Tr(B^\top_n \Delta\Omega_{yx}),
\]
where
\[
A_n = \Sigma^n_{yy} - (\Omega^*_{yy})^{-1} -
(\Omega^*_{yy})^{-1}\Omega^*_{yx}\Sigma^n_{xx}(\Omega^*_{yx})^\top(\Omega^*_{yy})^{-1}
, B_n = 2(\Sigma^n_{yx} +
(\Omega^*_{yy})^{-1}\Omega^*_{yx}\Sigma^n_{xx}).
\]
From the equalities in~\eqref{equat:omega_sigma_connection} we can
equivalently write
\[
A_n = \Sigma^n_{yy} - \Sigma^*_{yy}
-(\Omega^*_{yy})^{-1}\Omega^*_{yx}(\Sigma^n_{xx}-\Sigma^*_{xx})\Omega^{*\top}_{yx}(\Omega^*_{yy})^{-1}
 , B_n =  2(\Sigma^n_{yx} - \Sigma^*_{yx} + (\Omega^*_{yy})^{-1}\Omega^*_{yx}(\Sigma^n_{xx} -
\Sigma^*_{xx})).
\]
Note that we have
\[
|\Tr(A^\top_n \Delta\Omega_{yy})| \le |A_n|_\infty
|\Delta\Omega_{yy}|_1 \le \frac{\lambda_n}{2}
|\Delta\Omega_{yy}|_1,
\]
and
\[
|\Tr(B^\top_n \Delta\Omega_{yx})| \le |B_n|_\infty
|\Delta\Omega_{yx}|_1 \le \frac{\rho_n}{2}
|\Delta\Omega_{yx}|_1,
\]
where we have used the
assumption $\min\{\lambda_n,\rho_n\} \ge 2 \gamma_n$.
Therefore
\begin{equation}\label{equat:Lpa_diff}
\Lpa (\hat\Theta) - \Lpa(\Theta^*) \ge -\frac{\lambda_n}{2}
|\Delta\Omega_{yy}|_1 - \frac{\rho_n}{2}
|\Delta\Omega_{yx}|_1.
\end{equation}
By combing \eqref{equat:abs_Omega_yy_diff},
\eqref{equat:abs_Omega_yx_diff}, and \eqref{equat:Lpa_diff}, we obtain
\begin{eqnarray}
0 &\ge& \Lpa (\hat\Theta) + R_e(\hat\Theta) - \Lpa(\Theta^*) - R_e(\Theta^*) \nonumber \\
&\ge& -\frac{\lambda_n}{2} |\Delta\Omega_{yy}|_1 -
\frac{\rho_n}{2} |\Delta\Omega_{yx}|_1 + \lambda_n
(|(\Delta\Omega_{yy})_{\bar S_{yy}}|_1 -
|(\Delta\Omega_{yy})_{S_{yy}}|_1)
+\rho_n (|(\Delta\Omega_{yx})_{\bar S_{yx}}|_1 - |(\Delta\Omega_{yx})_{S_{yx}}|_1) \nonumber \\
&\ge& \frac{\lambda_n}{2}\left(|(\Delta\Omega_{yy})_{\bar
S_{yy}}|_1 - 3|(\Delta\Omega_{yy})_{S_{yy}}|_1\right)
+\frac{\rho_n}{2} \left(|(\Delta\Omega_{yx})_{\bar S_{yx}}|_1 - 3|(\Delta\Omega_{yx})_{S_{yx}}|_1\right)\nonumber \\
&\ge& \frac{\min(\lambda_n,\rho_n)}{2}\left(|(\Delta\Omega_{yy})_{\bar
S_{yy}}|_1 + |(\Delta\Omega_{yx})_{\bar S_{yx}}|_1 \right)
-\frac{3\max(\lambda_n,\rho_n)}{2}
\left(|(\Delta\Omega_{yy})_{S_{yy}}|_1 + |(\Delta\Omega_{yx})_{S_{yx}}|_1\right)\nonumber ,
\end{eqnarray}
which implies
$|(\Delta\Theta)_{\bar S}|_1 \le \alpha |(\Delta\Theta)_{S}|_1$.
\end{proof}

\begin{proof}[Proof of Theorem~\ref{thrm:F_norm_bound}]
Since $\lambda_n, \rho_n \in [2 \gamma_n, c_0 \gamma_n]$, by
Lemma~\ref{lemma:solution_set} we have
$|(\Delta\Theta)_{\bar S}|_1 \le \alpha |(\Delta\Theta)_{S}|_1$.
Let $\Delta\tilde\Theta =
(\Delta\tilde \Omega_{yy}, \Delta\tilde \Omega_{yx}) = t
\Delta\Theta$ where we pick $t=1$ if
$\|\Delta\Theta\|_{F}<r_0$ and $t \in (0,1)$ with
$\|\Delta\tilde\Theta\|_{F} = r_0$ otherwise. By definition, we have
$\|\Delta\tilde\Theta\|_{F} \le r_0$ and
$|(\Delta\Theta)_{\bar S}|_1 \le \alpha |(\Delta\Theta)_{S}|_1$.
Due to the optimality of $\hat\Theta$ and the convexity of $\Lpa$, it holds that
\[
\Lpa (\Theta^* +
t\Delta\Theta) + R_e(\Theta^* + t\Delta\Theta)\le \Lpa (\Theta^*) +
R_e(\Theta^*).
\]
Following the similar arguments in Lemma~\ref{lemma:solution_set}
and the LRSC of $\Lpa$ we obtain
\begin{eqnarray}
0 &\ge& \Lpa (\Theta^* + t\Delta\Theta) + R_e (\Theta^* + t\Delta\Theta) - \Lpa (\Theta^*) - R_e(\Theta^*) \nonumber \\
&\ge& \frac{\lambda_n}{2}\left(|(\Delta\tilde\Omega_{yy})_{\bar
S_{yy}}|_1 - 3|(\Delta\tilde\Omega_{yy})_{S_{yy}}|_1\right)
+\frac{\rho_n}{2} \left(|(\Delta\tilde\Omega_{yx})_{\bar S_{yx}}|_1 - 3|(\Delta\tilde\Omega_{yx})_{S_{yx}}|_1\right)\nonumber \\
&&+ \beta(\Theta^*; r_0, \alpha) \|\Delta\tilde\Theta\|^2_{F} \nonumber \\
&\ge& -1.5\max\{\lambda_n,\rho_n\} |(\Delta\tilde\Theta)_S|_{1}  + \beta_0 \|\Delta\tilde\Theta\|^2_{F} \nonumber \\
&\ge& -1.5c_0\gamma_n\sqrt{|S|} \|\Delta\tilde\Theta\|_{F}  + \beta_0
\|\Delta\tilde\Theta\|^2_{F} \nonumber,
\end{eqnarray}
which implies that
\[
\|\Delta\tilde\Theta\|_{F} \le
1.5c_0\beta_0^{-1}\gamma_n\sqrt{|S|} = \Delta_n.
\]
Since $\Delta_n < r_0$, we claim that $t=1$ and thus
$\Delta\tilde\Theta = \Delta\Theta$. Indeed, if otherwise $t<1$,
then $\|\Delta\tilde\Theta\|_{F} = r_0 > \Delta_n$ which
contradicts the above inequality. This completes the proof.
\end{proof}

\section{Additional Materials on Monte Carlo Simulations}
\label{append:results_montecarlo}

In this appendix section, we provide the detailed performance
figures on the synthetic data as described in
Section~\ref{ssect:simulation}. For support recovery, we use
F-score. We also measure the precision matrix estimation quality by
three matrix norms: the operator norm, the matrix $\ell_1$-norm, and
the Frobenius norm. The results are presented in
Table~\ref{tab:synthetic_results_loss} and
Table~\ref{tab:synthetic_results_loss_marginal} .

\begin{table}
\begin{center}
\caption{Comparison of average CPU run times and average
matrix losses and F-scores for synthetic datasets over
$50$ replications. In this experiment, we fix $n=100$ and $p=50$.
\label{tab:synthetic_results_loss}}
\begin{tabular}{c c c c c }
\hline
Methods & $q=50$ & $q=100$ & $q=200$ & $q=500$ \\
\hline \hline
& &  \multicolumn{2}{c}{CPU Time (sec.) $\downarrow$} &  \\
pGGM  & 0.17 & 0.26 & 0.46 & 0.98  \\
cGGM  & 0.22  & 0.28 & 0.45 & 1.09 \\
GLasso  & 0.45 & 1.51 & 8.52 & 150.98 \\
NSLasso  & 2.01 & 2.36 & 3.14 & 5.38 \\

& &  \multicolumn{2}{c}{Operator norm $\|\hat\Theta - \Theta^*\|_2$ $\downarrow$} &  \\
pGGM  & 0.98 (0.04) & 1.06 (0.03) & 1.17 (0.03) & 1.23 (0.02) \\
cGGM  & 0.99 (0.04) & 1.07 (0.04) & 1.18 (0.03) & 1.23 (0.02)  \\
GLasso & 1.22 (0.05) & 1.44 (0.07) & 1.71 (0.07) & 2.31 (0.04) \\
NSLasso & --- & --- & --- & --- \\

& &  \multicolumn{2}{c}{Matrix $\ell_1$-norm $\|\hat\Theta - \Theta^*\|_{1}$ $\downarrow$} &   \\
pGGM  & 2.01 (0.12) & 1.98 (0.23) & 1.81 (0.11) & 1.10 (0.10) \\
cGGM  & 2.35 (0.16) & 2.13 (0.20) & 1.89 (0.06) & 1.10 (0.10) \\
GLasso  & 2.90 (0.20) & 3.03 (0.32) & 3.11 (0.21) & 3.29 (0.32) \\
NSLasso  & --- & ---  & --- & --- \\

& &  \multicolumn{2}{c}{Frobenius norm $\|\hat\Theta - \Theta^*\|_{F}$ $\downarrow$} &  \\
pGGM  & 3.36 (0.07) & 3.91 (0.11) & 4.81 (0.12) & 4.58 (0.04) \\
cGGM  & 3.43 (0.07) & 3.96 (0.12) & 4.85 (0.13) & 4.59 (0.04)  \\
GLasso  & 4.58 (0.11) & 5.94 (0.06) & 7.89 (0.08) & 12.22 (0.03) \\
NSLasso  & ---  & ---  & ---  & --- \\

& &  \multicolumn{2}{c}{Support Recovery F-score $\uparrow$} &  \\
pGGM  & 0.41 (0.01) & 0.37 (0.01) & 0.35 (0.01) & 0.23 (0.01) \\
cGGM  & 0.33 (0.01) & 0.31 (0.01) & 0.32 (0.01) & 0.23 (0.01) \\
GLasso  & 0.31 (0.01) & 0.27 (0.01) & 0.27 (0.01) & 0.22 (0.01) \\
NSLasso & 0.40 (0.01) & 0.35 (0.01) & 0.32 (0.01) & 0.21 (0.01) \\

\hline
\end{tabular}
\end{center}
\end{table}

\begin{table}
\begin{center}
\caption{Comparison of average CPU run times and average
matrix losses and F-scores for synthetic datasets over
$50$ replications. Here we fix $n=100$ and $p=50$.
\label{tab:synthetic_results_loss_marginal}}
\begin{tabular}{c c c c c }
\hline
Methods & $q=50$ & $q=100$ & $q=500$ & $q=1000$ \\
\hline \hline
& &  \multicolumn{2}{c}{CPU Time $\downarrow$} &  \\
pGGM  & 0.17 & 0.26 & 0.46 & 0.98  \\
GLasso-M  & 0.04 & 0.05 & 0.05 & 0.05 \\

& &  \multicolumn{2}{c}{Operator norm $\|\hat\Omega_{yy} - \Omega^*_{yy}\|_2$ $\downarrow$} &  \\
pGGM  & 0.76 (0.04) & 0.86 (0.07) & 0.91 (0.06) & 0.58 (0.01) \\
GLasso-M & 0.88 (0.06) & 0.86 (0.09) & 0.88 (0.03) & 0.86 (0.02) \\

& &  \multicolumn{2}{c}{Matrix $\ell_1$-norm $\|\|\hat\Omega_{yy} - \Omega^*_{yy}\|\|_{1}$ $\downarrow$} &   \\
pGGM  & 1.94 (0.12) & 1.94 (0.26) & 1.879 (0.13) & 0.94 (0.03) \\
GLasso-M  & 2.80 (0.18) & 2.87 (0.29) & 2.76 (0.08) & 1.93 (0.08) \\

& &  \multicolumn{2}{c}{Frobenius norm $\|\hat\Omega_{yy} - \Omega^*_{yy}\|_{F}$ $\downarrow$} &  \\
pGGM  & 2.55 (0.08) & 2.68 (0.12) & 3.17 (0.15) & 2.18 (0.06) \\
GLasso-M  & 3.14 (0.09) & 3.11 (0.09) & 3.26 (0.05) & 3.03 (0.04) \\

& &  \multicolumn{2}{c}{Support Recovery F-score $\uparrow$} &  \\
pGGM  & 0.42 (0.01) & 0.38 (0.02) & 0.39 (0.02) & 0.30 (0.01) \\
GLasso-M  & 0.31 (0.01) & 0.28 (0.01) & 0.27 (0.01) & 0.27 (0.01)\\
\hline
\end{tabular}
\end{center}
\end{table}

\end{document}